%% file: main.tex
\documentclass[11pt, a4paper, logo]{googledeepmind}

\makeatletter
\renewcommand\bibentry[1]{\nocite{#1}{\frenchspacing\@nameuse{BR@r@#1\@extra@b@citeb}}}
\makeatother

\usepackage{kantlipsum, lipsum}
\usepackage{dsfont}
\usepackage{gdm-colors}

\input{myheaders.tex}

\allowdisplaybreaks

\usepackage{listings}

\usepackage[authoryear, sort&compress, round]{natbib}

\graphicspath{{figures/}}

\title{Building Math Agents with Multi-Turn Iterative Preference Learning}

\correspondingauthor{$^*$ Work done during an internship at Google DeepMind. A preliminary draft without the results of Gemma-2 had been circulated internally in early July. Correspondence to: wx13@illinois.edu, tongzhang@tongzhang-ml.org, tianqiliu@google.com.}

\keywords{RLHF, Agent learning, Mathematical reasoning}

\renewcommand{\today}

\author[1,*]{Wei Xiong}
\author[2]{Chengshuai Shi}
\author[3]{Jiaming Shen}
\author[4]{Aviv Rosenberg}
\author[3]{Zhen Qin}
\author[3]{Daniele Calandriello}
\author[3]{Misha Khalman}
\author[3]{Rishabh Joshi}
\author[3]{Bilal Piot}
\author[3]{Mohammad Saleh}
\author[5]{Chi Jin}
\author[1]{Tong Zhang}
\author[3]{Tianqi Liu}

\affil[1]{University of Illinois Urbana-Champaign}
\affil[2]{University of Virginia}
\affil[3]{Google Deepmind}
\affil[4]{Google Research}
\affil[5]{Princeton University}

\begin{abstract}
Recent studies have demonstrated the potential to enhance the mathematical problem-solving capabilities of large language models (LLMs) by integrating external tools such as code interpreters and employing multi-turn Chain-of-Thought (CoT) reasoning. While existing approaches primarily focus on synthetic data generation and Supervised Fine-Tuning (SFT), this paper explores complementary preference learning to further improve model performance. However, existing direct preference learning algorithms are originally designed for the single-turn chat task, and do not fully address the complexities of multi-turn reasoning and external tool integration required for tool-integrated mathematical reasoning tasks. To fill in this gap, we introduce a multi-turn online iterative direct preference learning framework tailored to this unique context, which incorporates feedback from code interpreters and optimizes trajectory-level preferences. The effectiveness of our framework is validated through training of various language models using an augmented prompt set derived from GSM8K and MATH datasets. Our results show significant improvements even with only final result checking: for instance, the performance of a supervised fine-tuned Gemma-1.1-it-7B model increased from 77.5\% to 83.9\% on GSM8K and from 46.1\% to 51.2\% on MATH. Similarly, a Gemma-2-it-9B model improved from 84.1\% to 86.3\% on GSM8K and from 51.0\% to 54.5\% on MATH.
\end{abstract}

\begin{document}

\maketitle


\section{Introduction} \label{sec:intro}

Large language models (LLMs) have demonstrated remarkable capacities across a variety of language tasks, showcasing their broad-ranging capabilities in natural language processing. Notable models include ChatGPT \citep{OpenAI2023GPT4TR}, Claude \citep{Anthropic@claude}, and Gemini \citep{team2023gemini}. However, despite these advances, even the most advanced closed-source LLMs still struggle with complex reasoning tasks that require multi-rounds of decision making. In particular, for the representative task of mathematical problem solving, LLMs often fail with basic arithmetic and symbolic computations \citep{hendrycks2021measuring, cobbe2021training, zheng2021minif2f}. To address this issue, recent studies recommend the integration of external tools (e.g., calculators, computational Python libraries and symbolic solvers) to augment the LLMs' mathematical problem-solving capabilities \citep{cobbe2021training, shao2022chaining, mishra2022lila, zhang2024evaluating}. Specifically, by integrating natural language reasoning with the use of these external tools, these enhanced LLMs can receive external messages from tool interactions and reason based on both previously generated tokens and external messages, which significantly improves their performance in mathematical tasks \citep{gou2023tora, toshniwal2024openmathinstruct, shao2024deepseekmath}.

    These successes of tool-integrated LLMs lead to a natural research question: how can we better train LLMs to combine tool usage with intrinsic reasoning to tackle complex reasoning tasks? For the mathematical problem solving task, existing works primarily focus on synthetic data generation (by a strong teacher model) and supervised fine-tuning (SFT), as seen in ToRA \citep{gou2023tora}, MetaMathQA \citep{yu2023metamath}, MAmmoTH \citep{yue2023mammoth, yue2024mammoth2}, and Open-MathInstruct \citep{toshniwal2024openmathinstruct}. These methods and synthetic datasets have yielded significant improvements in test accuracy on standard benchmarks like MATH \citep{hendrycks2021measuring} and GSM8K \citep{cobbe2021gsm8k}. 

Building on strong SFT models, \textit{Reinforcement Learning from Human Feedback} (RLHF) has proven to be a key technique to elicit LLMs' knowledge during the post-training stage and has become a standard practice in the LLM training pipeline \citep{bai2022training, ouyang2022training, touvron2023llama, team2023gemini}. Broadly speaking, the RLHF learning paradigm, which was originally designed for aligning large language models (LLMs) with human values and preferences \citep{bai2022training, ouyang2022training},  is distinct from SFT as it learns from \textit{relative feedback} \citep{christiano2017deep, ziegler2019fine}. It has notably enhanced the capabilities of models like ChatGPT, Claude, and Gemini, enabling them to generate responses that are more helpful, harmless, and honest \citep{bai2022training}. Inspired by RLHF's success in general chat applications, in this paper, we explore RLHF for improving LLMs' mathematical problem-solving abilities when equipped with external tools.

In particular, since deep RL methods (e.g., the proximal policy optimization, PPO algorithm \citep{schulman2017proximal}) are often sample inefficient and unstable \citep{choshen2019weaknesses}, our goal is to derive direct preference learning algorithms that directly learn from the preference dataset \citep{zhao2023slic, rafailov2023direct, azar2023general}. We begin by formulating the learning process as a Markov decision process (MDP), distinct from the contextual bandit approach typically used in RLHF for making general chatbots without external environment interactions \citep{xiong2024iterative, rafailov2023direct}. Then, we derive the optimality condition of the optimization problem and develop multi-turn variants of direct preference learning algorithms that incorporate external messages, where the primary modification is to mask out irrelevant tokens during training. Furthermore, we extend our approach to its online iterative variants, which recent works demonstrated to be promising \citep{xiong2024iterative, guo2024direct}. 

We evaluate our approach through case studies using augmented training sets from MATH and GSM8K benchmarks, employing various base models such as Gemma \citep{team2024gemma}, CodeGemma \citep{team2024codegemma}, and Mistral \citep{jiang2023mistral}. For instance, the performance of a supervised fine-tuned Gemma-1.1-it-7B model increased from 77.5\% to 83.9\% on GSM8K and from 46.1\% to 51.2\% on MATH. Similarly, a Gemma-2-it-9B model improved from 84.1\% to 86.3\% on GSM8K and from 51.0\% to 54.5\% on MATH. These empirical results indicate a significant improvement in performance over standard SFT models, demonstrating the potential of RLHF in complex reasoning task. We also provide a comprehensive recipe for the practical implementation of our online iterative multi-turn methods, and make our models, datasets, and code publicly available for further research and development.

\subsection{Problem Formulation} \label{sec:formulation}
We denote prompt as $x \in \cX$ and assume that the interactions run for up to $H$ rounds. At the first step, a prompt $x$ is sampled from some distribution $d_0$ as the initial state $s_1$ (We use the terminology ``state'' instead of ``context'' because we are concerning about an MDP instead of a contextual bandit here). Then, at each step $h \in [H]$, 
\begin{itemize}
    \item \textbf{Action:} the agent observes the current state $s_h$, which is the history of the first $h-1$ interactions with the external environment, and takes an action $a_h$ according to some policy $\pi_h(\cdot|s_h) \in \Delta(\cA)$. Typically, the action is in the ReAct manner, which consist of a reasoning step $f_h$ and an execution step $e_h$ (e.g., writing python code) \citep{yao2022react}.
    \item \textbf{Observation:} in response to the agent's action, the environment then returns an observation $o_h$ based on the history $s_h$ and current action $a_h$.
\end{itemize}
Then, we transit to a new state, which is the history up to the step $h+1$:
$$s_{h+1} = (s_h, a_h, o_h) = (x, a_1, o_1, \cdots, a_{h-1}, o_{h-1}),$$
and a new step begins. This process repeats for $H$ rounds in total and eventually, we collect a trajectory:
$$
\tau = (x, a_1, o_1, \cdots, o_{H-1}, a_{H}).
$$
See Table~\ref{tab:multi_turn_math} for an example. The framework presented here is a Markov decision process (MDP), which offers a distinct approach from the contextual bandit model discussed in \citet{xiong2024iterative}. Formally, we define the following MDP.
\begin{definition} An MDP is specified by a tuple $(\cS, \cA, H, \PP^*, d_0)$, where $\cA$ is the action space, $H$ is the episode length\footnote{In practice, the episode length can vary across the trajectories. We may additionally define that the shorter trajectories that output the final answer are in an absorbing state. We consider the fixed episode length to simplify the subsequent mathematical analysis. }, $\PP^*=\{\PP^*_h\}_{h=1}^H$ are the state transition kernels, and $d_0$ denotes the distribution of prompt $s_1=x$. For each $h \in [H]$, $\PP^*_h(\cdot|s_h,a_h)$ is the distribution of the next state given the state-action pair $(s_h,a_h)$ at step $h$. In our setup, a trajectory $\tau = (x, a_1, o_1, \cdots, o_{H-1}, a_{H})$ is generated by: $s_1=x \sim d_0$ and for all $h \in [H], a_h \sim \pi_h(\cdot|s_h), o_h \sim \PP^*_h(\cdot|s_h,a_h)$ where $s_{h+1} = (s_h, a_h, o_h)$. When there is no ambiguity, the abbreviation $s_{h+1} \sim \PP^*_h(\cdot|s_h,a_h)$ is also adopted.
\end{definition}

The MDP formulation of preference learning was recently studied in \citet{zhong2024dpo, rafailov2024r, xie2024exploratory} but with a focus on the single-turn chat task and without explicitly considering the external messages. A unique feature of RLHF, as opposed to traditional RL studies, is the \textit{relative feedback} obtained through comparisons between two trajectories that share the same initial state (prompt). We follow \citet{ziegler2019fine, ouyang2022training, bai2022training} to assume that the preference signal is generated by the so-called Bradley-Terry model.

\begin{definition}[Bradley-Terry model] We denote $\tau/x = y$, where the prompt is excluded from the trajectory. We assume that there exists a utility function of the trajectory $u^*$ such that given $(x, y^1, y^2)$, one response $y^1$ is preferred over another response $y^2$, denoted as $y^1 \succ y^2$, with probability
\begin{equation} \label{def:bt_model}
    \textup{Prob}\big(y^1 \succ y^2 \mid x, y^1,y^2 \big) = \sigma\big(u^*(x,y^1)-u^*(x,y^2)\big),
\end{equation}
where $\sigma$ is the sigmoid function $\sigma(z) = 1/(1+\exp(-z))$. Also, given $(x, y^1, y^2)$ we denote the sampled preference signal as $z$ with $z=1$ indicating $y^1 \succ y^2$ while $z=0$ indicating $y^2 \succ y^1$.
\end{definition}

Under this definition, we only assume access to the trajectory-level preference, but not an action-level one. This should distinguish our approach from a straightforward extension of the single-turn RLHF \citep{christiano2017deep, ziegler2019fine}, which fixes a prompt that may include mid-trajectory steps such as $(x, a_1, o_1, a_2, o_2)$ and look into the next single step $a_3$. However, we remark that the utility function itself, can be defined in a step-wise manner. To further illustrate the notion of the BT model in trajectory-level comparisons, we provide some examples of the utility function here.

\begin{example}[Result Checking in Math]
Since the math reasoning datasets GSM8K \citep{cobbe2021gsm8k} and  MATH \citep{hendrycks2021measuring} have the gold answer, we can check the final answer to determine the reward. In this case, $u^*(x, y) = \mathbb{I}(a_H = \text{gold answer})$. 
\end{example}

\begin{example}[Outcome-supervised Reward Models (ORMs)]
Final result checking is not perfectly reliable because we can encounter false positive solutions that have the correct answer but incorrect reasoning trajectory. Instead, as shown in \citet{cobbe2021training, lightman2023let}, we can uniformly sample $n$ trajectories per prompt and train an ORM to predict whether each solution is correct or not. Then, we can take the ORM prediction at the final token as the utility function. 
\end{example}

\begin{example}[Process-supervised Reward Model (PRM) and PRM without Human Annotation.]
\citet{lightman2023let} argues that if we can provide step-by-step supervision signal, the utility function is more effective. However, this requires more fine-grained human labels to give rating for each step of the trajectory. \citet{wang2023math} studies how to automatically construct the process-labeled data for math problems with gold answers. Specifically, for $s_h,a_h$, we generate $N$ trajectories with final answers $[a_H^j]_{j=1}^N$. We can define the proxy reward value: 
\begin{equation}
    \begin{aligned}
        r(s_h, a_h) := \frac{\sum_{j=1}^N \mathbb{I}(a_H^j = \text{gold answer}) }{N}.
    \end{aligned}
\end{equation}
We may also use a hard version
\begin{equation}
    \begin{aligned}
        r(s_h, a_h) := \mathbb{I}(\text{There exists a } j_0: a_H^{j_0} = \text{gold answer}). 
    \end{aligned}
\end{equation}
Then, we can train the PRM by 
\begin{equation}
    \mathcal{L}_{PRM}(\theta) = \E_{\tau \sim \cD} \Big[ \sum\nolimits_{h=1}^H r(s_h,a_h) \log r_\theta + (1-r(s_h,a_h)) \log (1-r_\theta)\Big]. 
\end{equation}
In this case, we can use $u^*(x, y) = \min_{h \in [H]} r_\theta(s_h, a_h)$ \citep{lightman2023let}, where $r_\theta$ is the constructed step-wise reward function. 
\end{example}

\textbf{Notations.}  To improve the readability of this work, we provide a notable table in Table~\ref{tab:notation}.

\subsection{Related Work}

\paragraph{LLMs for Mathematical Problem Solving.} A line of works proposes to prompt LLMs to solve the complex reasoning task in a step-by-step manner, known as the Chain-of-Thought (CoT) prompting  \citep{wei2022chain, zhou2022least, zhu2022solving, tongdart}, which has been a standard practice in reasoning task. However, LLMs often struggle with basic arithmetic and symbolic manipulations when relying solely on internal knowledge and natural language reasoning, as measured by standard benchmarks \citep{cobbe2021gsm8k, hendrycks2021measuring}. To overcome these limitations, several studies have explored the use of external tools to enhance the LLMs' problem-solving abilities. This includes calculators \citep{cobbe2021training, shao2022chaining}, symbolic solvers \citep{zhang2023mathematical}, and code interpreters \citep{mishra2022lila, OpenAI2023GPT4TR}. A particularly effective approach is the Program-based method (PoT), which performs CoT reasoning by writing code and using the output of the written code as the final answer \citep{gao2023pal, chen2022program}. This method significantly outperforms traditional CoT-based techniques in mathematical problem solving. However, PoT also faces challenges in planning and error handling, where natural language reasoning is more suitable \citep{gou2023critic}. In view of this, tool-integrated reasoning is proposed to combine the natural-language-based intrinsic reasoning with the external tools \citep{gou2023tora} and has achieved great progresses in recent studies \citep{gou2023tora, yue2023mammoth, yu2023metamath, shao2024deepseekmath, toshniwal2024openmathinstruct}. While these efforts have primarily focused on synthetic data generation for tool-integrated reasoning, our work aims to further boost the performance of tool-integrated LLMs by RLHF.

\paragraph{RLHF and RLHF Algorithms.} The predominant approach in RLHF is the deep RL method, Proximal Policy Optimization Algorithms (PPO) \citep{schulman2017proximal}, which leads to the great successes in Chat-GPT \citep{OpenAI2023GPT4TR}, Gemini \citep{team2023gemini}, and Claude \citep{Anthropic@claude}. However, applying PPO requires extensive efforts and resources \citep{choshen2019weaknesses, engstrom2020implementation}, often beyond the scope of open-source capabilities. In view of this, alternative approaches have been developed. The rejection sampling fine-tuning was first proposed with the name RAFT (reward ranked fine-tuning) in RLHF \citep{dong2023raft} and was later extended to machine translation \citep{gulcehre2023reinforced} and mathematical problem solving \citep{yuan2023scaling}. Its theoretical advantage was explored in \citet{gui2024bonbon}. Subsequently, another long line of works proposes direct preference learning algorithms, including Slic \citep{zhao2023slic}, DPO \citep{rafailov2023direct}, IPO \citep{azar2023general}, KTO \citep{ethayarajh2024kto}, and GPO \citep{tang2024generalized}. These algorithms bypass the reward modeling step and optimize carefully designed loss objectives directly on the preference dataset, hence the name direct preference learning. There are also some works focusing on more general preference structure \citet{munos2023nash, swamy2024minimaximalist, ye2024theoretical, rosset2024direct} beyond the reward-based framework or post-processing of the model \citep{lin2023speciality, zheng2024weak}.

The newly proposed direct preference learning algorithms have largely advanced the RLHF area, particularly the post-training of open-source models, with the Zephyr project as a notable example \citep{tunstall2023zephyr}. After this, a long line of work \citep[e.g.,][]{liu2023statistical, xiong2024iterative, guo2024direct, xu2023some, tajwar2024preference, xie2024exploratory, zhang2024self, liu2024lipo, liu2024provably, meng2024simpo} demonstrates the effectiveness of on-policy sampling (the samples are generated by the policy to be trained) and online exploration in enhancing direct preference learning. In particular, the online iterative DPO \citep{xiong2024iterative, xu2023some, snorkelai@pair} and its variants \citep[e.g.,][]{chen2024self, rosset2024direct, cen2024value, zhang2024iterative} have made state-of-the-art open-source models \citep{dong2024rlhf}, or even the industry models \citep{qwen2, meta_llama3}. Despite these advancements, most algorithms are proposed and designed for single-turn interactions and chat. The scenarios beyond single-turn chat remain largely unexplored in the existing literature. One exception is the very recent work by \citet{shani2024multi}, which studies multi-turn chat task under general preferences. In contrast, in this paper, we aim to explore the use of RLHF in multi-turn tasks that incorporate interactions with external tools. Meanwhile, they derive a mirror-descent-based policy optimization algorithm, which is also different from ours. 

\paragraph{RLHF for Math Problem Solving.} Algorithms traditionally used in general chatbot applications have been adapted to enhance the reasoning capabilities of LLMs in mathematical contexts. For instance, RAFT (Reward-rAnked Fine-Tuning) \citep{dong2023raft, yuan2023rrhf, touvron2023llama} is extensively employed for synthetic data generation, whether through on-policy (self-improving) \citep{yuan2023scaling} or off-policy (knowledge distillation) methods \citep{gou2023tora, yu2023metamath, toshniwal2024openmathinstruct, singh2023beyond, tongdart}. The reward signal in these scenarios is typically derived from either final result checking or Outcome-supervised Reward Models (ORMs) \citep{uesato2022solving, zelikman2022star}. A novel approach by \citet{lightman2023let} introduces Process-supervised Reward Models (PRMs), which provide feedback at each step of the Chain-of-Thought, demonstrating significant improvements over ORMs when combined with rejection sampling \citep{lightman2023let, wang2023math}.

In addition to the RAFT, the GRPO algorithm proposed in \citet{shao2024deepseekmath} studies multi-turn math problem solving but focuses on the CoT format without external inputs and the resulting model achieves the state-of-the-art performance in its class. The GRPO is a variant of Reinforce \citep{williams1992simple} thus falling into the scope of deep RL methods. 

Further advancements include adapting direct preference learning algorithms to mathematical problem solving. For instance, \citet{jiao2024learning, yuan2024advancing} have applied the original DPO or KTO by taking the trajectory completion as a ``meta'' action.  \citet{xie2024monte, pang2024iterative} further adapt the online iterative DPO originally designed for chat \citep{xiong2024iterative, xu2023some, snorkelai@pair} and achieve better performance for CoT reasoning. Inspired by the success of PRMs, recent studies have explored generating proxy step-wise labels for the intermediate steps of the reasoning trajectories. For instance, \citet{xie2024monte, chen2024step, lai2024step} leverage Monte Carlo Tree Search (MCTS) and use the estimated Q value to generate the proxy labels for the intermediate steps. \citet{lai2024step} proposes to use AI feedback like GPT-4 \citep{lai2024step} to find the first error step in the trajectory. Meanwhile, \citet{lu2024step} identifies a trajectory with the correct final answer and no errors as preferable, and prompts the SFT model with a high temperature, starting from some intermediate step to collect a rejected trajectory with errors \citep{pi2024strengthening}. Finally, a very recent study by \citet{chen2024step} proposes to use MCTS with a backward iteration from the final leaf node to compute the proxy unregularized value of each node. Preference pairs are then extracted from the tree by fixing the prefix and comparing \textit{the next single reasoning step}. Then, they run the original DPO on these intermediate actions with the proxy labels from MCTS. To summarize, these works present different ways of preference data collection and apply the original DPO algorithm (with some additional marginal loss and regularization adapted from the literature), thereby differing from our work in both algorithmic concepts and application scope. In contrast, we study preference learning in the context of trajectory-level comparison, where we derive the optimality condition and introduce a multi-turn DPO within an online iterative framework, specifically for tool-integrated mathematical problem solving. However, we remark that while we focus on the trajectory-level comparison, the preference signal itself can be generated in a step-by-step supervision (see Section~\ref{sec:formulation} for the detailed examples). When preference signals for partial trajectories with shared prefixes are available, our method can also adapt to learn these step-level signals (see the optimality condition in \eqref{eqn:optimality_condition}). In particular, the algorithmic design presented in this paper can be readily combined with the MCTS-based data collection strategy outlined in recent literature, which we leave for future work.

\section{Algorithms Development}
\label{sec:alg}
We develop the main algorithms of this paper in this section. We proceed to handle the general MDP formulation presented in Section~\ref{sec:formulation}, which subsumes the tool-integrated mathematical reasoning problem as a special example. Therefore, the algorithms may also be applied to more general scenarios with external messages..

\subsection{Planning with a Fixed Model: Optimality Condition}
Following \citet{rafailov2023direct}, we first establish the connection between any model $\cM = (\cS, \cA, H, \PP, d_0, u)$ and its associated optimal policy. In particular, we are interested in the following KL-regularized planning problem with respect to a reference policy $\pi_{\reff}$: 
\begin{equation}\label{eqn:optimal_policy}
    \argmax_{\pi} J(\pi; \cM, \pi_{\reff})= \E_{x \sim d_0} \E_{a_h \sim \pi_h(\cdot|s_h), o_h \sim \PP_h(\cdot|s_h,a_h)} \left[ u(x, y) - \eta \sum_{h=1}^H \KL\big( \pi_h(\cdot|s_h), \pi_{\reff, h}(\cdot|s_h)\big)\right].
\end{equation}
In the single-turn case (i.e., $H=1$ and without transitions $\PP$), \citet{rafailov2023direct, azar2023general} show that the optimal solution with respect to a utility function $u$ admits a closed-form solution, which is the \textit{Gibbs distribution} (see Lemma~\ref{lem:kl_solu}):
$$
\pi_\cM(a_1|x) \propto \pi_{\reff}(a_1|x)\exp\left(\frac{u(x,a_1)}{\eta} \right).
$$
Moving from the single-step to multi-turn scenario, we first show that we are still concerning about the Gibbs distribution, but in a dynamic programming manner. We summarize the results into the following proposition.
\begin{proposition} \label{prop:optimality}
    We can recursively define the following optimal value functions and optimal policies for a KL-regularized MDP with horizon $H$ and external observation $o_h$. For Q value, we have
    \begin{equation}
    \begin{aligned}
                Q_{\cM, h}(s_h, a_h) =\begin{cases}
    & u(s_H, a_H),  \qquad \text{ if } h = H,  \\
  & \E_{o_h \sim \PP_h(\cdot|s_h, a_h)} [V_{\cM, h+1}(s_{h+1})], \qquad \text{ if } h \leq H-1.
\end{cases}   
    \end{aligned}
    \end{equation}
    Also, for all $h \in [H]$, we have:
    \begin{equation*}
    {\footnotesize
        \begin{aligned}
            V_{\cM, h}(s_h)  = \eta \log \underbrace{\textcolor{red}{\mathbb{E}_{a_h \sim \pi_{\mathrm{ref}, h}(\cdot\mid s_h)}} \exp\Big(\frac{Q_{\cM, h}(s_h,a_h)}{\eta}\Big)}_{=: Z_h(s_h)},
        \end{aligned}}
    \end{equation*}
    \begin{equation} \label{eqn:3}
    {\footnotesize
        \begin{aligned}
            \pi_{\cM, h}(a_h \mid s_h) = \frac{\pi_{\mathrm{ref}, h}(a_h\mid s_h)}{Z_h(s_h)} \cdot \exp \Big(\frac{Q_{\cM, h}(s_h, a_h)}{\eta}\Big).
        \end{aligned}}
    \end{equation} 
\end{proposition}
We have a few interesting observations that may be of independent interests. 
\begin{enumerate}
    \item The optimal value function is characterized by the expectation with respect to the initial reference policy due to the additional KL constraint.
    \item For a fixed step $h$ and state-action pair $(s_h, a_h)$, we can treat the future as a bandit (with only one step), then, we have $Q_{\cM, h}(s_h, a_h) = \E_z u(s_h, a_h, z)$, where $z$ is a completion staring from $(s_h, a_h)$. One can use the Monte-Carlo estimation to estimate this value by multiple roll-outs. We notice that the non-regularized version of this process, is commonly referred to as the \textit{process-supervised reward} (PRM) in the literature \citep{wang2023math}. In other words, the PRM constructed in \citet{wang2023math} is essentially a Q learning process.
\end{enumerate}

The results are essentially from the study of entropy-regularized MDPs \citep{williams1991function, ziebart2010modeling}. 

To illustrate the idea, we first consider the simplest case of $H=2$, where the model is allowed to call the tool only once. Then, our goal is to maximize the following target:
$$
\begin{aligned}
    \E_{x \sim d_0}\Big[ \E_{a_1 \sim \pi_1(\cdot|x)} \Big[\E_{o_1 \sim \PP_1(\cdot|x, a_1)} \underbrace{\E_{a_2 \sim \pi_2(\cdot|s_2)} u(s_2, a_2) - \eta \KL\big(\pi_2(\cdot|s_2), \pi_{\reff, 2}(\cdot|s_2)\big)}_{ \displaystyle \text{Inner Loop} } \Big] - \eta \KL\big(\pi_1(\cdot|s_1), \pi_{\reff, 1}(\cdot|s_1)\big) \Big].
\end{aligned}
$$
The idea is to take a backward iteration from $h=H=2$ to $h=1$. Specifically, when we fix $s_2$ and consider the inner loop, we can leverage Lemma~\ref{lem:kl_solu} to solve 
$$
\pi_{\cM,2}(\cdot|s_2) = \argmax_{\pi_2} \E_{a_2 \sim \pi_2(\cdot|s_2)}\Big(u(s_2,a_2) - \eta \cdot  \KL\big(\pi_2(\cdot|s_2), \pi_{\reff, 2}(\cdot|s_2) \big) \Big) \propto \pi_{\reff, 2}(\cdot|s_2) \cdot \exp\left(\frac{u(s_2,\cdot)}{\eta}\right).
$$
Then, we can define the value of the inner loop associated with $\pi_{\cM,2}$ as 
\begin{align*}
    V_{\cM,2}(s_2)&: = \E_{a_2\sim \pi_{\cM, 2} (\cdot|s_2)}\left[ u(s_2,a_2) - \eta \KL\big(\pi_{\cM,2}(\cdot|s_2), \pi_{\reff, 2}(\cdot|s_2)\big) \right]  \\
    Q_{\cM, 1}(s_{1},a_{1})&:= \E_{o_{1}\sim \PP_{1}(\cdot|s_{1},a_{1})} \left[V_{\cM, 2}(s_2)\right].
\end{align*}
Then, for step $h=H-1=1$, we are concerning the following KL-regularized optimization problem:
\begin{align*}
\pi_{\cM, 1}(\cdot|s_{1}) = \argmax_{\pi_1} \E_{a_1 \sim \pi_1(\cdot|x)} \Big[Q_{\cM, 1}(s_1,a_1) - \eta \KL\big(\pi_1(\cdot|s_1), \pi_{\reff, 1}(\cdot|s_1)\big)  \Big] \propto \pi_{\reff,1}(\cdot|s_{1}) \cdot \exp\Big(\frac{Q_{\cM, 1}(s_{1}, \cdot)}{\eta}\Big).
\end{align*}
By construction, it can be observed that $\{\pi_{\cM, h}\}_{h=1}^2$ is optimal as it maximizes the KL-regularized target. 

For general $H$-step MDP, we can repeat the above process for $H$ times starting with $V_{\cM, H+1} = 0$ where we recursively define
\begin{equation}  \label{eqn:bellman_eqn0}
Q_{\cM, h}(s_h, a_h) =\begin{cases}
    & u(s_H, a_H),  \qquad \text{ if } h = H,  \\
  & \E_{o_h \sim \PP_h(\cdot|s_h, a_h)} [V_{\cM, h+1}(s_{h+1})], \qquad \text{ if } h \leq H-1, 
\end{cases}   
\end{equation}
Here the optimal policy and the $V$-values are given by
\begin{equation} \label{eqn:bellman_eqn}
    \begin{aligned}
        \pi_{\cM, h}(a_h|s_h) &:= \frac{1}{Z_h(s_h)} \pi_{\reff, h}(a_h|s_h) \cdot \exp \Big(\frac{Q_{\cM, h}(s_h, a_h)}{\eta}\Big) \qquad \text{(Gibbs distribution of $Q_{\cM, h}$)}  \\
        V_{\cM, h}(s_h) &:= \E_{a_h \sim \pi_{\cM, h}(\cdot| s_h)} \big[Q_{\cM, h}(s_h, a_h) - \eta \cdot \KL\big(\pi_{\cM, h} (\cdot|s_h), \pi_{\reff, h}(\cdot|s_h)\big)\big]\\
        &= \eta \log \E_{\pi_{\reff, h}(a'_h|s_h)} \exp\Big(\frac{Q_{\cM, h} (s_h,a_h')}{\eta}\Big),
    \end{aligned}
\end{equation}
where $Z_h(s_h) = \sum_{a_h \in \cA} \pi_{\reff, h}(a_h|s_h) \cdot \exp \big(\frac{Q_{\cM, h}(s_h, a_h)}{\eta}\big)$ is the normalization constant. The second equality in the definition of the $V$-value is from Lemma~\ref{lem:kl_solu}. Then, by definition, $[\pi_{\cM, h}]_{h=1}^H$ is the optimal policy. Essentially, we solve $H$ Gibbs distributions in terms of the $Q$-values\footnote{The definitions of $Q$-values are different from that of \citet{ziebart2010modeling} so that the optimal policy can be interpreted as the Gibbs distribution of $Q$-values.}. 

\subsection{Planning with a Fixed Model: Practical Algorithm}\label{subsec:practical_plan_alg}
While \eqref{eqn:bellman_eqn} can be approximately solved with standard deep RL methods, here we are interested in the implementation in a direct preference learning manner like Slic \citep{zhao2023slic}, DPO \citep{rafailov2023direct} or IPO \citep{azar2023general}. The existing attempts \citep[e.g.,][]{yuan2024advancing} take the completion $y$ as a ``meta action'' and plug it into the single-step DPO loss. In other words, they treat the external messages as the regular texts generated by the model itself. Another natural idea is to plug the probability of the trajectory into the single-step DPO loss. To be specific, for a pair $(x,\tau^w,\tau^l)$, where $\tau^w$ refers to the preferred (i.e., winning) trajectory, we have
\begin{equation} \label{eqn:m_dpo_loss}
\begin{aligned}
&- \log \sigma\Big(\eta \Big[\log \frac{\textup{Prob}_\pi(\tau^l|x)}{\textup{Prob}_{\pi_{\reff}}(\tau^l|x)} - \log \frac{\textup{Prob}_\pi(\tau^w|x)}{\textup{Prob}_{\pi_{\reff}}(\tau^w|x)}\Big] \Big) \\
&= - \log \sigma\Big(\eta \Big[\log \prod_{h=1}^H \frac{\pi_h(a^l_h|s^l_h) \cancel{\PP_h(o^l_h|s^l_h,a^l_h)}}{\pi_{\reff, h}(a^l_h|s_h^l) \cancel{\PP_h(o^l_h|s^l_h,a^l_h)}} - \log \prod_{h=1}^H \frac{\pi_h(a^w_h|s^w_h) \cancel{\PP_h(o^w_h|s^w_h,a^w_h)}}{\pi_{\reff, h}(a^w_h|s_h^w) \cancel{\PP_h(o^w_h|s^w_h,a^w_h)}}\Big] \Big) \\
&= -  \log \sigma\Big(\eta \sum_{h=1}^H \Big[\log \frac{\pi_h(a^l_h|s_h^l)}{\pi_{\reff, h}(a^l_h|s_h^l)} - \log \frac{\pi_h(a^w_h|s_h^w)}{\pi_{\reff, h}(a^w_h|s_h^w)}\Big] \Big).
\end{aligned}
\end{equation}
Unfortunately, the resulting algorithm does not always lead to the optimal policy as we explain next. In particular, we can solve the $Q$-values as 
\begin{equation}
\begin{aligned}
    Q_{\cM, h}(s_h, a_h) &= \log \frac{\pi_{\cM, h}(a_h|s_h)}{\pi_{\reff, h}(a_h|s_h)} + \eta \log \E_{\pi_{\reff, h}(a'_h|s_h)} \exp\Big(\frac{Q_{\cM, h} (s_h,a_h')}{\eta}\Big)\\
    & = \log \frac{\pi_{\cM, h}(a_h|s_h)}{\pi_{\reff, h}(a_h|s_h)} + V_{\cM, h}(s_h),
\end{aligned}
\end{equation}
where two equalities uses the definition of the optimal policy $\pi_{\cM, h}$ and $V$-value $V_{\cM, h}$ in \eqref{eqn:bellman_eqn}, respectively. Furthermore, by the definition of $Q$-values $Q_{\cM,h}$ in \eqref{eqn:bellman_eqn0}, we have
\begin{equation}
\begin{aligned}
    \E_{o_h \sim \PP_h(\cdot|s_h, a_h)} V_{\cM, h+1}(s_{h+1}) &= \log \frac{\pi_{\cM, h}(a_h|s_h)}{\pi_{\reff, h}(a_h|s_h)} + V_{\cM, h}(s_h), \qquad \text{if } h \leq H-1\\
    u(s_H, a_H) &= \log \frac{\pi_{\cM, H}(a_H|s_H)}{\pi_{\reff, H}(a_H|s_H)} + V_{\cM, H}(s_H).
\end{aligned}
\end{equation}
Summing over $h \in [H]$, we have
\begin{equation} \label{eqn:optimality_condition}
    \begin{aligned}
        u(s_H, a_H) &= \eta \sum_{h=1}^H \log \frac{\pi_{\cM, h}(a_h|s_h)}{\pi_{\reff, h}(a_h|s_h)} + \sum_{h=1}^H \Big[V_{\cM, h}(s_h) -\E_{o_h \sim \PP_h(\cdot|s_h, a_h)} V_{\cM, h+1}(s_{h+1}) \Big]\\
        &= \underbrace{\eta \sum_{h=1}^{H} \log \frac{\pi_{\cM, h}(a_h|s_h)}{\pi_{\reff, h}(a_h|s_h)}}_{\text{term $(A)$}} + \underbrace{V_{\cM, 1}(s_1)}_{\text{term $(B)$}} + \underbrace{\sum_{h=1}^{H-1} \Big[V_{\cM, h+1}(s_{h+1}) -\E_{o_h \sim \PP_h(\cdot|s_h, a_h)} V_{\cM, h+1}(s_{h+1}) \Big]}_{\text{term $(C)$}}.
    \end{aligned}
\end{equation}
Here, term $(A)$ is the counterpart of $\eta \log \frac{\pi(a_1|s_1)}{\pi_{\reff}(a_1|s_1)}$ in the single-step DPO derivation and term $(B)$ will be cancelled if we consider the reward difference of two trajectories with the same prompt $s_1 = x$. Unfortunately, in practice, term $(C)$ is typically not feasible to directly compute. Especially, some simple math leads to that with probability at least $0.9$,
$$
|C| \leq 4 \left[\sum_{h=1}^{H-1} \sigma_h^2\right]^{1/2},
$$
where $\sigma_h^2$ is the conditional variance of $V_{\cM, h+1}(s_{h+1}) -\E_{o_h \sim \PP_h(\cdot|s_h, a_h)} V_{\cM, h+1}(s_{h+1})$. Therefore, the bias term $(C)$ is related to the randomness of the external environment. 

For most cases of tool-integrated LLMs for mathematical reasoning, i.e., the focus of this work, luckily the code execution result is determined by the history (the codes written by the LLMs). In other words, given the history $s_h$, the external observation is deterministic, which leads to $\text{term } (C)=0$. Thus, with a dataset $\cD$ consisting of $(x, \tau^w, \tau^l)$, the following multi-turn DPO (M-DPO) loss can be adopted:
\begin{equation}
    \label{eqn:final_m_dpo_loss}
    \mathcal{L}_{\textup{M-DPO}}(\theta) = -\sum_{(x, \tau^w, \tau^l) \in \cD} \log \sigma\Big(\eta \sum_{h=1}^H \Big[\log \frac{\pi_{\theta, h} (a^l_h|s_h^l)}{\pi_{\reff, h}(a^l_h|s_h^l)} - \log \frac{\pi_{\theta, h}(a^w_h|s_h^w)}{\pi_{\reff, h}(a^w_h|s_h^w)}\Big] \Big),
\end{equation} 
We emphasize again that although the loss presented in \eqref{eqn:final_m_dpo_loss} is identical to the one in \eqref{eqn:m_dpo_loss}, a rigorous derivation procedure (rather than a direct plug-in) is provided. To the best of our knowledge, \eqref{eqn:final_m_dpo_loss} is new in the context of multi-turn reasoning task with external messages. In particular, it is noted that such a M-DPO loss is only valid upon deterministic transitions, i.e., term $(C) = 0$. 

Moreover, with \eqref{eqn:optimality_condition} implying that with term $(C)=0$, the implicit reward is given by $A=\eta \sum_{h=1}^{H} \log \frac{\pi^*_{h}(a_h|s_h)}{\pi_{\reff, h}(a_h|s_h)}$, a multi-turn version of KTO \citep{ethayarajh2024kto}, denoted as M-KTO, can also be naturally derived:
\begin{equation}\label{eqn:final_m_kto_loss}
    \mathcal{L}_{\textup{M-KTO}}(\theta) = \E_{x, y \sim \cD} \big[\lambda_y - v(x,y) \big],
\end{equation}
where 
$$
\begin{aligned}
    u_\theta(x, y) &= \eta \sum_{h=1}^{H} \log \frac{\pi_{u, h}(a_h|s_h)}{\pi_{\reff, h}(a_h|s_h)},\\
    z_0 & = \E_{x' \sim \cD, \tau' \sim \pi_\theta(\cdot|x') } \sum_{h=1}^H \KL\big(\pi_\theta(\cdot|s_h), \pi_{\reff}(\cdot|s_h) \big),
\end{aligned}
$$
and
$$
v(x,y) = \left\{
\begin{aligned}
\lambda_+ \sigma \big(\eta (u_\theta(x, y) - z_0) \big) & \qquad \text{if } y \sim y_{desirable} | x\\
\lambda_- \sigma \big(\eta (z_0 - u_\theta(x, y)) \big) & \qquad \text{if } y \sim y_{undesirable} | x
\end{aligned}
\right..
$$
Here $\lambda_+$ and $\lambda_-$ are two hyper-parameters. We notice that \citet{mitra2024orca} developed an online iterative version of KTO for the CoT format reasoning task. Here we extend it to build the tool-integrated reasoning agent.

The above discussions, in particular, M-DPO and M-KTO losses provided in \eqref{eqn:final_m_dpo_loss} and \eqref{eqn:final_m_kto_loss},  are focused on deterministic observations due to the deterministic nature of tool-integrated LLMs for mathematical reasoning. In contrast, some other applications may encounter stochastic observations, e.g., multi-turn chats with the external message provided by a human or another LLM \citep{shani2024multi}. In these scenarios, \eqref{eqn:final_m_dpo_loss} is biased and cannot lead to the optimal policy since $\text{term } (C)\neq 0$. Instead, one should first construct a value network based on the Bellman equations provided in  \eqref{eqn:bellman_eqn0} and \eqref{eqn:bellman_eqn}, similar to the approach in \citet{richemond2024offline}. Subsequently, $\text{term } (C)$ can be estimated using Monte-Carlo methods and serve as an adaptive margin in the preference training. Particularly, the distinctions between direct preference learning algorithms and classical deep RL methods become less clear. The exploration of this more complex algorithm and its application to general multi-turn learning scenarios is left for future research.

We note that the MDP formulation above and related discussions have been previously derived by \citet{zhong2024dpo, rafailov2024r, xie2024exploratory} in the context of either token-wise MDP or more general MDP with deterministic transition but their focuses are all on the single-turn chat tasks. Although the mathematical formulations appear similar, our primary focus lies on tool-integrated reasoning tasks that incorporate additional external messages $\{o_h\}_{h=1}^{H-1}$.

\subsection{Learning with Online Iterative Training}\label{subsec:extension_online}

In the literature of direct preference learning, a long line of work shows that the online single-turn RLHF significantly outperforms their offline counterpart, both in the literature of direct preference learning \citep{xiong2024iterative, ye2024theoretical, guo2024direct, rosset2024direct, dong2024rlhf, tajwar2024preference} and DRL-based approach or rejection sampling fine-tuning \citep{bai2022training, ouyang2022training, touvron2023llama}. Motivated by these successes, we propose to further incorporate online interactive learning to the multi-turn RLHF studied in this work. In the following, we illustrate the proposed ideas from mainly two aspects: two learning objectives and one unified algorithmic framework.


\paragraph{Learning objective.} We consider two different learning objectives. The first one is the KL-regularized target:
\begin{equation} \label{eqn:val_target}
   \max_{\pi}\E_{x \sim d_0} \E_{a_h \sim \pi(\cdot|s_h), o_h \sim \PP^*_h(\cdot|s_h,a_h)} \Big[ u^*(x, y) - \eta \sum_{h=1}^H \KL\big( \pi(\cdot|s_h), \pi_0(\cdot|s_h)\big)\Big],
\end{equation}
i.e., $\max_{\pi} J(\pi; \cM^*, \pi_0)$ where $\cM^* = (\cS, \cA, H, \PP^*, d_0, u^*)$ is the groundtruth environment and $\pi_0$ is the initial policy (e.g., from SFT) that RLHF starts from.
This target is widely adopted in practice \citep{christiano2017deep, ouyang2022training, bai2022training, rafailov2023direct, dong2024rlhf} and requires us to search for the optimal policy only at a \textit{fixed} KL ball centered at the SFT policy $\pi_0$ \citep{xiong2024iterative, ye2024theoretical, xie2024exploratory}.

In contrast, the second one is the non-regularized target, i.e., directly optimizing the reward:  
\begin{equation} \label{eqn:non_regularized_target}
   \max_{\pi} \E_{x \sim d_0} \E_{a_h \sim \pi(\cdot|s_h), o_h \sim \PP^*_h(\cdot|s_h,a_h)} \big[ u^*(x, y) \big] .
\end{equation}
This target is the standard one in canonical RL studies \citep{sutton2018reinforcement}. One motivation for this target is that in the reasoning task, the reward function is more interpretable (e.g. final result checking) compared to the chat task. 

Additionally, we note that a stronger KL regularization in the target \eqref{eqn:val_target} is known to be beneficial for mitigating over-fitting issue and forgetting on the \textit{out-of-domain} tasks \citep{gao2023scaling, lin2023speciality, coste2023reward}. On the other hand, \eqref{eqn:non_regularized_target} allows the model to move more far away, thus achieving a better \textit{in-domain} performance. Thus, from one perspective, the choice between the above two targets can be viewed as a tradeoff between out-of-domain and in-domain performances. This intuition is also verified by later experiments, where optimizing the second target in \eqref{eqn:non_regularized_target} leads to better performance on in-domain test sets. In the rest of this section, we discuss two learning objectives to fully develop the multi-turn preference learning framework. We also conduct an ablation study on these objectives in the experimental section.

\paragraph{Algorithmic framework.} We present a general online iterative algorithmic framework in Algorithm~\ref{alg:online_m_gshf}. Specifically, starting from $\pi_0$, at each iteration, we first collect a pair of trajectories by the current policy pair, where the preference signal is also revealed according to Definition~\ref{def:bt_model}. Then, we update our policy pair given the data collected so far and the next iteration begins. We now discuss some features of the framework as follows.

\textit{Policy choice for exploration-exploitation trade-off.} We update our behavior policies in a non-symmetric way. The first agent, which aims to extract the historical information we have gathered so far, planning with respect to the empirically best model on the historical dataset $\cD$ to get $\pi_t^1$, where the planning algorithms have been discussed in Section~\ref{subsec:practical_plan_alg}, e.g., optimizing the M-DPO or M-KTO loss in \eqref{eqn:final_m_dpo_loss} or \eqref{eqn:final_m_kto_loss}. However, it is widely recognized in RL studies \citep{sutton2018reinforcement,auer2002finite} that simply exploiting the historical data via following the empirically best model is not sufficient to obtain a good final policy, while it is also required to explore the environment so that  new information can be collected to facilitate subsequent learning, i.e., the exploration-exploitation tradeoff. While the main agent targeting exploitation, we design the second agent, in contrast, to strategically incorporate the uncertainty of the future relative to $\pi_t^1$ given the historical information we collect so far into its policy choice. We call the policy of the second agent $\pi_t^2$ as an exploration policy because it serves to explore the underlying environment and facilitate the first agent's learning. In practice, this principle of exploration is generally interpreted as maximizing the difference between the two behavior policies or increasing the diversity of the collected data. We summarize some popular heuristic exploration policy adopted in the online iterative RLHF practice:
\begin{itemize}
    \item Mixture sampling: in the Claude project \citep{Anthropic@claude}, the authors choose to use the checkpoints from different training steps to collect data;
    \item Inference parameters tuning: in the LLaMA project \citep{touvron2023llama}, the authors carefully tune the sampling temperature to balance data diversity and data quality;
    \item West-of-n sampling: \citet{xu2023some, snorkelai@pair, pace2024west, dong2024rlhf} samples n responses per prompt and extract the best one and the worst one (based on some ranking criteria) to construct a preference pair.
    \end{itemize}
We will explore the mixture sampling in the experimental section and also provide a theoretical justification in the next subsection.

\textit{Reference model choice for controlling regularization level.} Despite two different learning targets are discussed in \eqref{eqn:val_target} and \eqref{eqn:non_regularized_target} seperately, we note that one general algorithmic framework can be adopted with the reference model choice taking as a hyper-parameter to control the regularization level and account for the two targets:
\begin{itemize}
    \item KL-regularized target in \eqref{eqn:val_target}: if we fix the reference model as the initial policy, i.e., $\pi_{t,\reff} = \pi_0, \forall t\in [T]$, we always search the optimal policy within the KL ball centered at $\pi_0$, and thus optimize the KL-regularized target.
    \item Non-regularized target in \eqref{eqn:non_regularized_target}: in contrast, inspired by the mirror descent \citep{nemirovskij1983problem}, if we update the reference policy every iteration to be the policy learned in the last iteration, i.e., $\pi_{t, \reff} = \pi^1_{t-1}, \forall t\in [T]$, the cumulative update can make the model to move away from the original $\pi_0$ (while a constraint is made on the per-iteration update magnitude) and we thus optimize the non-regularized target.
\end{itemize}

A graphical illustration is provided in Figure~\ref{fig:kl_ball} to facilitate the understanding.

\begin{figure}[H]
    \centering 
    \includegraphics[width=0.8\linewidth]{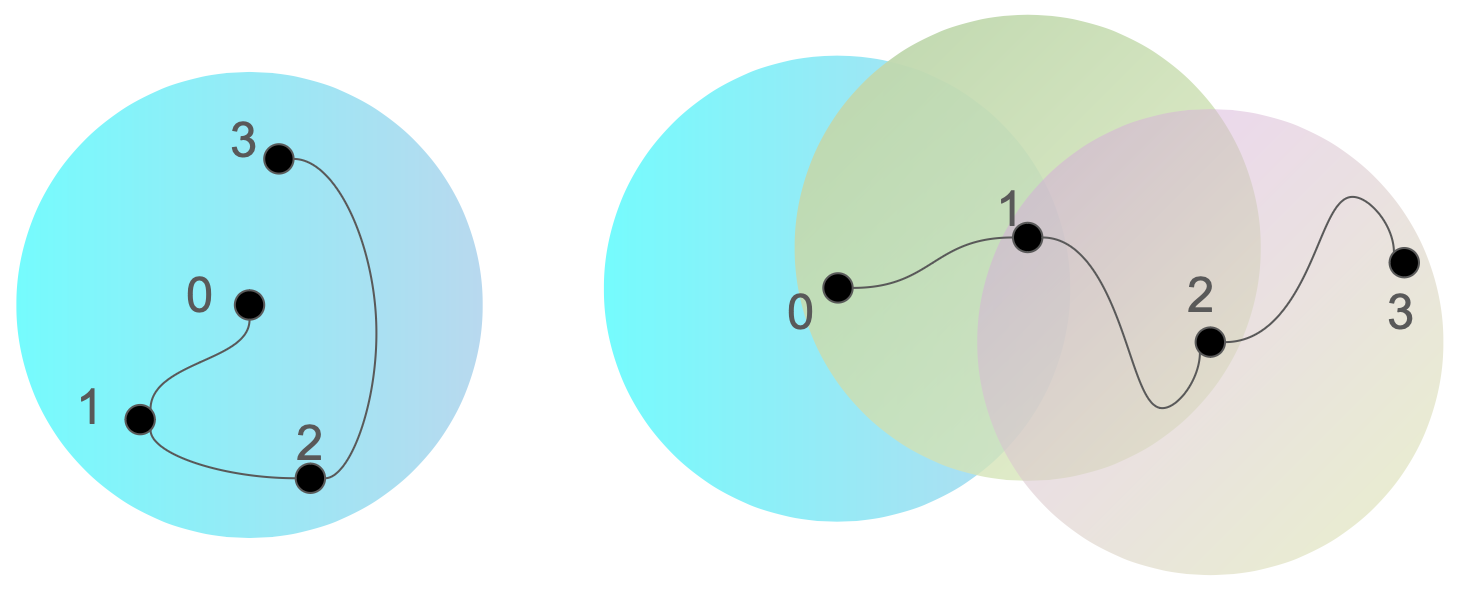}
    \caption{Illustration of the difference between the two learning objectives. The left-hand figure corresponds to the KL-regularized target where we do not update the reference model. The right-hand figure corresponds to the non-regularized target where we always update the reference model as the last-iteration one.}
    \label{fig:kl_ball}
\end{figure}

\begin{algorithm}
\caption{Online Iterative M-GSHF}
\label{alg:online_m_gshf}
\begin{algorithmic}[1]
    \STATE \textbf{Input:} KL coefficient $\eta >0$, horizon $T > 0$, initial policy $\pi_{0}$, batch size $m > 0$.
    \STATE Initialize $\cD \gets \emptyset$ and $\pi_1^1=\pi_1^2 = \pi_{1, \reff} \gets \pi_{0}$.
    \FOR{$t=1,2,\cdots,T$}
    \STATE Sample $m$ pairs $(x, \tau^1, \tau^2, z)$ as $\cD_t$ by $x \sim d_0, \tau^1 \sim \pi_t^1, \tau^2 \sim \pi_t^2$, receive the $m$ preference signals $z$ following the Bradley-Terry model from Definition~\ref{def:bt_model} and update the preference dataset $\cD \gets \cD \cup \cD_t$.
    \STATE \texttt{$\triangleright$ \textbf{Extract the empirically optimal policy from historical data}}
    \STATE \textcolor{magenta}{\textbf{Practical:}} Perform the planning algorithms on $\cD$ to get $\pi_t^1$ (e.g., using the M-DPO loss in \eqref{eqn:final_m_dpo_loss} or the M-KTO loss in \eqref{eqn:final_m_kto_loss})
    \STATE \textcolor{blue}{\textbf{Theoretical:}} Perform MLE on $\cD$ to obtain model estimation $\hat{\cM}_t = (\hat{u}_t, \hat{\PP}_t)$ as in \eqref{eqn:mle_reward} and \eqref{eqn:mle_transition}; call Oracle~\ref{def:gibbs_oracle} with $\hat{\cM}_t, \eta, \pi_{t, \reff}$ to get $\pi_t^1$
    \STATE \texttt{$\triangleright$ \textbf{Select the exploration policy to facilitate learning}}
    \STATE \textcolor{magenta}{\textbf{Practical:}} Given $\pi_t^1$, select $\pi_t^2$ as an exploration policy using heuristic methods  (such as mixture sampling, inference parameters tuning and west-of-n sampling listed in Section~\ref{subsec:extension_online})
    \STATE \textcolor{blue}{\textbf{Theoretical:}} Given $\pi_t^1$, choose $\pi_t^2$ as an exploration policy following \eqref{eqn:exploration_policy_thm}
    \STATE \texttt{$\triangleright$ \textbf{Choose the reference model to  control regularization level}}
    \IF{KL-regularized target in \eqref{eqn:val_target}}
    \STATE Keep $\pi_{t+1, \reff} \gets\pi_{0}$
    \ELSIF{Non-regularized target in \eqref{eqn:non_regularized_target}}
    \STATE Update $\pi_{t+1, \reff} \gets \pi^1_{t}$
    \ENDIF
    \ENDFOR
        \STATE \textbf{Output:} the best model in $\pi^1_{1:T}$ by a validation set.
\end{algorithmic}
\end{algorithm}

\subsection{Theoretical Result}

In this section, we show that the multi-turn RLHF problem can be solved in a statistically efficient manner under standard assumptions in learning theory literature. In particular, for generality, we target the most challenging scenario with stochastic and unknown transitions, while as aforementioned, multi-turn mathematical reasoning with external tools falls into an relatively easier regime with deterministic transitions. Here we mostly studies the KL-regularized target due to the lack of theoretical research on it. The other target of optimizing the rewards has been theoretically studied in \citet{wang2023rlhf} while the techniques of analyzing mirror-descent-style algorithm and corresponding guarantees have also be developed in \citet{cai2020provably}, which can be migrated to considering preference feedbacks. Also, to ease the presentation, we consider the scenario with batch size $m=1$, while the results can be easily generalized to large batches.

First, to measure the online learning process, we define the optimal policy as
\begin{equation} \label{eqn:val_target2}
    \pi^* := \argmax_{\pi} J(\pi):= J(\pi; \cM^*, \pi_{0}),
\end{equation}
and introduce the standard notion of regret as
\begin{equation}\label{eqn:regret}
    \text{Reg}(T):=\sum_{t\in [T]}J(\pi^*) - J(\pi^1_t),
\end{equation}
which represents the cumulative performance loss over $T$ steps comparing the learned policies $[\pi_t^1]_{t=1}^T$ against the optimal policy $\pi^*$. In addition, we consider that a bounded $u^*(x, y) \in [0, B]$ for all $(x, y)$ to maintain a reasonable utillity regime. Also,  it is assumed that we have accesses to the following policy improvement oracle, that is analogue to the one considered in \citet{xiong2024iterative}.

\begin{definition}[Policy Improvement Oracle] \label{def:gibbs_oracle} For any model $\cM = (\cS, \cA, H, \PP, d_0, u)$ and a reference function $\pi_{\reff}$, we can compute the optimal policy associated with the model $[\pi_{\cM, h}]_{h=1}^H$ iteratively as in \eqref{eqn:bellman_eqn}.
\end{definition}

The overall algorithm, i.e., the theoretical version of online iterative M-GSHF, is also summarized in Algorithm~\ref{alg:online_m_gshf}. At each round $t$, with $\cD = \cup_{i =1}^{t-1}\cD_{i}$ as the aggregated dataset, it starts with performing a maximum likelihood estimation (MLE) of the reward function $u^*$ over a set $\cU$, whose elements are bounded in $[0, B]$, as
\begin{equation}\label{eqn:mle_reward}
    \hat{u}_t = \argmax_{\hat{u}\in \cU} L_t(\hat{u}) := \sum_{(x, \tau^1, \tau^2, z) \in \cup_{i =1}^{t-1}\cD_{i}} \Big[z\log(\sigma(\hat{u}(\tau^1) - \hat{u}(\tau^2))) + (1-z)\log(\sigma(\hat{u}(\tau^2) - \hat{u}(\tau^1))) \Big], 
\end{equation}
and also an MLE of the transition kernel $\PP^*$ over a set $\cP$ as
\begin{equation}\label{eqn:mle_transition}
    \hat{\PP}_t = \argmax_{\hat{\PP}\in \cP} L_t(\hat{\PP}) := \sum_{(\pi, \tau) \in \cup_{i =1}^{t-1}\cD_{i}} \log \hat{\PP}^\pi(\tau),
\end{equation}
where $\PP^\pi(\tau)$ denotes the probability of trajectory $\tau$ under policy $\pi$ and transition kernel $\PP$. With the obtained model $\hat{\cM}_t = (\hat{u}_t, \hat{\PP}_t)$, the Oracle defined in Definition~\ref{def:gibbs_oracle} is called with the reference policy $\pi_{\reff}$ set as the initial policy $\pi_{0}$, whose output is adopted as the main policy $\pi^1_t$.

Then, we specify how to choose a theoretically sound exploration policy $\pi_t^2$. The previous work of \citet{xiong2024iterative} on single-turn RLHF has demonstrated the intuition that the exploration policy should be in charge of collecting information of the uncertain parts of the environment $\cM$, which is thus often selected to maximize one uncertainty measurement. In the multi-turn RLHF setup considered in this work, the following proposition serves as the cornerstone to find a suitable uncertainty measurement to decide the exploration policy. In particular, we can observe that the optimal policy is parameterized by the optimal $Q$-function. If a different set of $Q$-function is adopted for policy parameterization, we can bound its performance as follows.
\begin{proposition}[Value Decomposition Lemma for KL-regularized MDP]\label{prop:performance difference}
    If considering a set of $Q$-functions $[\hat{Q}_h]_{h=1}^H$ and a reference policy $\pi_\reff$ with the induced policy $\hat{\pi}$ as
    \begin{align*}
        \hat{\pi}_h(a_h|s_h) \propto \pi_{\reff,h}(a_h|s_h) \cdot \exp\left(\hat{Q}_h(s_h, a_h)/\eta\right),
    \end{align*}
    and the corresponding set of $V$-functions $[\hat{V}_h]_{h=1}^H$ as
    \begin{align*}
        \hat{V}_h(s_h) = \E_{a_h\sim \hat{\pi}_h(\cdot|s_h)}\left[\hat{Q}_h(s_h, a_h)\right] - \eta \KL(\hat{\pi}_h(\cdot|s_h), \pi_{\reff,h}(\cdot|s_h)), \qquad \hat{V}_{H+1}(s_{H+1}) = 0,
    \end{align*}
    for any comparator policy $\pi$, it holds that
    \begin{align*}
        J(\pi) - J(\hat{\pi})& = \E_{d_0,\pi, \PP^*}[u^*(s_H, a_H)] - \E_{d_0, \hat{\pi}, \PP^*}[u^*(s_H, a_H)]\\
        &+ \sum_{h\in [H]}\E_{d_0, \pi, \PP^*} \left[\hat{V}_{h+1}(s_{h+1}) - \hat{Q}_{h}(s_{h}, a_{h})\right] - \sum_{h\in [H]}\E_{d_0, \hat{\pi}, \PP^*}\left[\hat{V}_{h+1}(s_{h+1}) - \hat{Q}_{h}(s_{h}, a_{h})\right]\\
        &- \eta\cdot \sum_{h\in [H]}\E_{d_0, \pi, \PP^*}\left[\KL(\pi_{h}(\cdot|s_{h}), \hat{\pi}_{h}(\cdot|s_{h}))\right],
    \end{align*}
    where the expectation $\E_{d_0, \pi, \PP^*}$ is with respect to the prompt and response (i.e., the trajectory) generated following $d_0, \PP^*$ and $\pi$.
\end{proposition}

Based on Proposition~\ref{prop:performance difference}, the exploration policy $\pi_t^2$ is selected as
    \begin{align}
        \pi_t^2 = \argmax_{\pi}\max_{ \tilde{u}\in \tilde{\cU}_t, \tilde{\PP}\in \tilde{\cP}_t}&\underbrace{\E_{d_0, \pi, \tilde{\PP}}\left[\tilde{u}(s_H, a_H)\right] - \E_{d_0, \pi^1_{t}, \tilde{\PP}}\left[\tilde{u}(s_H, a_H)\right] - \left(\E_{d_0, \pi, \tilde{\PP}}\left[\hat{u}_t(s_H, a_H)\right] - \E_{d_0, \pi^1_t, \tilde{\PP}}\left[\hat{u}_t(s_H, a_H)\right]\right)}_{\text{uncertainty measurement of reward estimation}} \notag\\
        &+ \underbrace{\sum_{h\in [H]}\E_{d_0,\pi, \tilde{\PP}}\left[  \hat{V}_{t, h+1}(s_{h+1})  - \left[\hat{\PP}_{t, h}\hat{V}_{t, h+1}\right](s_h, a_h)\right]}_{\text{uncertainty measurement of transition estimation}} \label{eqn:exploration_policy_thm},
    \end{align}
    where $\tilde{\cU}_t$ and $\tilde{\cP}_t$ are two confidence sets defined as
    \begin{equation}\label{eqn:confidence}
        \begin{aligned}
                \tilde{\cU}_t &= \{u\in \cU: L_t(u) \geq L_t(\hat{u}_t) -  c_1 \log(|\cU|T/\delta)\},\\
    \tilde{\cP}_t &= \{\PP\in \cP: L_t(\PP) \geq L_t(\hat{\PP}_t) - c_1 \log(|\cP|T/\delta)\}
        \end{aligned}
    \end{equation}
with $c_1$ denoting an absolute constant here. Note that for the theoretical convenience, we have assumed $\cU$ and $\cP$ are finite here, which can be extended to the infinite case using standard discretization techniques.
It can be observed that $\pi_t^2$ is selected to maximize a combination of uncertainty from estimations of both rewards and transitions. If considering known transitions (i.e., without the need to estimate $\PP$), the uncertainty from the estimation of transitions dimimishes, which leads to a similar uncertainty measurement adopted in \citet{xiong2024iterative}.

The following theorem establishes a rigorous guarantee for the regret incurred.
\begin{theorem}\label{thm:online_guarantee}
    Assuming $u^*\in \cU$ and $\PP^*\in \cP$, with probability at least $1-\delta$, we have that
    \begin{align*}
        \textup{Reg}(T)  \lesssim & \kappa^{-1}B\sqrt{d_{\cU}T\log(|\cU|T/\delta)} + B^2H\xi(d_\cP, T, c_2\log(|\cP|HT/\delta))\\
        &- \eta\cdot \sum_{h\in [H]}\E_{d_0, \pi^*, \PP^*}\left[\KL(\pi^*_{h}(\cdot|s_{h}), \pi^1_{t, h}(\cdot|s_{h}))\right],
    \end{align*}
    where $\kappa:= 1/(2+ \exp(-B)+ \exp(B))$, $c_2$ is an absolute constant, $d_{\cU}$ is the Eluder coefficient defined in Definition~\ref{def:eluder_coefficient} while $d_\cP$ and $\xi(\cdot)$ are from the generalized Eluder-type condition defined in Definition~\ref{def:eluder_condition}.
\end{theorem}
We note that the Eluder coefficient  and the generalized Eluder-type condition are standard and well-adopted conditions in the theoretical studies on RL \citep{zhang2023mathematical, zhong2022gec, liu2023optimistic, xie2022role, agarwal2023vo} and also RLHF \citep{zhan2023provable,wang2023rlhf, ye2024theoretical}. Moreover, for a board class of RL problems (see \citet{zhang2023mathematical,liu2023optimistic} for more details), the Eluder coefficient $d_{\cU}$ is small and the condition is satisfied with $\xi(d_\cP, T, c_2\log(|\cP|HT/\delta)) \lesssim \sqrt{d_{\cP}T\log(|\cP|HT/\delta)}$, which implies that the regret of theoretical version of Algorithm~\ref{alg:online_m_gshf} is sublinear in $T$, further evidencing its statistical efficiency.

\section{Experiments}

\subsection{Experiment Setup}
\paragraph{Task, and datasets.} We use the test sets of MATH \citep{hendrycks2021measuring} and GSM8K \citep{cobbe2021gsm8k} to measure the model's ability to solve the mathematical problems. The MATH dataset includes 5K problems across diverse mathematical fields such as algebra, geometry, probability, number theory, and calculus. The GSM8K test set consists of 1319 grade-school math word problems, which are generally simpler than those in the MATH dataset. Examples from each dataset are as follows:
\begin{itemize}
    \item GSM8K: Natalia sold clips to 48 of her friends in April, and then she sold half as many clips in May. How many clips did Natalia sell altogether in April and May?
    \item MATH: Find the center of the circle with equation $x^2 - 6x + y^2 + 2y = 9$.
\end{itemize}
To effectively solve these problems, the model needs to perform multi-turn reasoning and arithmetic operations before getting the final answer. To construct the training prompt set, we follow \citet{gou2023tora, yu2023metamath, yue2023mammoth, liu2024augmenting, toshniwal2024openmathinstruct} to use an augmented prompt set from the 7.5K training problems of MATH and 7.47K training problems of GSM8K. In particular, we use the prompts from MetaMathQA \citep{yu2023metamath} and MMIQC \citep{liu2024augmenting}. The new questions include rephrasing question, backward question (starting with the final answer and thinking backward to determine an unknown variable in the original question), and bootstrapping questions by in-context learning and iterative question composing \citep{liu2024augmenting}. We delete the duplicate questions and also ensure that none from the test sets of MATH and GSM8K were used. Eventually, we have 60K training prompts in total for training and randomly split them into three disjoint sets for iterative training. We also reserve a set of 1K prompts for model selection during the training.

\paragraph{Base models.} We train with a range of base models, including Gemma-1.1-it-7B \citep{team2024gemma}, CodeGemma-1.1-it-7B \citep{team2024codegemma}, Mistral-7B-v0.3 \citep{jiang2023mistral}, and Gemma2-it-9B. We use the pre-trained version of Mistral instead of the instruction version because the chat template of its huggingface checkpoint and that of their own code base are not consistent so we start from the pre-trained model and fine-tune it by ourselves.

\paragraph{Data format and generation.} We format the data into a multi-turn chat where the user initially ask the LLMs a question, and provide the messages returned by the Python interpreter in the subsequent user rounds of chat. In each model turn, the model reasons based the history gathered so far and can output a final answer enclosed in $\backslash$boxed, or call the Python interpreter by writing a code wrapped in ```python and ```. After receiving the response of the model, we return the execution result of the code if the model calls the tool, and stop if the model outputs the final answer or reaches the maximal number of rounds H (6 in our setting). See Table~\ref{tab:multi_turn_math} for an illustration. We generated N=30 samples per prompt for each iteration using a temperature setting of 1.0, without employing top-K or top-p sampling. We employ a mixture sampling strategy, where the up-to-date model generates only 20 trajectories, and the remainder (10 trajectories) are collected using the model from the last iteration. For the initial iteration, we employed models fine-tuned for 3 epochs and 1 epoch, respectively, to conduct mixture sampling. Intuitively, the mixture sampling helps to improve the diversity of the collected samples, and have been employed in previous RLHF practices \citep{bai2022training, dong2024rlhf}. For all the data generation process, we adopt the following constraints: (1) for each turn, the model can generate up to 512 tokens; (2) the maximal number of steps is H=6; (3) the maximal number of token for each trajectory is 2048. 


\paragraph{Supervised fine-tuning (SFT).} We first fine-tune the model for the tool-integrated reasoning task \citep{gou2023tora}, using a subset of the Open-MathInstruct dataset, which was generated by the permissively licensed \texttt{Mixtral-8x7B} model through in-context learning. The problems are from the training sets of MATH and GSM8K datasets. We restrict the number of samples for each question to be $50$ and remove the nearly duplicate responses. Eventually we get 510K samples in the SFT dataset. We train the models for 4 epochs at most with a learning rate of 5e-6 for Gemma instruct models \citep{team2024gemma} and a learning rate of 1e-5 for Mistral-v0.3 model \citep{jiang2023mistral}. The learning rates are determined by searching \{2e-6, 5e-6, 1e-5\}. We use the pretrained model of Mistral because the chat template of Mistral instruct models are not consistent in different code bases (huggingface and the official one) at the time of our experiments. We use a cosine learning rate scheduler and set the warm-up steps as 100. The samples are packed into blocks with length 4096 to accelerate training and a global batch size of 64 is used. We also mask all the user messages (i.e., the prompt and the messages returned by the Python interpreter) in the training. It takes roughly 10-15 hours when training with 8xA100 80G GPUs. The checkpoint at the end of the third epoch is used for Gemma and the checkpoint of the end of the second epoch is used for Mistral as the starting point for RLHF. This is because these models outperform the last-iteration one with considerable margin and is very close to the next one.  An ablation study on the SFT epochs is also included.

\paragraph{Data Annotation.} For each prompt, we first divide the responses into the winning set $G^w$ and the losing set $G^l$ by checking the final answer. In practice, we observe that the model can memorize the final answer and output it even though the reasoning path itself is incorrect. To mitigate this issue, we include some heuristic filtering process. First, we delete all the trajectories in the winning set where the returned messages in the second last round indicate the code is with some bugs, but the models just ignore it and predict the ground-truth answer. Then, we delete the responses in both the winning set $G^w$ and losing set $G^l$ if they are longer than 2048 tokens. Finally, we randomly sample a trajectory from the $G^w$ and a trajectory from $G^l$ to construct a pair or to add them into the training set of KTO algorithm. For each iteration, we typically get 15K-20K samples because some of the prompts may not have any correct answer. We notice that it is possible to leverage AI feedback like Gemini \citep{team2023gemini} or GPT4 \citep{OpenAI2023GPT4TR} to further verify the correctness of the trajectory step by step or construct a PRM \citep{lightman2023let, wang2023math} to rank the trajectories, which we leave for future work.

\paragraph{Implementation of M-DPO and M-KTO.} To implement the M-DPO, we simply set the labels of all the user-turn tokens to be -100 and mask the log-probability in the subsequent loss computation. We train the model for 1 epoch at most and tune the learning rate in \{2e-7, 4e-7, 7e-7, 1e-6\} with the first iteration of iterative training. Eventually, the learning rate of 4e-7 is used for Gemma-1.1 models and 2e-7 is used for Gemma-2 model and Mistral model. The global batch size is 32 with a warm-up step of 40. We evaluate the model every 50 training steps by the split prompt set and the best model is typically obtained between 150 steps to 600 steps, which is expected because the prompts for SFT and prompts for RLHF are overlapped. This has also been observed in previous work of RLHF for making general chatbot \citep{lin2023speciality}. Further exploration of prompt scaling is also left for future work. The hyper-parameters are of M-KTO are mostly the same as the M-DPO. We also set the $\lambda_+ = \lambda_- = 1$ following the original KTO paper \citep{ethayarajh2024kto}. The RLHF experiments of this paper are run with 8xA100 80G GPUs, where an additional machine with 8xA100 40G GPUs is also used for data collection and model evaluation. The main experiment of this paper can be reproduced by 24 - 48 hours with this setup. We defer some other implementation details to Appendix~\ref{appendix:implementation} due to space constraint.

\subsection{Main Results}

We evaluate the models in the zero-shot setting and report the main results in Table~\ref{tab:main_result}. 

\paragraph{Competitors.} The existing literature mainly focuses on the synthetic data generation and teach the models to use the external tool by supervised fine-tuning on the collected data. We use the results from \citet{toshniwal2024openmathinstruct} as baselines because we use the same SFT dataset so the results are generally comparable. For the CoT baselines, we use the Wizardmath models from \citet{luo2023wizardmath}. We also include the reward ranked fine-tuning (RAFT) as a baseline \citep{dong2023raft}, which is also known as rejection sampling fine-tuning in the literature \citep{touvron2023llama}. RAFT first collects N trajectories per prompt, filters the low-quality data (by reward function), and fine-tune on the selected trajectories. Another baseline is the single-turn online iterative DPO and KTO \citep{rafailov2023direct, ethayarajh2024kto}, which ignores the problem structure (i.e., the external messages) and treats the trajectory as a whole. In implementation, it means that we do not mask the user turn and the tokens of external messages also contribute to the loss.

\begin{table}[htp]
    \centering \small
    \caption{Main results of different methods on the test sets of GSM8K and MATH. The SFT training with external tool is based on (a subset of) Open-MathInstruct so the results are generally comparable to the previous SFT models.     
    $\dagger$: the model also serves as the starting checkpoint of other methods except for prompting and CoT without tool use. All the models are allowed to use code interpreter except for the CoT without tool use. The results of the CoT methods are borrowed from the technical reports \citep{toshniwal2024openmathinstruct, gou2023tora}. The gains relative to the SFT starting checkpoint are marked by \up{}. } \vspace{4pt}\label{tab:main_result}
    \begin{tabular}{ccc|ccc}
    \toprule
    \textbf{Base Model} & \textbf{Method} & \textbf{with Tool} & \textbf{GSM8K} & \textbf{MATH} & \textbf{AVG}\\ \midrule
    WizardMath-7B & SFT for CoT & \textcolor{red}{\ding{55}} & 54.9 & 10.7 & 32.8 \\
    WizardMath-13B & SFT for CoT & \textcolor{red}{\ding{55}} & 63.9 & 14.0 & 39.0\\
    WizardMath-70B & SFT for CoT & \textcolor{red}{\ding{55}} & 81.6 & 22.7 & 52.2\\\midrule
    CodeLLaMA-2-7B & SFT &\textcolor{green}{\ding{51}} & 75.9 & 43.6  & 59.8 \\ 
    CodeLLaMA-2-13B & SFT & \textcolor{green}{\ding{51}} &  78.8 & 45.5 & 62.2 \\ 
    CodeLLaMA-2-34B & SFT & \textcolor{green}{\ding{51}} &80.7 & 48.3 & 64.5 \\ 
    LLaMA-2-70B & SFT & \textcolor{green}{\ding{51}} & 84.7 & 46.3 & 65.5 \\
    CodeLLaMA-2-70B & SFT & \textcolor{green}{\ding{51}} & 84.6 & 50.7 & 67.7\\ 
    \midrule
    Gemma-1.1-it-7B & SFT$^\dagger$ & \textcolor{green}{\ding{51}}& 77.5 & 46.1 & 61.8 \\
    Gemma-1.1-it-7B & RAFT & \textcolor{green}{\ding{51}} & 79.2 & 47.3 & 63.3\\
    Gemma-1.1-it-7B & Iterative Single-turn DPO &\textcolor{green}{\ding{51}} & 81.7 & 48.9	&65.3\\
    Gemma-1.1-it-7B & Iterative Single-turn KTO &\textcolor{green}{\ding{51}} & 80.6 & 49.0	& 64.8 \\
 Gemma-1.1-it-7B & Iterative M-DPO + fixed reference  & \textcolor{green}{\ding{51}} & 79.9  & 
    48.0& 64.0\\
    Gemma-1.1-it-7B & M-DPO Iteration 1 & \textcolor{green}{\ding{51}}& 81.5 & 49.1	& 65.3\\
    Gemma-1.1-it-7B & M-DPO Iteration 2 & \textcolor{green}{\ding{51}} & 82.5 & 49.7 & 66.1	\\
   \rowcolor[rgb]{ .867, .922, .969} Gemma-1.1-it-7B & M-DPO Iteration 3 & \textcolor{green}{\ding{51}} & {83.9} \up{6.4} & 51.2 \up{5.1} & 67.6 \up{5.8}  \\
      \rowcolor[rgb]{ .867, .922, .969} Gemma-1.1-it-7B & Iterative M-KTO & \textcolor{green}{\ding{51}} &  82.1 \up{4.6} &  49.5 \up{3.4} & 65.8 \up{4.0} \\\midrule
    CodeGemma-1.1-it-7B & SFT$^\dagger$ & \textcolor{green}{\ding{51}} &77.3 & 46.4 &61.9  \\
    CodeGemma-1.1-it-7B & RAFT & \textcolor{green}{\ding{51}} & 78.8 & 48.4 & 63.6 \\
    CodeGemma-1.1-it-7B & Iterative Single-turn DPO & \textcolor{green}{\ding{51}} & 79.1 & 48.9 & 64.0 \\
    CodeGemma-1.1-it-7B & Iterative Single-turn KTO & \textcolor{green}{\ding{51}} & 80.2 & 48.6 & 64.4\\
  \rowcolor[rgb]{ .867, .922, .969}  CodeGemma-1.1-it-7B & Iterative M-DPO  & \textcolor{green}{\ding{51}} & {81.5} \up{4.2} & {50.1} \up{3.7}	& 65.8 \up{4.0} \\
  \rowcolor[rgb]{ .867, .922, .969}  CodeGemma-1.1-it-7B & Iterative M-KTO  &  \textcolor{green}{\ding{51}} & 81.6 \up{4.3} & 49.6 \up{3.2} & 65.6	\up{3.8}\\
    \midrule
    Mistral-7B-v0.3 & SFT$^\dagger$ & \textcolor{green}{\ding{51}} & 77.8 & 42.7 & 60.3 \\
    Mistral-7B-v0.3 & RAFT & \textcolor{green}{\ding{51}} & 79.8 & 43.7 & 61.8 \\
    Mistral-7B-v0.3 & Iterative Single-turn DPO & \textcolor{green}{\ding{51}} & 79.8 & 45.1& 62.5	\\
    Mistral-7B-v0.3 & Iterative Single-turn KTO & \textcolor{green}{\ding{51}} & 81.3 & 46.3 & 63.8	\\
   \rowcolor[rgb]{ .867, .922, .969} Mistral-7B-v0.3 & Iterative M-DPO  & \textcolor{green}{\ding{51}} & {82.3} \up{4.5} & {47.5} \up{4.8} & 64.9	\up{4.7}\\
   \rowcolor[rgb]{ .867, .922, .969} Mistral-7B-v0.3 & Iterative M-KTO  & \textcolor{green}{\ding{51}} & 81.7 \up{3.9} & 46.7 \up{4.0} & 64.2 \up{4.0}	\\
        \midrule
    Gemma-2-it-9B & SFT$^\dagger$ & \textcolor{green}{\ding{51}} & 84.1 & 51.0  & 67.6 \\
    Gemma-2-it-9B & RAFT & \textcolor{green}{\ding{51}} & 84.2 & 52.6 & 68.4 \\
    Gemma-2-it-9B & Iterative Single-turn DPO & \textcolor{green}{\ding{51}} & 85.2 & 53.1 & 69.2\\
    Gemma-2-it-9B & Iterative Single-turn KTO & \textcolor{green}{\ding{51}} & 85.4 & 52.9 & 69.2 \\
  \rowcolor[rgb]{ .867, .922, .969}  Gemma-2-it-9B & Iterative M-DPO  & \textcolor{green}{\ding{51}} & \textbf{86.3} \up{2.2} & \textbf{54.5} \up{3.5}& 70.4 \up{2.9} \\
    \rowcolor[rgb]{ .867, .922, .969}  Gemma-2-it-9B & Iterative M-KTO  & \textcolor{green}{\ding{51}} & 86.1 \up{2.0} & 54.5 \up{3.5}	& 70.3 \up{2.8} \\
    \bottomrule
    \end{tabular}
    \end{table}

From the first two sections in Table~\ref{tab:main_result}, we first observe that the tool-integrated LLMs significantly outperform their CoT counterparts with only SFT, demonstrating the benefits of leveraging external tools. In the subsequent discussions, we focus on the comparison within the scope of tool-integrated LLMs. 

\paragraph{Iterative M-DPO and M-KTO considerably improve the SFT models.} We observe that for all the four base models, after the iterative training with M-DPO or M-KTO, the resulting model outperforms their starting SFT checkpoint with considerable margins on both GSM8K and MATH. In particular, with M-DPO, the aligned Gemma-1.1-it-7B model attains accuracies of 83.9\% and 51.2\% on GSM8K and MATH, respectively, and is comparable to the open-source Open-MathInstruct-finetuned CodeLLaMA-2-70B (slightly worse on GSM8K but also slightly better on MATH). Moreover, the aligned Gemma-2-it-9B model achieves accuracies of 86.3\% and 54.5\% on GSM8K and MATH, surpassing all of the open-source models trained with Open-MathInstruct in the 7B to 70B range. Overall, our framework can robustly further boost the tool-integrated models' ability on the top of supervised fine-tuning. 

\paragraph{Iterative M-DPO and M-KTO surpass existing RLHF baselines.} We also observe that the iterative M-DPO and M-KTO surpass other existing RLHF baselines. First, they consistently and significantly outperform the RAFT algorithm across all four base models, which is known to be a robust and competitive baseline in the literature \citep{dong2023raft, yuan2023scaling}. This is because the RAFT algorithm only utilizes the positive signal by imitating the correct trajectories, while the DPO-based and KTO-based algorithms further leverage the negative signal from those incorrect trajectories. We note that the SFT stage in our pipeline can also be viewed as an application of RAFT, an idea that further dates back to expert iteration \citep{anthony2017thinking}. Consequently, our results should be interpreted to be that on the top of the first stage of SFT, algorithms with negative signal are more sample efficient. Moreover, while the online iterative single-turn DPO (KTO) \citep{xiong2024iterative, xu2023some} also gives a boost performance, it is generally worse than the multi-turn version. This suggests that learning to predict the off-policy external messages returned by the code interpreter usually has a negative impact on the reasoning ability improvement. Essentially, this corresponds to the fact that when deriving the optimality condition of the KL-regularized optimization problem, we are not allowed to optimize the external messages. Meanwhile, we present a representative example we encounter in Table~\ref{tab:example_observation}, where LLMs generate poorly constructed code resulting in anomalous and lengthy external messages. Forcing LLMs to learn to predict these messages can significantly hurt the model's reasoning abilities.

\paragraph{Iterative training and reference update lead to better performance.} We use the Gemma-1.1-it-7B and M-DPO as a representative example and observe that the model benefits from online iterative training, where the test accuracy of GSM8K improves from 77.5\% (SFT) to 81.5\% (iter 1) to 82.5\% (iter2) to 83.9\% (iter3), and the test accuracy of MATH improves from 46.1\% (SFT) to 49.1\% (iter 1) to 49.7\% (iter2) to 51.2\% (iter3). This is consistent with our theoretical insight that iterative training allows the models to explore the underlying space and learn the optimal policy progressively. Moreover, we observe that if we fix the reference model as the SFT policy, the final model performance is much worse compared to the case where we update the reference model as the current model at every iteration. We suspect that this is because this version of algorithm essentially optimizes the non-regularized reward and the reward in the mathematical reasoning task is more accurate than those in the general chat task, leading to the superior in-domain performance. We defer a more detailed ablation study on the impact of KL regularization to next subsection.

\begin{figure}[H]
    \centering
    \includegraphics[width=0.49\linewidth]{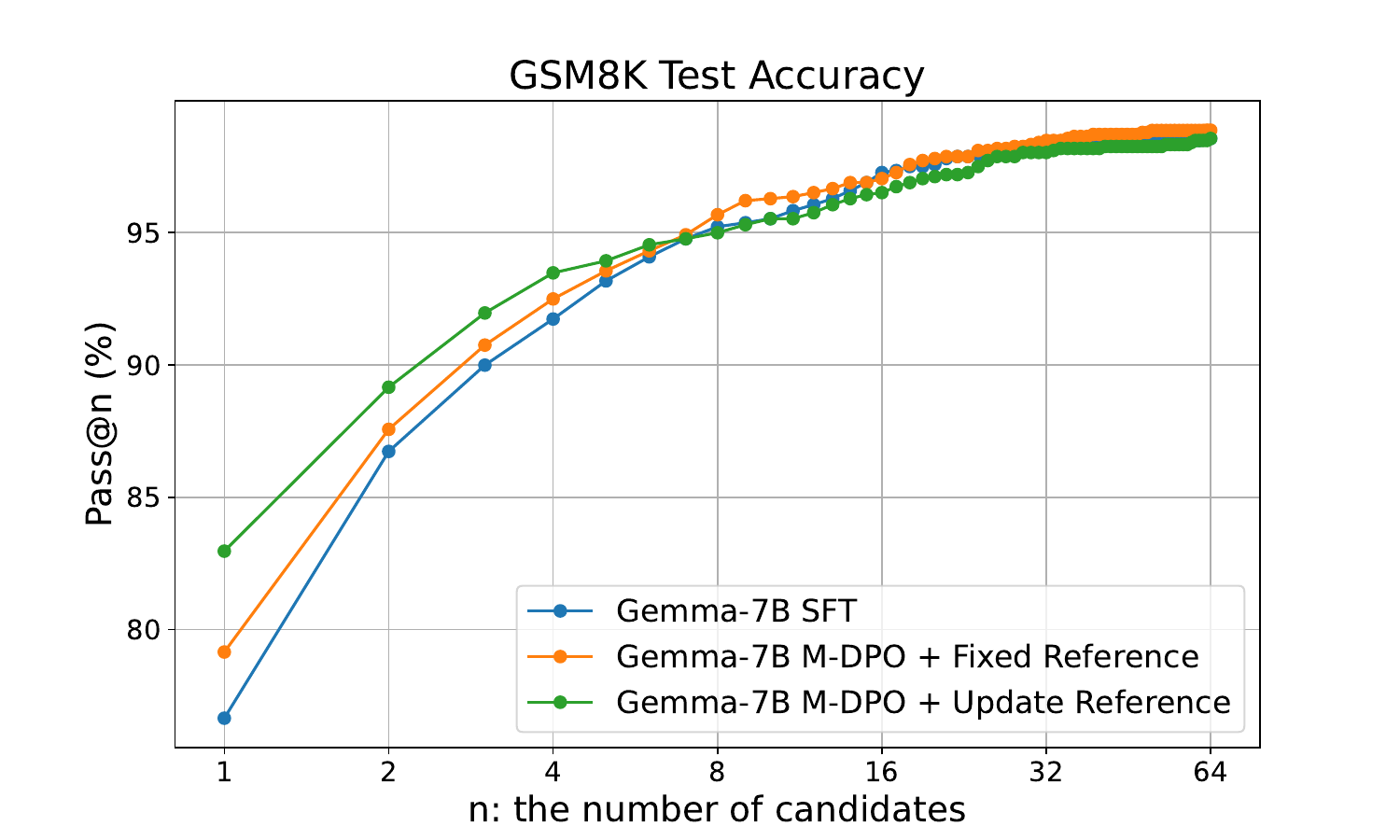}
    \includegraphics[width=0.49\linewidth]{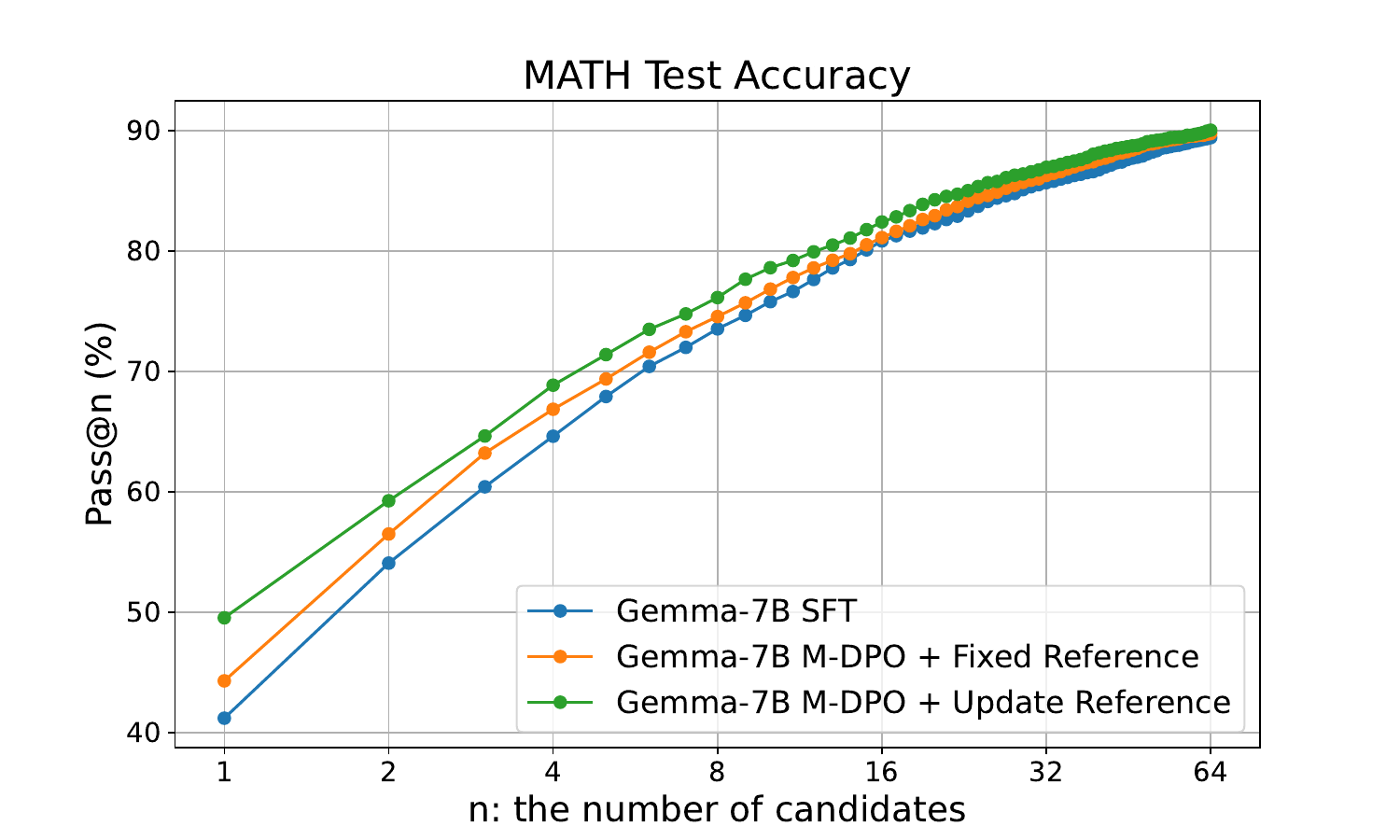}
    \caption{The pass@n rate with respect to the number of candidates n. We evaluate the models using temperature 0.7 following the previous works \citet{shao2024deepseekmath, toshniwal2024openmathinstruct}. We notice that preference learning only improves the metric pass@n when n is relatively small.}
    \label{fig:passn}
\end{figure}

\paragraph{Preference learning improves pass@n only when n is relatively small.} We plot the pass@n accuracy in terms of the number of candidate trajectories n in Figure~\ref{fig:passn}. To evaluate the pass@n, for each question, we independently sample n trajectories, and the question is considered to be solved if there \textit{exists} at least one trajectory with the correct final answer. We observe that the preference learning only improves the pass@n when n is relatively small. In particular, when $n>16$, all the models perform similarly on both GSM8K and MATH. In other words, the iterative M-DPO does not inject new knowledge but elicits the models' knowledge acquired in pre-training and SFT stages by boosting the quality of Top n responses. The observation is consistent with that of \citet{shao2024deepseekmath}, which studies the DRL-based GRPO method for the CoT mathematical reasoning task. Therefore, the success of preference learning is on top of a well-trained SFT model. We expect that the final model performance can be further improved with more high-quality SFT data.
    
\subsection{Ablation Study and Discussion}
We conduct ablation studies in this subsection for a more comprehensive understanding of the proposed algorithm.

\paragraph{A moderate level of KL regularization balances the per-iteration improvement and exploration efficiency.}  The effectiveness of (iterative) DPO is significantly influenced by the choice of reference model and the KL coefficient. Previous research by \citet{tunstall2023zephyr} on offline DPO for general chatbot applications suggests that a lower KL coefficient, specifically 0.01, yields superior performance by allowing the resulting model to move more far away from the SFT model $\pi_0$. Meanwhile, for online iterative training, \citet{xiong2024iterative, dong2024rlhf} adopt a fixed reference model of $\pi_0$, and achieves continuous improvements as the training iterates. In our ablation study, we consider two different choices: (1) using the fixed reference model $\pi_0$; (2) updating the reference model to the last iteration's model at each round. Moreover, we search the KL coefficient $\eta \in \{0.01, 0.1, 0.5\}$. The results are summarized in Table~\ref{tab:ablation_kl}. First, we notice that if we update the reference model at each iteration, the final model outperforms the one with a fixed reference model $\pi_0$ with a large margin. Essentially, this dynamic approach optimizes the non-regularized reward, while the one with a fixed reference model $\pi_0$ aims to maximize the KL-regularized reward. This can be viewed as a trade-off between the generation diversity and reward optimization. We suspect this performance difference is because for reasoning task, the correct reasoning paths are highly concentrated on a small subset of the generation space, and the diversity is less important in this case.

We also find that the strongest model is obtained by a moderate KL coefficient of $0.1$, outperforming both 0.01 and 0.5. To understand this phenomena, we plot the test accuracy of GSM8K in Figure~\ref{fig:gsm8k} along the way of iterative training. As we can see, for the first iteration, the results align with \citet{tunstall2023zephyr}'s findings, where a smaller KL coefficient leads to a larger model improvement. However, the resulting intermediate model is further used to collect trajectories for subsequent iterative training. The models trained with very low KL coefficients tend to lose diversity rapidly, potentially reducing their capacity to collect diverse trajectories for subsequent training, leading to diminishing gains in the second and third iterations. In contrast, a higher KL coefficient of 0.5 imposes strong regularization between the resulting model and the reference model, and the model improvement is less compared to that of $0.1$ for each iteration. To summarize, for online iterative training, we need to strike a balance between the per-iteration improvement and exploration efficiency to optimize the overall performance. We will see that such an intuition also extends to the choices of sampling strategy choice and other experimental tricks.

\begin{table}[H]
    \centering 
    \caption{Ablation study of the impact of KL regularization. The SFT policy is the starting checkpoint for all other experiments.}\vspace{4pt}\label{tab:ablation_kl}
    \begin{tabular}{cc|ll}
    \toprule
    \textbf{Model} & \textbf{Method} & \textbf{GSM8K} & \textbf{MATH} \\ \midrule
    Gemma-1.1-it-7B & SFT 3 epoch & 77.5 & 46.1  \\
    Gemma-1.1-it-7B & Iterative M-DPO + $\eta=0.01$  & 81.7  & 50.1 	\\
    Gemma-1.1-it-7B & Iterative M-DPO + $\eta=0.1$ & 83.9 & 51.2	\\
    Gemma-1.1-it-7B & Iterative M-DPO + $\eta=0.5$  & 82.8  & 49.7 	\\
\midrule
    Gemma-1.1-it-7B & Iterative M-DPO + fixed reference + $\eta=0.1$  & 79.9  & 48.0 	\\
    \bottomrule
    \end{tabular}
    \end{table}

\begin{figure}[H]
    \centering
    \includegraphics[width=0.8\linewidth]{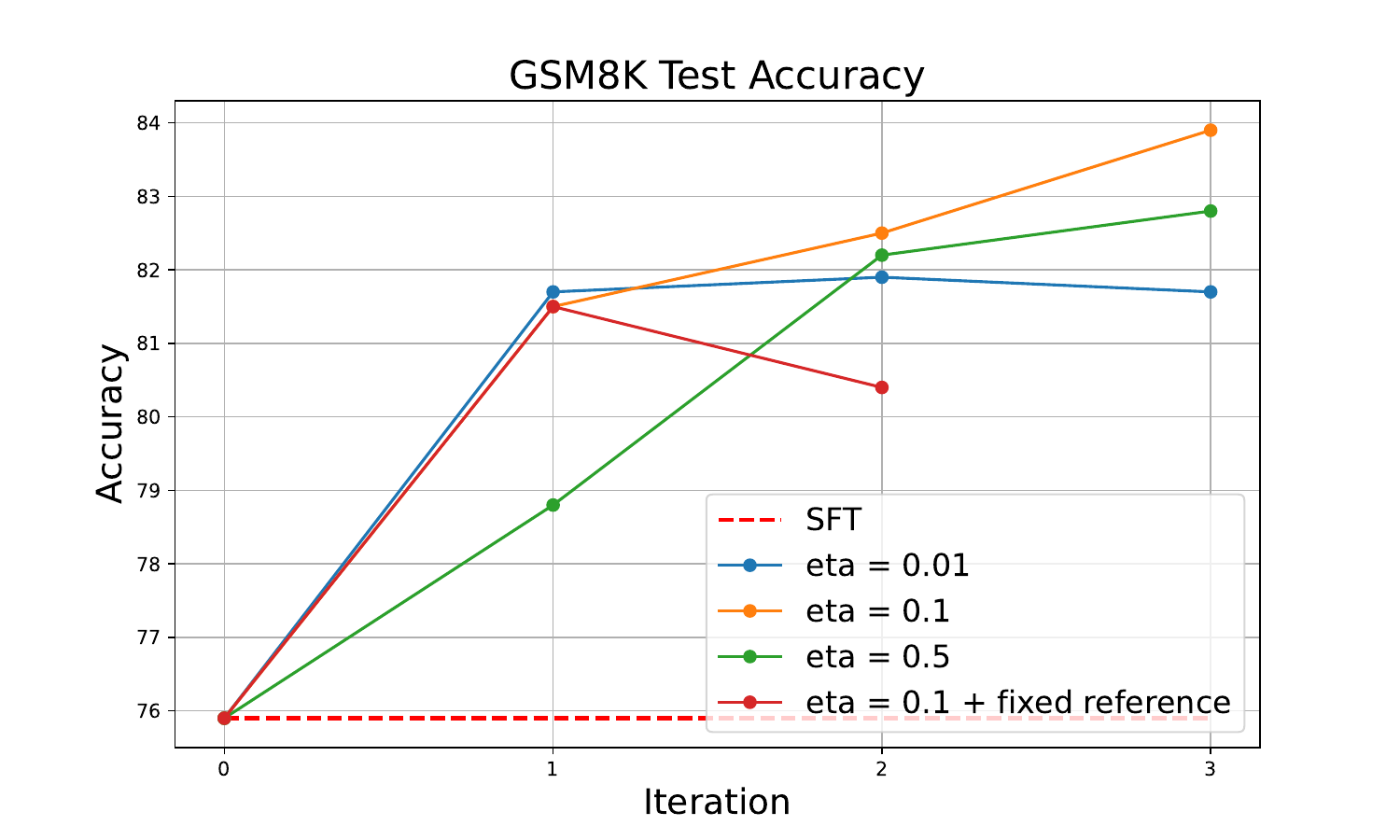}
    \caption{The plot of test accuracy on GSM8K dataset and iterations with different levels of KL regularization.}
    \label{fig:gsm8k}
\end{figure}

\paragraph{The impact of sampling strategy: data diversity and coverage are crucial.} Throughout our iterative training process of the Gemma-1.1-it-7B, we observed a steady increase in the percentage of correct trajectories—from 47\% in the first iteration to 76\% in last iteration. Moreover, since we update the reference model at each iteration, the diversity of the generated trajectories also decrease rapidly. However, the diversity of the collected data is critical for DPO/KTO training due to their contrastive nature. Prior studies in online iterative DPO for general chatbots \citep{dong2024rlhf} recommend employing model variants with different sampling temperatures or training steps to enhance trajectory diversity. Motivated by this, we explored two data collection strategies: (1) on-policy sampling, where all trajectories are sampled using the current policy, and (2) mixture sampling, where 20 trajectories are collected using the current model and 10 from the last iteration's model. We report the results in Table~\ref{tab:misc_ablation}, where with mixture sampling, the final model performance considerably outperform the one with only on-policy sampling. To understand this phenomena, we plot the MATH test accuracy in terms of the iteration in Figure~\ref{fig:math_sampling}. We observe that on-policy sampling fails to improve MATH test accuracy in the third iteration, while we achieve considerable gain with the mixture sampling. This again demonstrates the importance of the diversity of the collected responses in the iterative training and also aligns with previous findings that advanced exploration strategies, which prevent diversity collapse, provide more meaningful signals for iterative preference learning \citep{bai2022training, touvron2023llama, xiong2024iterative, pace2024west, dong2024rlhf}. It would be interested to explore more advanced exploration strategy like Monte Carlo tree search (MCTS) in the future study. 

In our experiments, we collected N trajectories per prompt to ensure the presence of both correct and incorrect reasoning paths for constructing the comparison pair. A larger N generally leads to a better coverage of the prompt set because for some difficult problem, we need to sample more responses to find a correct reasoning path. For instance, in iteration 1, with N=30, 92.5\% of the prompts are covered, compared to 83.0\% for N=12 and 60\% for N=6. See Figure~\ref{fig:passn} for an illustration of the relationship between pass@1 and N. However, increasing N also incurs higher computational costs. To understand the impact of the parameter N, we conduct an ablation study with $N \in \{6, 12, 30\}$ and summarize the results in Table~\ref{tab:ablation_N}. We observe a substantial performance boost when increasing N from 6 to 12, reflecting a better coverage of the complex problems that require more attempts to find a correct path. In contrast, from N=12 to N=30, we only get very minor improvement in the test accuracy, suggesting that the incremental benefits of increasing N in best-of-N sampling diminish rapidly.

\begin{table}[H]
    \centering 
    \caption{Ablation study of the impact of sampling strategy. The SFT policy is the starting checkpoint for all other experiments. Mixture sampling is adopted for the iterative M-DPO training by default and we run for three iterations in total.}\vspace{4pt}\label{tab:ablation_N}
    \begin{tabular}{cc|cc}
    \toprule
    \textbf{Model} & \textbf{Method} & \textbf{GSM8K} & \textbf{MATH} \\ \midrule
    Gemma-1.1-it-7B & SFT 3 epoch &  77.5 & 46.1  \\
    Gemma-1.1-it-7B & Iterative M-DPO with N=30 	& 83.9 & 51.2\\
    Gemma-1.1-it-7B & Iterative M-DPO with N=12 & 83.5 & 51.2  \\
    Gemma-1.1-it-7B & Iterative M-DPO with N=6 & 82.0 & 49.2 \\
    Gemma-1.1-it-7B & Iterative M-DPO with N=30 + On-policy sampling 	& 83.1 & 49.5\\
    \bottomrule
    \end{tabular}
    \end{table}

\begin{figure}[H]
    \centering
    \includegraphics[width=0.8\linewidth]{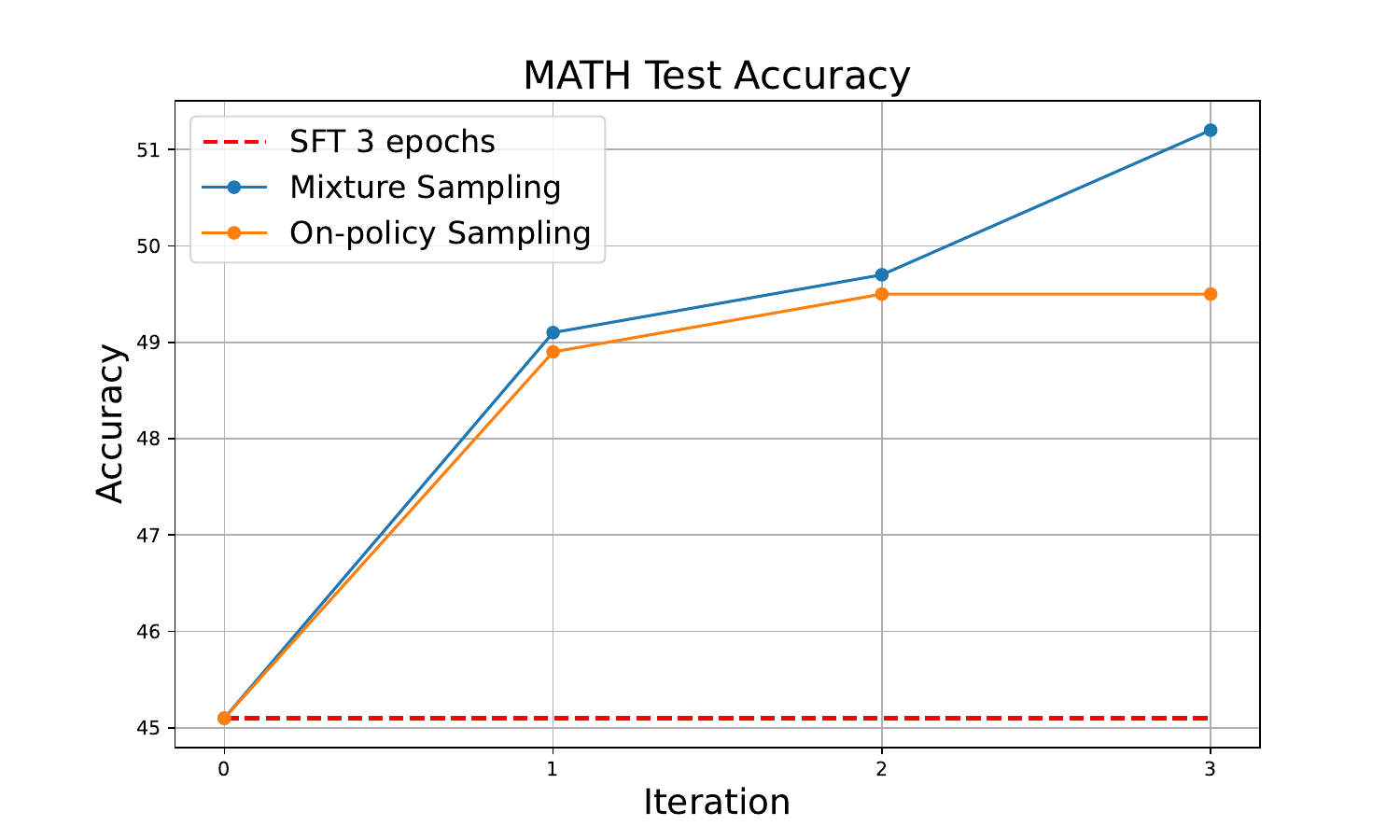}
    \caption{The plot of test accuracy on MATH dataset in terms of training iterations with different sampling strategies.}
    \label{fig:math_sampling}
\end{figure}


\paragraph{The best model is obtained with starting checkpoint fine-tuned with more than 1 epochs.} \citet{tunstall2023zephyr} finds that if the SFT model is trained for more than one epoch, the subsequent DPO training will lead to performance regression with longer training in terms of instruction-following ability and benchmark for a general chatbot. In other words, there exists a trade-off between the SFT training epochs and the DPO training steps. Moreover, the best model is obtained by SFT for one epoch in their practice. We also conduct an ablation study on the impact of the SFT epoch and summarize the results in Table~\ref{tab:ablation_SFT}. Consistently across all tested scenarios, the subsequent iterative M-DPO training leads to considerable model improvement compared to the SFT model. Meanwhile, we also observe a similar trade-off between SFT and RLHF training because with more SFT epochs, the gains from the RLHF stage decrease. However, in our case, the strongest model is obtained with three epochs of SFT, followed by fine-tuning through iterative M-DPO, which is different from the offline DPO training \citep{tunstall2023zephyr} or the iterative DPO for general chatbot \citep{dong2024rlhf} with only one epoch of SFT.

\begin{table}[H]
    \centering 
    \caption{Ablation study of the impact of SFT epoch. Mixture sampling is adopted for the iterative M-DPO training and we run for three iterations in total. The gains relative to their starting SFT checkpoints are marked by \up{}. } \vspace{4pt}\label{tab:ablation_SFT}
    \begin{tabular}{cc|ll}
    \toprule
    \textbf{Model} & \textbf{Method} & \textbf{GSM8K} & \textbf{MATH} \\ \midrule
    Gemma-1.1-it-7B & SFT 1 epoch & 75.1 & 41.1\\
    Gemma-1.1-it-7B & SFT 1 epoch + Iterative M-DPO   & 80.6 \up{5.5} & 46.7 \up{5.6} \\
    Gemma-1.1-it-7B & SFT 2 epoch & 75.3 & 44.0  \\
    Gemma-1.1-it-7B & SFT 2 epoch + Iterative M-DPO   & 82.4 \up{7.1} & 49.8 \up{5.8}	 \\
    Gemma-1.1-it-7B & SFT 3 epoch & 77.5 & 46.1  \\
    Gemma-1.1-it-7B & SFT 3 epoch + Iterative M-DPO   & 83.9 \up{6.4} & 51.2 \up{5.1}	\\
    \bottomrule
    \end{tabular}
    \end{table}

\paragraph{NLL loss helps when the SFT model is substantially underfitting.} The recent work \citet{pang2024iterative} has introduced iterative RPO, specifically aimed at enhancing Chain of Thought (CoT) capabilities for solving mathematical problems. A key feature of this approach is the inclusion of an additional negative log-likelihood (NLL) loss for the preferred response. The main intuition for adding the NLL loss is that the original DPO algorithm \citep{rafailov2023direct} tends to reduce the likelihood of the preferred responses, and this is believed to hurt the reasoning ability \citep{mint2024a}.  Motivated by their results, we explored the applicability of this idea to our setup. We conduct an ablation study by adding the NLL loss into the iterative M-DPO training and observe performance regression as reported in Table~\ref{tab:misc_ablation}. We observe that the best model is obtained in the second iteration if we add the additional NLL loss even though we use the mixture sampling to increase the diversity of the collected data. With time-weighted exponential moving average for smoothing training record, we observe that the log probability of the chosen responses and rejected responses are (-126, -222) at the 200th step of the third iteration training when we add the NLL loss, as compared to (-166, -350) in the case without the NLL loss. This is consistent with the result of \citet{pang2024iterative} where with the additional NLL loss, both the log probability of chosen responses and that of rejected responses increase. These evidences indicate that the NLL loss further contributes to the model distribution collapse and eventually hurt the overall performance of online iterative learning. Finally, we notice that the additional NLL loss can be viewed as an implementation of the pessimistic principle \citep{liu2024provably}. This also explains its inferior in-domain performance though it may be helpful to stable the training, which requires more in-depth studies.

However, one distinct feature between our setup and \citet{pang2024iterative} is whether we first fine-tune the initialized SFT model with in-domain data. To further understand the phenomena, we fine-tune the Gemma-1.1-it-7B with only 100 steps (so that the model knows to leverage Python code to solve the problem) as the starting checkpoint of preference learning and conduct an ablation study with the NLL loss using this model. We observe when the SFT model is substantially underfitting, the addition of NLL loss actually enhances performance. This scenario mirrors the findings of \citet{pang2024iterative}, who utilized a general LLaMA2-70B-chat model \citep{touvron2023llama} without firstly fine-tuning on the in-domain data. Our observations align with prior research in the context of developing general chatbots \citep{lin2023speciality}, which suggests that RLHF is less effective without preliminary SFT.

\begin{table}[H]
    \centering 
    \caption{Other ablation studies. Mixture sampling is adopted for the iterative M-DPO training and we run for three iterations in total. The gains relative to the iterative M-DPO are marked by \up{}. }\vspace{4pt}\label{tab:misc_ablation}
    \begin{tabular}{cc|ll}
    \toprule
    \textbf{Model} & \textbf{Method} & \textbf{GSM8K} & \textbf{MATH} \\ \midrule
    Gemma-1.1-it-7B & SFT 3 epoch & 77.5 & 46.1  \\
    Gemma-1.1-it-7B & SFT 3 epoch + Iterative M-DPO   & 83.9  & 51.2 	\\
    Gemma-1.1-it-7B & Iterative M-DPO with NLL loss & 81.7 \down{2.2} &  49.5 \down{1.7} \\
\midrule
    Gemma-1.1-it-7B & SFT 100 steps & 50.8 & 23.7\\
    Gemma-1.1-it-7B & + M-DPO Iteration 1  & 57.8 & 27.9   \\
    Gemma-1.1-it-7B & + M-DPO and NLL loss Iteration 1 & 61.0 \up{3.2} & 30.1 \up{2.2}  \\    
    \bottomrule
    \end{tabular}
\end{table}

\paragraph{On-policy sampling and small learning rate mitigate the probability drops in preferred responses.} In the literature, the Direct Preference Optimization (DPO) algorithm is often reported to diminish reasoning capabilities by reducing the likelihood of preferred responses \citep{yuan2024advancing, hong2024orpo, meng2024simpo}. In our preliminary experiments, we also observe similar phenomena with a large learning rate (1e-6), where the model's reasoning ability collapses after only a few training steps, preventing convergence to good reasoning performance. In contrast, we find that using on-policy sampling within our online iterative training framework, coupled with a smaller learning rate (2e-7 or 4e-7), the DPO algorithm enhances the model's reasoning abilities. To interpret our observation, we can first write down the gradient of the DPO as follows:
$$\begin{aligned}
\nabla_\theta \mathcal{L}_{DPO}(\pi_\theta, \pi_{\reff}) =-\eta\cdot \sigma\Big( r_\theta(x, y^l) - r_\theta(x, y^w) \Big) \Big[\frac{1}{\pi_\theta(y^w|x)} \nabla_\theta \pi_\theta(y^w|x) - \frac{1}{\pi_\theta(y^l|x)} \nabla_\theta \pi_\theta(y^l|x)  \Big],
\end{aligned}
$$
where $r_\theta(x,y) = \eta \log \frac{\pi_\theta(x,y)}{\pi_{\reff}(x,y)}$ is the implicit reward and we use the single-turn one for simplicity. In practice, the probability of the rejected responses typically decrease, and their gradient quickly dominates when $\pi_\theta(y^l|x) << \pi_\theta(y^w|x)$ and the optimization becomes unlearning of the rejected responses. In this case, the probability of the chosen responses cannot increase. This phenomenon was also discussed in the blog \citet{guo2024alignment}. When we adopt on-policy sampling, it leads to a relatively large probability for both rejected and chosen responses at the initial stage, ensuring that both gradients remain valid and effective. Moreover, a small learning rate prevents the model from deviating too significantly, maintaining the effectiveness of both gradients. We also notice that for the KTO algorithm, the preferred responses and the rejected responses do not appear in pairs. We suspect that the probability of the preferred response increases because the gradients of the rejected response do not dominate in every mini-batch of data. A more comprehensive understanding of the training dynamic of the direct preference learning algorithms remains largely open and we leave a more detailed study of this phenomena to future study.

\section{Conclusion, Limitation, and Future Research Direction}

We demonstrate that preference learning, as an alternative learning paradigm to supervised fine-tuning, can further boost the performance of the tool-integrated reasoning LLMs on top of iterative best-of-n fine-tuning. We introduce an online iterative multi-turn direct preference optimization algorithm and validate its effectiveness through extensive experimentation across various base models. Our results indicate substantial improvements in the pass@1 metric over the SFT policy, as evidenced by performance gains on standard benchmarks such as GSM8K \citep{cobbe2021gsm8k} and MATH \citep{hendrycks2021measuring}. Additionally, we also conduct various ablation studies to show that achieving optimal performance requires a careful balance between per-iteration improvement and exploration, facilitated by moderate levels of KL regularization and strategic exploration choices.

There are also several potential directions to further improve the model performance that we have not explored in this paper. Currently, our experiments only use final result check as the preference signal, so we cannot effectively compare trajectories that both end with correct or incorrect answers. Although our algorithm is designed for \textit{trajectory-level} preference learning, the reward signal in the Bradley-Terry model could be adapted to a step-wise level. In particular, we may leverage AI feedback to verify trajectories step by step or train a process-supervised reward model \citep{lightman2023let} to provide learning signals. Additionally, with more fine-grained reward signals, it is also possible to adopt more advanced heuristic exploration policy like west-of-n sampling, which prove to be effective in the practice of making general chatbot \citep{pace2024west, dong2024rlhf, xu2023some, snorkelai@pair} and Monte Carlo tree search (MCTS) \citep{xie2024monte}. Furthermore, it is also possible to leverage some well-established tricks like adaptive margin and length regularization for DPO training \citep{hong2024orpo, meng2024simpo}. These techniques have proven to be effective for achieving a better in-domain performance for the chat task. We expect that these more fine-grained preference signals and algorithmic designs can largely improve the models' performance.

Finally, while the direct preference learning algorithms show promising gains for the mathematical reasoning tasks with code interpreter, it is not directly applicable to the general agent learning with more complex and stochastic external environments or against dynamic opponents. In particular, it requires to construct a value network for involving an adaptive margin in the optimization target and take the randomness of the external environment into consideration. We leave the study of this more involved algorithm to the future work. Moving beyond the framework presented this paper, it is also possible to explore more general preference structures beyond the BT model \citep{munos2023nash, ye2024theoretical}. We hope that the insights from this paper will inspire further research in this direction, extending the utility of preference learning beyond the general structured chat tasks.

\section*{Acknowledgements}
Wei Xiong and Tong Zhang are partially supported by an NSF IIS grant No. 2416897


\bibliographystyle{abbrvnat}
\nobibliography*
\bibliography{myrefs}

\newpage
\appendix

\section{Notation Table}

\begin{table}[h]
\centering 
\small
\begin{tabular}{|c|c|}
\hline
\textbf{Notation} & \textbf{Description} \\
\hline
\hline
$x, \cX$ & The prompt and the prompt space. \\
$d_0$ & The distribution of initial state (prompt). \\
$s_h \in \cS, a_h \in \cA, o_h$ & The state, action, and observation. \\
$H$ & Episode length, e.g., the maximal number of tool calls. \\
$\PP^* = [\PP^*_h]_{h=1}^H$ & The true observation kernel.\\
$\tau = (x, y)$ & $\tau$ is a trajectory and $y$ is the completion part, i.e., we exclude $x$ from $\tau$.\\
$u^*$ & The true utility function associated with the BT model defined in Definition~\ref{def:bt_model}. \\
$\cM^* = (\cS, \cA, H, \PP^*, d_0, u^*)$ & The true model with observation kernel $\PP^*$ and utility function $u^*$\\
$\sigma(\cdot)$ & $\sigma(z) = 1/(1+\exp(-z))$ is the sigmoid function.\\
$z \in \{0, 1\}$ & Preference signal. \\
\hline
$\pi=[\pi_h]_{h=1}^H$ & The policy, which is parameterized by the LLM. \\
$\cM = (\cS, \cA, H, \PP, d_0, u)$ & One arbitrary environment with observation kernel $\PP$ and utility function $u$.\\
$\pi_{\reff} = [\pi_{\reff,h}]_{h=1}^H$ & One arbitrary reference policy.  \\
$J(\pi; \cM, \pi_{\reff})$ & The KL-regularized target (\eqref{eqn:optimal_policy}) with environment $\cM$ and reference $\pi_{\reff}$.\\
$\eta$ & The coefficient of KL penalty, defined in \eqref{eqn:optimal_policy}.\\
$Q_{\cM} = [Q_{\cM,h}]_{h=1}^H$ & The optimal $Q$-values associated with $J(\pi; \cM, \pi_{\reff})$, defined in \eqref{eqn:bellman_eqn0}.\\
$V_{\cM} = [V_{\cM,h}]_{h=1}^H$ & The optimal $V$-values associated with $J(\pi; \cM, \pi_{\reff})$, defined in \eqref{eqn:bellman_eqn}.\\
$\pi_\cM = [\pi_{\cM,h}]_{h=1}^H$ & The optimal policy associated with $J(\pi; \cM, \pi_{\reff})$, defined in \eqref{eqn:bellman_eqn}.\\
\hline
$\mathcal{L}_{\textup{M-DPO}}(\cdot)$ & M-DPO loss, defined in \eqref{eqn:final_m_dpo_loss}.\\
$\mathcal{L}_{\textup{M-KTO}}(\cdot)$ & M-KTO loss, defined in \eqref{eqn:final_m_kto_loss}.\\
\hline
$J(\pi)$ & The abbreviation of $J(\pi; \cM^*, \pi_0)$, defined in \eqref{eqn:val_target2}. \\
$\pi^* = [\pi^*_h]_{h=1}^H$ & The optimal policy associated with $J(\pi)$.\\
$\pi^1_t, \pi^2_t$ & The main and exploration policy at round $t$\\
$\text{Reg}(T)$ & Regret over horizon $T$, defined in \eqref{eqn:regret}.\\
$\cU, \cP$ & Known sets such that $u^*\in \cU$ and $\PP^*\in \cP$\\
$B$ & Assuming $u^*(x,y) \in [0,B], \forall (x,y)$.\\ 
$\hat{u}_t, \hat{\PP}_t$ & MLE of $u^*$ and $\PP^*$ at round $t$, defined in \eqref{eqn:mle_reward} and \eqref{eqn:mle_transition}.\\
$\tilde{\cU}_t, \tilde{\cP}_t$ & Confidences sets of $u^*$ and $\PP^*$ at round $t$, defined in \eqref{eqn:confidence}.\\
$c_1, c_2, c$ & Absolute constants.\\
$\kappa$ & $1/(2+ \exp(-B)+ \exp(B))$.\\
$d_\cU$ & Eluder coefficient from Definition~\ref{def:eluder_coefficient}.\\
$d_\cP, \xi(\cdot)$ & Generalized Eluder-type condition from Definition~\ref{def:eluder_condition}.\\
$\TV(\cdot, \cdot)$ & Total variation distance between two distributions.\\
\hline
\end{tabular}
\caption{The table of notations used in this paper.}
\label{tab:notation}
\end{table}

\section{Implementation Detail}
\label{appendix:implementation}

\paragraph{Tools in Math Problem Solving.}
Following \citet{gou2023tora, toshniwal2024openmathinstruct}, the LLM agent is allowed to call the python interpreter when it decodes a python code starting with ```python and ending with ```. For each step $h$, to generate the observation $o_h$, we leverage the python package \texttt{IPython}, and run all the codes in the history one by one and treat each code snippet as a Jupyter cell. We only return the standard output or the error message from the last snippet. When there exists some bug in the code, we only return the error message which is typically less than 20 tokens as in \citet{toshniwal2024openmathinstruct}. We notice that some works (e.g. \citet{shao2024deepseekmath}) also returns the first and the last 50 tokens of the trackback information.

\paragraph{Data Generation.} All the models are evaluated in the zero-shot setting. For all the data generation process, we adopt the following constraints: (1) for each turn, the model can generate up to 512 tokens; (2) the maximal number of steps is H=6; (3) the maximal number of generated token for each trajectory is 2048. When collecting new data for online iterative M-DPO, we set temperature to be 1.0 and decode without top-K or top-p sampling. For evaluation, greedy decoding is employed so that the results are generally comparable with previous works \citet{gou2023tora, toshniwal2024openmathinstruct}. For evaluating the models with pass@n rate, we follow \citet{toshniwal2024openmathinstruct} to adopt a temperature of $0.7$.

\paragraph{Data Annotation.} For each prompt, we first divide the responses into the winning set $G^w$ and the losing set $G^l$ by checking the final answer. In practice, we observe that the model can memorize the final answer and output it even though the reasoning path itself is incorrect. To mitigate this issue, we include some heuristic filtering process. First, we delete all the trajectories in the winning set where the returned messages in the second last round indicate the code is with some bugs, but the models just ignore it and predict the ground-truth answer. Then, we delete the responses in both the winning set $G^w$ and losing set $G^l$ if they are longer than 2048 tokens. Finally, we randomly sample a trajectory from the $G^w$ and a trajectory from $G^l$ to construct a pair or to add them into the training set of KTO algorithm. For each iteration, we typically get 15K-20K samples because some of the prompts may not have any correct answer. We notice that it is possible to leverage AI feedback like Gemini \citep{team2023gemini} or GPT4 \citep{OpenAI2023GPT4TR} to further verify the correctness of the trajectory step by step or construct a PRM \citep{lightman2023let, wang2023math} to rank the trajectories, which we leave for future work.

\paragraph{Python Experiment Environment.} We find that the evaluation can be influenced by the python environment, the precision (especially for the Gemma-1.1 models), and even the virtual machine we use. This does not affect the overall trend and conclusion because the magnitude of oscillation is relatively small compared to the overall improvement. For completeness, however, we specify some of the key package versions here. We use transformers 4.42.4, torch 2.3.0, sympy 1.2, antlr4-python3-runtime 4.11.0, IPython 8.26.0 for all models. We evaluate the models using torch.float and use vllm 0.5.0.post1 for most the experiments except for Gemma-2 where vllm 0.5.1 is required. The inconsistency of vllm version is because Gemma-2 model was not released when we performed the main experiments of this project. We fix the python environment and machine for our evaluation throughout the experiment. For SFT, we use the open-source axolotl project with version 0.4.1 and for online iterative preference learning and RAFT, we use the code base from RLHF Workflow \citep{dong2024rlhf}. 

\paragraph{RAFT implementation.} The data generation step is similar to the online iterative M-DPO training, except that we only keep the trajectories with correct final answer. For each prompt, we sample at most $k$ trajectories where we search $k \in \{1, 3, 8\}$ and use $k=1$ eventually because we do not see improvement by leveraging more data. We run the algorithm for three iterations in total. The training parameters are similar to the SFT stage, but we use a smaller batch size of 32 so that there are enough optimization steps. For Gemma models, we use a learning rate of 5e-6. For each training stage, we train the models for two epochs in total according to our parameter search. For Mistral model, we find that a smaller learning rate of 1e-6 and training for 1 epoch give us much better performance.

\paragraph{Prompt template.} We do not tune the prompt though we do observe that the prompt engineering can further improve the performance. For all the experiments, we simply adopt the chat template of the models as in Table~\ref{tab:multi_turn_math}.



\section{Omitted Theoretical Proofs}
\subsection{Proof of Proposition~\ref{prop:performance difference}}
\begin{proof}[Proof of Proposition~\ref{prop:performance difference}]
For one policy $\pi$, starting with $V^\pi_{\cM, H+1} = 0$, we recursively define its $V$-value and $Q$-value functions on one model $\cM = (\cS, \cA, H, \PP, d_0, u)$ and the reference policy $\pi_{\reff}$ as
\begin{align*}
Q^\pi_{\cM, h}(s_h, a_h) &:=\begin{cases}
    & u(s_H, a_H),  \qquad \text{ if } h = H,  \\
  & \E_{o_h \sim \PP_h(\cdot|s_h, a_h)} [V^\pi_{\cM, h+1}(s_{h+1})], \qquad \text{ if } h \leq H-1, 
\end{cases}\\
    V^\pi_{\cM, h}(s_h) &:= \E_{a_h \sim \pi_{h}(\cdot| s_h)} \big[Q^\pi_{\cM, h}(s_h, a_h) - \eta \cdot \KL\big(\pi_{h} (\cdot|s_h), \pi_{\reff, h}(\cdot|s_h)\big)\big].
\end{align*}
It is noted that with the optimal policy $\pi_\cM$, $Q_{\cM, h} = Q^{\pi_\cM}_{\cM, h}$ and $V_{\cM, h} = V^{\pi_\cM}_{\cM, h}$. In the following discussions, we exclusively focus on the model $\cM^* = (\cS, \cA, H, \PP^*, d_0, u^*)$ with abbreviations $Q^\pi_{h} = Q^\pi_{\cM^*, h}$ and  $V^\pi_{h} = V^\pi_{\cM^*, h}$.

For any comparator policy $\pi$, it holds that
\begin{align*}
    &J(\pi) - J(\hat{\pi})  = \E_{d_0}\left[V_1^\pi(s_1) - \hat{V}_1(s_1)\right] - \E_{d_0}\left[V_1^{\hat{\pi}}(s_1) - \hat{V}_1(s_1)\right],
\end{align*}
For any $h\in [H]$, we can obtain that
\begin{align*}
    &\E_{d_0, \pi_{1:h-1}, \PP^*_{1:h-1}}\left[V_h^\pi(s_h) - \hat{V}_h(s_h)\right] - \E_{d_0, \hat{\pi}_{1:h-1}, \PP^*_{1:h-1}}\left[V_h^{\hat{\pi}}(s_h) - \hat{V}_h(s_h)\right]\\
    & \overset{(a)}{=} \E_{d_0, \pi_{1:h-1}, \PP^*_{1:h-1}}\left[\E_{\pi_h}\left[Q_h^\pi(s_h, a_h)\right] - \eta \KL\left(\pi_h(\cdot|s_h),\pi_{\reff,h}(\cdot|s_h)\right)\right]  \\
    & - \E_{d_0, \pi_{1:h-1}, \PP^*_{1:h-1}}\left[ \E_{\hat{\pi}_h}\left[\hat{Q}_h(s_h, a_h)\right] - \eta \KL\left(\hat{\pi}_h(\cdot|s_h),\pi_{\reff,h}(\cdot|s_h)\right)\right]\\
    &- \E_{d_0, \hat{\pi}_{1:h-1}, \PP^*_{1:h-1}}\left[\E_{\hat{\pi}_h}\left[Q_h^{\hat{\pi}}(s_h, a_h)\right] - \eta \KL(\hat{\pi}_h(\cdot|s_h), \pi_{\reff,h}(\cdot|s_h))\right]\\
    &+ \E_{d_0, \hat{\pi}_{1:h-1}, \PP^*_{1:h-1}}\left[\E_{\hat{\pi}_h}\left[\hat{Q}_h(s_h, a_h)\right]- \eta \KL(\hat{\pi}_h(\cdot|s_h), \pi_{\reff,h}(\cdot|s_h))\right]\\
    & = \E_{d_0, \pi_{1:h}, \PP^*_{1:h-1}}\left[Q_h^\pi(s_h, a_h) - \hat{Q}_h(s_h, a_h)\right] - \E_{d_0, \hat{\pi}_{1:h}, \PP^*_{1:h-1}}\left[Q_h^{\hat{\pi}}(s_h, a_h) - \hat{Q}_h(s_h, a_h)\right]\\
    & + \E_{d_0, \pi_{1:h-1}, \PP^*_{1:h-1}}\underbrace{\left[\E_{\pi_h} \left[\hat{Q}_h(s_h, a_h)\right] - \E_{\hat{\pi}_h}\left[\hat{Q}_h(s_h, a_h)\right]\right]}_{\text{term (I)}} \\
    & -  \eta \cdot \E_{d_0, \pi_{1:h-1}, \PP^*_{1:h-1}} \left[\KL\left(\pi_h(\cdot|s_h), \pi_{\reff,h}(\cdot|s_h)\right)\right] +  \eta \cdot \E_{d_0, \pi_{1:h-1}, \PP^*_{1:h-1}} \left[\KL\left(\hat{\pi}_h(\cdot|s_h), \pi_{\reff,h}(\cdot|s_h)\right)\right]\\
    & \overset{(b)}{=} \E_{d_0, \pi_{1:h}, \PP^*_{1:h-1}}\left[Q_h^\pi(s_h, a_h) - \hat{Q}_h(s_h, a_h)\right] - \E_{d_0, \hat{\pi}_{1:h}, \PP^*_{1:h-1}}\left[Q_h^{\hat{\pi}}(s_h, a_h) - \hat{Q}_h(s_h, a_h)\right]\\
    & -  \eta \cdot \E_{d_0, \pi_{1:h-1}, \PP^*_{1:h-1}} \left[\KL\left(\pi_h(\cdot|s_h), \hat{\pi}_h(\cdot|s_h)\right)\right].
\end{align*}
In the above derivation, equation (a) is from the definitions of $Q^\pi$ and $V^\pi$, and the relationship between $\hat{Q}$ and $\hat{V}$. The equation (b) is because
\begin{align*}
    \text{(term I)} &:= \E_{\pi_h} \left[\hat{Q}_h(s_h, a_h)\right] - \E_{\hat{\pi}_h} \left[\hat{Q}_h(s_h, a_h)\right] \\
    & = \eta \cdot \E_{\pi_h} \left[\log\frac{\hat{\pi}_h(a_h|s_h)}{\pi_{\reff,h}(a_h|s_h)}\right] -  \eta \cdot \E_{\hat{\pi}_h} \left[\log\frac{\hat{\pi}_h(a_h|s_h)}{\pi_{\reff,h}(a_h|s_h)}\right]\\
    & =  \eta \cdot \KL\left(\pi_h(\cdot|s_h),\pi_{\reff,h}(\cdot|s_h)\right) -  \eta \cdot \KL\left(\pi_h(\cdot|s_h), \hat{\pi}_h(\cdot|s_h)\right) -  \eta \cdot \KL\left(\hat{\pi}_h(\cdot|s_h), \pi_{\reff,h}(\cdot|s_h)\right).
\end{align*}
where the second equation is from the relationship that
\begin{align*}
    \hat{Q}_h(s_h, a_h) = \eta \cdot \log \frac{\hat{\pi}_h(a_h|s_h)}{\pi_{\reff,h}(a_h|s_h)} - \eta \cdot \log \hat{Z}_h(s_h).
\end{align*}

Furthermore, if $h = H$, we can obtain that
\begin{align*}
    &\E_{d_0, \pi_{1:H-1}, \PP^*_{1:H-1}}\left[V_H^\pi(s_H) - \hat{V}_H(s_H)\right] - \E_{d_0, \hat{\pi}_{1:H-1}, \PP^*_{1:H-1}}\left[V_H^{\hat{\pi}}(s_H) - \hat{V}_H(s_H)\right]\\    
    & = \E_{d_0, \pi_{1:H}, \PP^*_{1:H-1}}\left[u^*(s_H, a_H) - \hat{Q}_H(s_H, a_H)\right]  - \E_{d_0, \hat{\pi}_{1:H}, \PP^*_{1:H-1}}\left[u^*(s_H, a_H) - \hat{Q}_H(s_H, a_H)\right]\\
    & -  \eta \cdot \E_{d_0, \pi_{1:H-1}, \PP^*_{1:H-1}} \left[\KL\left(\pi_H(\cdot|s_H), \hat{\pi}_H(\cdot|s_H)\right)\right] \\
    & = \E_{d_0, \pi_{1:H}, \PP^*_{1:H-1}}\left[u^*(s_H, a_H)\right] - \E_{d_0, \hat{\pi}_{1:H}, \PP^*_{1:H-1}}\left[u^*(s_H, a_H)\right]\\ 
    & + \E_{d_0, \pi_{1:H}, \PP^*_{1:H}}\left[\hat{V}_{H+1}(s_{H+1}) - \hat{Q}_H(s_H, a_H)\right]  - \E_{d_0, \hat{\pi}_{1:H}, \PP^*_{1:H}}\left[\hat{V}_{H+1}(s_{H+1}) - \hat{Q}_H(s_H, a_H)\right]\\
    & -  \eta \cdot \E_{d_0, \pi_{1:H-1}, \PP^*_{1:H-1}} \left[\KL\left(\pi_H(\cdot|s_H)||\hat{\pi}_H(\cdot|s_H)\right)\right],
\end{align*}
where the second equality leverages that $\hat{V}_{H+1}(s_{H+1}) = 0$;
otherwise, for all $h \leq H-1$, it holds that
\begin{align*}
    &\E_{d_0, \pi_{1:h-1}, \PP^*_{1:h-1}}\left[V_h^\pi(s_h) - \hat{V}_h(s_h)\right] - \E_{d_0, \hat{\pi}_{1:h-1}, \PP^*_{1:h-1}}\left[V_h^{\hat{\pi}}(s_h) - \hat{V}_h(s_h)\right]\\
    & = \E_{d_0, \pi_{1:h}, \PP^*_{1:h-1}}\left[Q^\pi_{h}(s_{h},a_h) - \hat{Q}_h(s_h, a_h)\right] - \E_{d_0, \hat{\pi}_{1:h}, \PP^*_{1:h-1}}\left[Q^{\hat{\pi}}_{h}(s_{h},a_h) - \hat{Q}_h(s_h, a_h)\right]\\
    & -  \eta \cdot \E_{d_0, \pi_{1:h-1}, \PP^*_{1:h-1}} \left[\KL\left(\pi_h(\cdot|s_h)||\hat{\pi}_h(\cdot|s_h)\right)\right] \\
    & = \E_{d_0, \pi_{1:h}, \PP^*_{1:h}}\left[\hat{V}_{h+1}(s_{h+1}) - \hat{Q}_h(s_h, a_h)\right] - \E_{d_0, \hat{\pi}_{1:h}, \PP^*_{1:h}}\left[\hat{V}_{h+1}(s_{h+1}) - \hat{Q}_h(s_h, a_h)\right]\\
    & -  \eta \cdot \E_{d_0, \pi_{1:h-1}, \PP^*_{1:h-1}} \left[\KL\left(\pi_h(\cdot|s_h)||\hat{\pi}_h(\cdot|s_h)\right)\right] \\
    & + \E_{d_0, \pi_{1:h}, \PP^*_{1:h}}\left[V^{\pi}_{h+1}(s_{h+1})  - \hat{V}_{h+1}(s_{h+1})\right] - \E_{d_0, \pi_{1:h}, \PP^*_{1:h}}\left[V^{\hat{\pi}}_{h+1}(s_{h+1})  - \hat{V}_{h+1}(s_{h+1})\right].
\end{align*}
The proposition can be obtained by iteratively using the above relationship for $h\in [H]$.
\end{proof}

\subsection{Proof of Theorem~\ref{thm:online_guarantee}}

First, with the assumption $u^*\in \cU$ and $\PP^*\in \cP$, the following lemma demonstrates that $\tilde{\cU}_t$ and $\tilde{\cP}_t$ are valid confidence sets.
\begin{lemma}[Proposition B.1 from \citet{liu2023optimistic}]\label{lem:confidence_set_valid}
    There exists an absolute constant $c_1$ such that for any $\delta\in (0, 1]$, with probability at least $1-\delta$, for all $t\in [T]$, $\hat{u}\in \cU$, and $\hat{\PP}\in \cP$, it holds that
    \begin{align*}
        L_t(\hat{u}) - L_t(u^*) \leq c_1 \log(|\cU|T/\delta), \qquad L_t(\hat{\PP}) - L_t(\PP^*) \leq c_1 \log(|\cP|T/\delta),
    \end{align*}
    which implies that $u^*\in \tilde{\cU}_t$ and $\PP^*\in \tilde{\cP}_t$.
\end{lemma}


Then, we provide an additional lemma demonstrating the in-sample error of the MLE and optimistic estimators.
\begin{lemma}\label{lem:in_sample_error}
    There exists an absolute constant $c_2$ such that for any $\delta\in (0, 1]$, with probability at least $1-\delta$, for all $t\in [T]$, we have
    \begin{align*}
        \sum_{i < t}\left|\sigma\left(\hat{u}_t(s^2_{i,H}, a^2_{i,H}) - \hat{u}_t(s^1_{i,H}, a^1_{i,H})\right) - \sigma\left(u^*(s^2_{i,H}, a^2_{i,H}) - u^*(s^1_{i,H}, a^1_{i,H})\right)\right|^2 \leq c_2\log(|\cU|T/\delta);\\
        \sum_{i < t}\left|\sigma\left(\tilde{u}_t(s^2_{i,H}, a^2_{i,H}) - \tilde{u}_t(s^1_{i,H}, a^1_{i,H})\right) - \sigma\left(u^*(s^2_{i,H}, a^2_{i,H}) - u^*(s^1_{i,H}, a^1_{i,H})\right)\right|^2 \leq c_2\log(|\cU|T/\delta),
    \end{align*}
    and for all $t\in [T]$, $h\in [H]$, we have
    \begin{align*}
        \sum_{j\in \{1,2\}}\sum_{h\in [H]}\sum_{i < t}\TV\left(\{d_0, \pi^j_i, [\PP^*_{1:h-1}, \hat{\PP}_{t,h},  \PP^*_{h+1: H}]\}, \{d_0, \pi^j_i, \PP^*_{1:H}\}\right)^2 \leq c_2 \log(|\cP|HT/\delta);\\
        \sum_{j\in \{1,2\}}\sum_{h\in [H]}\sum_{i < t}\TV\left(\{d_0, \pi^j_i, [\PP^*_{1:h-1}, \tilde{\PP}_{t,h},  \PP^*_{h+1: H}]\}, \{d_0, \pi^j_i, \PP^*_{1:H}\}\right)^2 \leq c_2 \log(|\cP|HT/\delta),
    \end{align*}
    where $\TV(\{d_0, \pi, \PP\}, \{d_0, \pi', \PP'\})$ denotes the TV distance between the probability distributions over the trajectories induced by $d_0, \pi, \PP$ and $d_0, \pi', \PP'$.
\end{lemma}
\begin{proof}[Proof of Lemma~\ref{lem:in_sample_error}]
    First, for $\tilde{u}_t$, we can obtain that with probability at least $1-\delta$, there exists an absolute constant $c$ such that for all $t\in [T]$,
    \begin{align*}
        &\sum_{i < t}\left|\sigma\left(\tilde{u}_t(s^2_{i,H}, a^2_{i,H}) - \tilde{u}_t(s^1_{i,H}, a^1_{i,H})\right) - \sigma\left(u^*(s^2_{i,H}, a^2_{i,H}) - u^*(s^1_{i,H}, a^1_{i,H})\right)\right|^2 \\
        &\leq c\left(\sum_{i< t}\log \frac{z_i \cdot \sigma\left(u^*(s^1_{i,H}, a^1_{i,H}) - u^*(s^2_{i,H}, a^2_{i,H})\right)+ (1-z_i) \cdot \sigma\left(u^*(s^2_{i,H}, a^2_{i,H}) - u^*(s^1_{i,H}, a^1_{i,H})\right)}{z_i \cdot \sigma\left(\tilde{u}_t(s^1_{i,H}, a^1_{i,H}) - \tilde{u}_t(s^2_{i,H}, a^2_{i,H})\right) + (1- z_i) \cdot \sigma\left(\tilde{u}_t(s^2_{i,H}, a^2_{i,H}) - \tilde{u}_t(s^1_{i,H}, a^1_{i,H})\right)} + \log(|\cU|T/\delta)\right)\\
        & = c\left(L_{t}(u^*) - L_{t}(\tilde{u}_t) + \log(|\cU|T/\delta)\right)\\
        & \leq c\left(L_{t}(u^*) - L_{t}(\hat{u}_t) +  c_1 \log(|\cU|T/\delta) + \log(|\cU|T/\delta)\right)\\
        & \leq c_2  \log(|\cU|T/\delta).
    \end{align*}
    where the first inequality is from Proposition B.2 from \citet{liu2023optimistic} and the second inequality uses Lemma~\ref{lem:confidence_set_valid}. The result for $\hat{u}_t$ can be similarly established.

    Then, following similar steps, for $\tilde{\PP}_t$,  we can obtain that with probability at least $1-\delta$, there exists an absolute constant $c$ such that for all $t\in  [T]$,
    \begin{align*}
        &\sum_{j\in \{1,2\}}\sum_{h\in [H]}\sum_{i < t}\TV\left(\{d_0, \pi^j_i, [\PP^*_{1:h-1}, \tilde{\PP}_{t,h},  \PP^*_{h+1: H}]\}, \{d_0, \pi^j_i, \PP^*_{1:H}\}\right)^2\\
        &\leq \sum_{j\in \{1,2\}}\sum_{h\in [H]} c\cdot \left(\sum_{i<t} \log \frac{\PP^*_h(s^j_{i,h+1}|s^j_{i,h}, a^j_{i,h})}{\tilde{\PP}_{t,h}(s^j_{i,h+1}|s^j_{i,h}, a^j_{i,h})} + \log(|\cP_h|HT/\delta)\right)\\
        & = c\cdot \left(\sum_{j\in \{1,2\}}\sum_{i<t} \log \frac{\PP^{*,\pi^j_i}(\tau^j_i)}{\tilde{\PP}^{\pi^j_i}_{t}(\tau^j_i)} + 2\log(|\cP|HT/\delta)\right)\\
        & = c\cdot \left(L_t(\PP^*) - L_t(\tilde{\PP}_t) + 2\log(|\cP|HT/\delta)\right)\\
        & \leq c\cdot \left(L_t(\PP^*) - L_t(\hat{\PP}_t) + c_1 \log(|\cP|T/\delta) + 2\log(|\cP|HT/\delta)\right)\\
        & \leq c_2 \log(|\cP|HT/\delta).
    \end{align*}
    The result for $\hat{\PP}_t$ can also be similarly established.
\end{proof}

\begin{proof}[Proof of Theorem~\ref{thm:online_guarantee}] 
In the following proofs, we omit the KL term in the decomposition to ease the presentation. Then, with probability at least $1-\delta$, for all $t\in [T]$, we can obtain that
\begin{align*}
    &J(\pi^*) - J(\pi^1_t) \\
    & = \E_{d_0, \pi^*, \PP^*}\left[u^*(s_H, a_H)\right] - \E_{d_0, \pi^1_{t}, \PP^*}\left[u^*(s_H, a_H)\right] - \left(\E_{d_0, \pi^*, \PP^*}\left[\hat{u}_t(s_H, a_H)\right] - \E_{d_0, \pi^1_{t}, \PP^*}\left[\hat{u}_t(s_H, a_H)\right]\right) \\
    &+ \sum_{h\in [H]} \E_{d_0, \pi^*, \PP^*}\left[\hat{V}_{t, h+1}(s_{h+1}) - \left[\hat{\PP}_{t, h}\hat{V}_{t, h+1}\right](s_h,a_h)\right]- \sum_{h\in [H]}\E_{d_0, \pi^1_{t}, \PP^*}\left[ \hat{V}_{t, h+1}(s_{h+1}) - \left[\hat{\PP}_{t, h}\hat{V}_{t, h+1}\right](s_h,a_h)\right]\\
    & \leq  \underbrace{\E_{d_0, \pi_t^2, \tilde{\PP}_t}\left[\tilde{u}_t(s_H, a_H)\right] - \E_{d_0, \pi^1_{t}, \tilde{\PP}_t}\left[\tilde{u}_t(s_H, a_H)\right] - \left(\E_{d_0, \pi_t^2, \tilde{\PP}_t}\left[\hat{u}_t(s_H, a_H)\right] - \E_{d_0, \pi^1_{t}, \tilde{\PP}_t}\left[\hat{u}_t(s_H, a_H)\right]\right)}_{\text{term (I)}_t} \\
    &+ \underbrace{\sum_{h\in [H]} \E_{d_0, \pi_t^2, \tilde{\PP}_t}\left[\hat{V}_{t, h+1}(s_{h+1}) - \left[\hat{\PP}_{t, h}\hat{V}_{t, h+1}\right](s_h,a_h)\right]+ \sum_{h\in [H]}\E_{d_0, \pi^1_{t}, \PP^*}\left[ \left[\hat{\PP}_{t, h}\hat{V}_{t, h+1}\right](s_h,a_h) - \hat{V}_{t, h+1}(s_{h+1})\right]}_{\text{term (II)}_t},
\end{align*}
where the inequality is from the definition of $\pi^2_t$ and the fact that $(u^*, \PP^*)\in \tilde{\cU}_t\times \tilde{\cP}_t$ from Lemma~\ref{lem:confidence_set_valid}.

We define the following terms:
\begin{align*}
    \text{term (A)}_t &:=  \E_{d_0, \pi_t^2, \PP^*}\left[\tilde{u}_t(s_H, a_H)\right] - \E_{d_0, \pi^1_{t}, \PP^*}\left[\tilde{u}_t(s_H, a_H)\right] - \left(\E_{d_0, \pi_t^2, \PP^*}\left[u^*(s_H, a_H)\right] - \E_{d_0, \pi^1_{t}, \PP^*}\left[u^*(s_H, a_H)\right]\right),\\
    \text{term (B)}_t &:= \E_{d_0, \pi_t^2, \PP^*}\left[u^*(s_H, a_H)\right] - \E_{d_0, \pi^1_{t}, \PP^*}\left[u^*(s_H, a_H)\right] - \left(\E_{d_0, \pi_t^2, \PP^*}\left[\hat{u}_t(s_H, a_H)\right] - \E_{d_0, \pi^1_{t}, \PP^*}\left[\hat{u}_t(s_H, a_H)\right]\right),\\
    \text{term (C)}_t &:= \sum_{j\in \{1,2\}}\sum_{h\in [H]}\E_{d_0, \pi^j_t, \PP^*} \left[\TV\left(\tilde{\PP}_{t,h}(\cdot|s_h, a_h), \PP^*_h(\cdot|s_h, a_h)\right)\right],\\
    \text{term (D)}_t &:= \sum_{j\in\{1,2\}}\sum_{h\in [H]}\E_{d_0, \pi^j_{t}, \PP^*}\left[ \TV\left(\hat{\PP}_{t,h}(\cdot|s_h, a_h), \PP^*_h(\cdot|s_h, a_h)\right)\right].
\end{align*}

For $\text{term (I)}_t$, we have that
\begin{align*}
    \text{term (I)}_t &:= \E_{d_0, \pi_t^2, \tilde{\PP}_t}\left[\tilde{u}_t(s_H, a_H)\right] - \E_{d_0, \pi^1_{t}, \tilde{\PP}_t}\left[\tilde{u}_t(s_H, a_H)\right] - \left(\E_{d_0, \pi_t^2, \tilde{\PP}_t}\left[\hat{u}_t(s_H, a_H)\right] - \E_{d_0, \pi^1_{t}, \tilde{\PP}_t}\left[\hat{u}_t(s_H, a_H)\right]\right)\\
    & = \E_{d_0, \pi_t^2, \PP^*}\left[\tilde{u}_t(s_H, a_H)\right] - \E_{d_0, \pi_t^1, \PP^*}\left[\tilde{u}_t(s_H, a_H)\right] - \left(\E_{d_0, \pi_t^2, \PP^*}\left[u^*_t(s_H, a_H)\right] - \E_{d_0, \pi_t^1, \PP^*}\left[u^*_t(s_H, a_H)\right]\right)\\
    & + \E_{d_0, \pi_t^2, \PP^*}\left[u^*_t(s_H, a_H)\right] - \E_{d_0, \pi_t^1, \PP^*}\left[u^*_t(s_H, a_H)\right] - \left(\E_{d_0, \pi_t^2, \PP^*}\left[\hat{u}_t(s_H, a_H)\right] - \E_{d_0, \pi^1_{t}, \PP^*}\left[\hat{u}_t(s_H, a_H)\right]\right) \\
    & + \E_{d_0, \pi_t^2, \tilde{\PP}_t}\left[\tilde{u}_t(s_H, a_H)\right] - \E_{d_0, \pi_t^1, \tilde{\PP}_t}\left[\tilde{u}_t(s_H, a_H)\right] - \left(\E_{d_0, \pi_t^2, \PP^*}\left[\tilde{u}_t(s_H, a_H)\right] - \E_{d_0, \pi_t^1, \PP^*}\left[\tilde{u}_t(s_H, a_H)\right]\right)\\
    & + \E_{d_0, \pi_t^2, \PP^*}\left[\hat{u}_t(s_H, a_H)\right] - \E_{d_0, \pi^1_{t}, \PP^*}\left[\hat{u}_t(s_H, a_H)\right] - \left(\E_{d_0, \pi_t^2, \tilde{\PP}_t}\left[\hat{u}_t(s_H, a_H)\right] - \E_{d_0, \pi^1_{t}, \tilde{\PP}_t}\left[\hat{u}_t(s_H, a_H)\right]\right)\\
    & \leq \E_{d_0, \pi_t^2, \PP^*}\left[\tilde{u}_t(s_H, a_H)\right] - \E_{d_0, \pi_t^1, \PP^*}\left[\tilde{u}_t(s_H, a_H)\right] - \left(\E_{d_0, \pi_t^2, \PP^*}\left[u^*_t(s_H, a_H)\right] - \E_{d_0, \pi_t^1, \PP^*}\left[u^*_t(s_H, a_H)\right]\right)\\
    & + \E_{d_0, \pi_t^2, \PP^*}\left[u^*_t(s_H, a_H)\right] - \E_{d_0, \pi_t^1, \PP^*}\left[u^*_t(s_H, a_H)\right] - \left(\E_{d_0, \pi_t^2, \PP^*}\left[\hat{u}_t(s_H, a_H)\right] - \E_{d_0, \pi^1_{t}, \PP^*}\left[\hat{u}_t(s_H, a_H)\right]\right)\\
    & + 4B\cdot  \TV\left(\{d_0, \pi^1_t, \tilde{\PP}_t\}, \{d_0, \pi^1_t, \PP^*\}\right)+ 4B\cdot  \TV\left(\{d_0, \pi^2_t, \tilde{\PP}_t\}, \{d_0, \pi^2_t, \PP\}\right)\\
    & \leq \underbrace{\E_{d_0, \pi_t^2, \PP^*}\left[\tilde{u}_t(s_H, a_H)\right] - \E_{d_0, \pi_t^1, \PP^*}\left[\tilde{u}_t(s_H, a_H)\right] - \left(\E_{d_0, \pi_t^2, \PP^*}\left[u^*_t(s_H, a_H)\right] - \E_{d_0, \pi_t^1, \PP^*}\left[u^*_t(s_H, a_H)\right]\right)}_{\text{term (A)}_t}\\
    & + \underbrace{\E_{d_0, \pi_t^2, \PP^*}\left[u^*_t(s_H, a_H)\right] - \E_{d_0, \pi_t^1, \PP^*}\left[u^*_t(s_H, a_H)\right] - \left(\E_{d_0, \pi_t^2, \PP^*}\left[\hat{u}_t(s_H, a_H)\right] - \E_{d_0, \pi^1_{t}, \PP^*}\left[\hat{u}_t(s_H, a_H)\right]\right)}_{\text{term (B)}_t}\\
    & + 4B\cdot \underbrace{\sum_{j\in \{1,2\}}\sum_{h\in [H]}\E_{d_0}\E_{\pi^j_t, \PP^*} \left[\TV\left(\tilde{\PP}_{t,h}(\cdot|s_h, a_h), \PP^*_h(\cdot|s_h, a_h)\right)\right]}_{\text{term (C)}_t}.
\end{align*}

For $\text{term (II)}_t$, we have that
\begin{align*}
    \text{term (II)}_t &= \sum_{h\in [H]} \E_{d_0, \pi_t^2, \tilde{\PP}_t}\left[\hat{V}_{t, h+1}(s_{h+1}) - \left[\hat{\PP}_{t, h}\hat{V}_{t, h+1}\right](s_h,a_h)\right]\\
    & + \sum_{h\in [H]}\E_{d_0, \pi^1_{t}, \PP^*}\left[ \left[\hat{\PP}_{t, h}\hat{V}_{t, h+1}\right](s_h,a_h) - \hat{V}_{t, h+1}(s_{h+1})\right]\\
    & = \sum_{h\in [H]} \E_{d_0, \pi_t^2, \PP^*}\left[\hat{V}_{t, h+1}(s_{h+1}) - \left[\hat{\PP}_{t, h}\hat{V}_{t, h+1}\right](s_h,a_h)\right]\\
    & + \sum_{h\in [H]} \E_{d_0, \pi_t^2, \tilde{\PP}_t}\left[\hat{V}_{t, h+1}(s_{h+1}) - \left[\hat{\PP}_{t, h}\hat{V}_{t, h+1}\right](s_h,a_h)\right]\\
    & - \sum_{h\in [H]} \E_{d_0, \pi_t^2, \PP^*}\left[\hat{V}_{t, h+1}(s_{h+1}) - \left[\hat{\PP}_{t, h}\hat{V}_{t, h+1}\right](s_h,a_h)\right]\\
    & + \sum_{h\in [H]}\E_{d_0, \pi^1_{t}, \PP^*}\left[ \left[\hat{\PP}_{t, h}\hat{V}_{t, h+1}\right](s_h,a_h) - \hat{V}_{t, h+1}(s_{h+1})\right]\\
    & \leq 2B\cdot \sum_{j\in \{1,2\}}\sum_{h\in [H]} \E_{d_0, \pi_t^j, \PP^*}\left[\TV(\hat{\PP}_{t,h}(\cdot|s_h, a_h)), \PP^*_h(\cdot|s_h, a_h)\right]\\
    & + 2BH \cdot \TV(\{d_0, \pi_t^2, \tilde{\PP}_t\}, \{d_0, \pi_t^2, \PP^*\})\\
    & \leq 2B\cdot \underbrace{\sum_{j\in \{1,2\}}\sum_{h\in [H]} \E_{d_0, \pi_t^j, \PP^*}\left[\TV(\hat{\PP}_{t,h}(\cdot|s_h, a_h)), \PP^*_h(\cdot|s_h, a_h)\right]}_{\text{term (D)}_t}\\
    & + 2BH \cdot \underbrace{\sum_{j\in \{1,2\}}\sum_{h\in [H]} \E_{d_0, \pi_t^j, \PP^*}\left[\TV(\tilde{\PP}_{t,h}(\cdot|s_h, a_h)), \PP^*_h(\cdot|s_h, a_h)\right]}_{\text{term (C)}_t}.
\end{align*}

In the above derivations, we have repeatedly used similar relationships as follows:
\begin{align*}
    \TV(\{d_0, \pi^2_t, \tilde{\PP}_t\}, \{d_0, \pi^2_t, \PP^*\}) \leq \sum_{h\in [H]} \E_{d_0, \pi^2_t, \PP^*}\left[\TV\left(\tilde{\PP}_{t,h}(\cdot|s_h, a_h), \PP^*_{h}(\cdot|s_h, a_h)\right)\right],
\end{align*}
which can be derived as
\begin{align*}
    \TV(\{d_0, \pi^2_t, \tilde{\PP}_t\}, \{d_0, \pi^2_t, \PP^*\}) &\leq \sum_{h\in [H]} \TV\left(\{d_0, \pi^2_t, \PP^*_{1:h-1}, \tilde{\PP}_{t, h:H}\}, \{d_0, \pi^2_t, \PP^*_{1:h}, \tilde{\PP}_{t, h+1:H}\}\right)\\
    &= \sum_{h\in [H]} \E_{d_0, \pi^2_t, \PP^*}\left[\TV\left(\tilde{\PP}_{t, h}(\cdot|s_h,a_h), \PP^*_{h}(\cdot|s_h,a_h)\}\right)\right].
\end{align*}

Then, we can obtain that
\begin{align*}
    \sum_{t\in [T]}J(\pi^*) - J(\hat{\pi}^1_t) &\leq \sum_{t\in [T]} \text{term (A)}_t + \sum_{t\in [T]}\text{term (B)}_t + (4B + 2BH) \sum_{t\in [T]}\text{term (C)}_t + 2B \sum_{t\in [T]}\text{term (D)}_t.
\end{align*}

Then, we control the sum of each individual term in the following. First, for $\text{term (A)}_t$, with probability at least $1-\delta$, we have that
\begin{align*}
    &\sum_{t\in [T]}\text{term (A)}_t \\
    &= \sum_{t\in [T]}\E_{d_0, \pi_t^2, \PP^*}\left[\tilde{u}_t(s_H, a_H)\right] - \E_{d_0, \pi^1_{t}, \PP^*}\left[\tilde{u}_t(s_H, a_H)\right] - \left(\E_{d_0, \pi_t^2, \PP^*}\left[u^*(s_H, a_H)\right] - \E_{d_0, \pi^1_{t}, \PP^*}\left[u^*(s_H, a_H)\right]\right)\\
    & \leq \sum_{t\in [T]} \tilde{u}_t(s^2_{t,H}, a^2_{t,H}) - \tilde{u}_t(s^1_{t,H}, a^1_{t,H}) - \left(u^*(s^2_{t,H}, a^2_{t,H}) - u^*(s^1_{t,H}, a^1_{t,H})\right) + O(B\sqrt{T\log(1/\delta)})\\
    & \leq \sqrt{d_{\cU}\sum_{t=2}^T\left(1+ \sum_{i=1}^{t-1} \left(\tilde{u}_t(s^2_{i,H}, a^2_{i,H}) - \tilde{u}_t(s^1_{i,H}, a^1_{i,H}) - \left(u^*(s^2_{i,H}, a^2_{i,H}) - u^*(s^1_{i,H}, a^1_{i,H})\right)\right)^2\right)} + O(B\sqrt{T\log(1/\delta)})\\
    & \leq \sqrt{d_\cU\sum_{t=2}^T\left(1+ \kappa^{-2}\sum_{i=1}^{t-1} \left(\sigma\left(\tilde{u}_t(s^2_{i,H}, a^2_{i,H}) - \tilde{u}_t(s^1_{i,H}, a^1_{i,H})\right) - \sigma\left(u^*(s^2_{i,H}, a^2_{i,H}) - u^*(s^1_{i,H}, a^1_{i,H})\right)\right)^2\right)} + O(B\sqrt{T\log(1/\delta)})\\
    & \lesssim \kappa^{-1} B \sqrt{d_\cU T \log(|\cU|T/\delta)},
\end{align*}
where the first inequality is from the Hoeffding inequality, the second inequality uses the Eluder coefficient $d_\cU := \texttt{EC}(1, \cU - \cU, T)$ from Definition~\ref{def:eluder_coefficient}, the third inequality leverages the mean value theorem with $\kappa:= 1/(2+ \exp(-B)+ \exp(B))$ representing the minimum derivative of $\sigma(\cdot)$ in the regime of $[0, B]$, and the last inequality incorporates Lemma~\ref{lem:in_sample_error}. A similar result can be obtained for $\text{term (B)}_t$.

For $\text{term (C)}_t$, we have that
\begin{align*}
    \sum_{t\in [T]}\text{term (C)}_t &= \sum_{j\in \{1,2\}}\sum_{t\in [T]}\sum_{h\in [H]}\E_{d_0, \pi^j_t, \PP^*} \left[\TV\left(\tilde{\PP}_{t,h}(\cdot|s_h, a_h), \PP^*_h(\cdot|s_h, a_h)\right)\right]\\
    & = \sum_{j\in \{1,2\}} \sum_{t\in [T]}\sum_{h\in [H]}\TV\left(\{d_0, \pi^j_t, [\PP^*_{1:h-1}, \tilde{\PP}_{t,h},  \PP^*_{h+1: H}]\}, \{d_0, \pi^j_t, \PP^*_{1:H}\}\right)\\
    & \leq 2H\cdot \xi(d_\cP, T, c_2\log(|\cP|HT/\delta)),
\end{align*}
where the last step is from the generalized Eluder-type condition in Definition~\ref{def:eluder_condition} and Lemma~\ref{lem:in_sample_error}. A similar result can be obtained for $\text{term (D)}_t$. 

Finally, we obtain that
\begin{align*}
    \text{Reg}(T) \lesssim  &\kappa^{-1}B\sqrt{d_{\cU}T\log(|\cU|T/\delta)} + B^2H\xi(d_\cP, T, c_2\log(|\cP|HT/\delta) \\
    &- \eta\cdot \sum_{h\in [H]}\E_{d_0, \pi^*, \PP^*}\left[\KL(\pi^*_{h}(\cdot|s_{h}), \pi^1_{t, h}(\cdot|s_{h}))\right],
\end{align*}
which concludes the proof.
\end{proof}

\section{Technical Lemmas}

\begin{lemma}[Solution of KL-regularized Optimization (Proposition 7.16 and Theorem 15.3 of \citet{zhang2023mathematical})] \label{lem:kl_solu} Given a loss functional with respect to $p(\cdot|x)$, written as $$ 
\E_{w \sim p(\cdot)} \Big[-U(w) + \eta \KL\big(p(\cdot),p_0(\cdot) \big)\Big] = \eta \KL\Big( p(\cdot), p_0(\cdot)\exp\Big(\frac{1}{\eta}U(\cdot)\Big) \Big)- \eta \cdot \log \underbrace{\E_{w \sim p_0(\cdot)} \exp \big(\frac{1}{\eta}U(w)\big)}_{C_r}, 
$$
where the minimizer of the loss functional is 
$
p^*(w) = \frac{1}{C_r} p_0(w)\exp\Big(\frac{1}{\eta} U(w)\Big) 
$, also known as Gibbs distribution.
\end{lemma}

\begin{definition}[Eluder Coefficient, Definition 17.17 in \citet{zhang2023mathematical}]\label{def:eluder_coefficient}
    Given a function class $\cF$, its Eluder coefficient $\texttt{EC}(\lambda, \cF, T)$ is defined as the smallest number $d$ so that for any sequence $\{x_t: t\in [T]\}$ and $\{f_t: t\in [T]\} \in \cF$, 
    \begin{align*}
        \sum_{t = 2}^T |f_t(x_t) - f^*(x_t)| \leq \sqrt{d\sum_{t = 2}^T\left(\lambda + \sum_{i = 1}^{t-1} (f_t(x_i) - f^*(x_i))^2\right)}.
    \end{align*}
\end{definition}

\begin{definition}[Generalized Eluder-type Condition, Condition 3.1 in \citet{liu2023optimistic}]\label{def:eluder_condition}
    There exists a real number $d_\cP\in \R^+$ and a function $\xi$ such that for any $(T, \Delta) \in \NN \times \R^+$, transitions $\{\PP'_t: t\in [T]\}$ and policies $\{\pi_t: t\in [T]\}$, we have
    \begin{align*}
        \forall t\in [T], \quad  \sum_{i < t} \TV(\{d_0, \PP'_i, \pi_i\}, \{d_0, \PP, \pi_i\})^2  \leq \Delta \quad \Rightarrow \sum_{t \in [T]} \TV(\{d_0, \PP'_t, \pi_t\}, \{d_0, \PP, \pi_t\}) \leq \xi(d_\cP, T, \Delta).
    \end{align*}
\end{definition}

\end{document}

%% file: myheaders.tex
\usepackage{cancel}
\newcommand{\hmm}{\textbf{User: }}     
\newcommand{\assi}{\textbf{Assistant: }}
\usepackage[skins]{tcolorbox} 
\tcbuselibrary{breakable} 
\newtcolorbox[auto counter, number within=section, list type=subsubsection, list inside=toc]{sectionbox}[2][]{
colback=white!98!gray, colframe=black, 
colbacktitle=white!90!gray, coltitle=black, 
fonttitle=\bfseries,
title={#2}, 
list entry={Table \thetcbcounter\quad}
}
\usepackage{cancel}

\definecolor{darkorange}{RGB}{255, 140, 0}
\definecolor{darkblue}{RGB}{84, 112, 198}
\definecolor{lightgreen}{RGB}{145, 204, 117}
\definecolor{lightyellow}{RGB}{250, 200, 88}
\definecolor{lightred}{RGB}{238, 102, 102}
\definecolor{lightblue}{RGB}{115, 192, 222}
\newtcolorbox{promptbox}[2][Prompt]{
colback=black!5!white,
arc=5pt,
boxrule=0.5pt,
fonttitle=\bfseries,
title=#1,
before upper={\small}, fontupper=\fontfamily{ptm}\selectfont,
colframe=#2,
}

\usepackage{listings}

\usepackage{color} 

\definecolor{codegreen}{rgb}{0,0.6,0}
\definecolor{codegray}{rgb}{0.5,0.5,0.5}
\definecolor{codepurple}{rgb}{0.58,0,0.82}
\definecolor{backcolour}{rgb}{0.95,0.95,0.92}

\usepackage{tcolorbox}
\tcbuselibrary{breakable,xparse,skins}

\definecolor{bg}{gray}{0.95}
\newcommand{\up}[1]{\textcolor{OliveGreen}{\small \ $\uparrow${#1}}}
\newcommand{\down}[1]{\textcolor{Maroon}{\small \ $\downarrow${#1}}}

\usepackage{pifont}

\usepackage{algorithm}
\usepackage{algorithmic}

\def\1{\bm{1}}

\usepackage{enumitem}

\def\##1\#{\begin{align}#1\end{align}}
\def\$#1\${\begin{align*}#1\end{align*}}

\let\tilde\widetilde

\def\TV{\mathrm{TV}}

\newcommand{\cA}{\mathcal{A}}

\newcommand{\cD}{\mathcal{D}}

\newcommand{\cF}{\mathcal{F}}

\newcommand{\cM}{\mathcal{M}}

\newcommand{\cP}{\mathcal{P}}

\newcommand{\cS}{{\mathcal{S}}}

\newcommand{\cU}{\mathcal{U}}

\newcommand{\cX}{\mathcal{X}}

\newcommand{\E}{\mathbb{E}}

\newcommand{\NN}{\mathbb{N}}

\newcommand{\PP}{\mathbb{P}}

\newcommand{\reff}{\mathrm{ref}}

\definecolor{red1}{HTML}{f47983}
\definecolor{blue1}{HTML}{3eede7}
\definecolor{yellow1}{HTML}{f5dd6f}

\newcommand{\argmax}{\mathop{\mathrm{argmax}}}

\newtheorem{theorem}{Theorem}

\newtheorem{example}{Example}

\newtheorem{proposition}{Proposition}
\newtheorem{lemma}{Lemma}
\newtheorem{definition}{Definition}

\newcommand{\R}{\mathbb{R}}

\newcommand{\KL}{D_{\mathrm{KL}}}



%% file: main.bbl
\begin{thebibliography}{110}
\providecommand{\natexlab}[1]{#1}
\providecommand{\url}[1]{\texttt{#1}}
\expandafter\ifx\csname urlstyle\endcsname\relax
  \providecommand{\doi}[1]{doi: #1}\else
  \providecommand{\doi}{doi: \begingroup \urlstyle{rm}\Url}\fi

\bibitem[qwe(2024)]{qwen2}
Qwen2 technical report.
\newblock 2024.

\bibitem[Agarwal et~al.(2023)Agarwal, Jin, and Zhang]{agarwal2023vo}
A.~Agarwal, Y.~Jin, and T.~Zhang.
\newblock V{O$Q$L}: Towards optimal regret in model-free rl with nonlinear function approximation.
\newblock In \emph{The Thirty Sixth Annual Conference on Learning Theory}, pages 987--1063. PMLR, 2023.

\bibitem[Anthony et~al.(2017)Anthony, Tian, and Barber]{anthony2017thinking}
T.~Anthony, Z.~Tian, and D.~Barber.
\newblock Thinking fast and slow with deep learning and tree search.
\newblock \emph{Advances in neural information processing systems}, 30, 2017.

\bibitem[Anthropic(2023)]{Anthropic@claude}
Anthropic.
\newblock Introducing claude.
\newblock 2023.
\newblock URL \url{https://www.anthropic.com/index/introducing-claude}.

\bibitem[Auer et~al.(2002)Auer, Cesa-Bianchi, and Fischer]{auer2002finite}
P.~Auer, N.~Cesa-Bianchi, and P.~Fischer.
\newblock Finite-time analysis of the multiarmed bandit problem.
\newblock \emph{Machine learning}, 47:\penalty0 235--256, 2002.

\bibitem[Azar et~al.(2023)Azar, Rowland, Piot, Guo, Calandriello, Valko, and Munos]{azar2023general}
M.~G. Azar, M.~Rowland, B.~Piot, D.~Guo, D.~Calandriello, M.~Valko, and R.~Munos.
\newblock A general theoretical paradigm to understand learning from human preferences.
\newblock \emph{arXiv preprint arXiv:2310.12036}, 2023.

\bibitem[Bai et~al.(2022)Bai, Jones, Ndousse, Askell, Chen, DasSarma, Drain, Fort, Ganguli, Henighan, et~al.]{bai2022training}
Y.~Bai, A.~Jones, K.~Ndousse, A.~Askell, A.~Chen, N.~DasSarma, D.~Drain, S.~Fort, D.~Ganguli, T.~Henighan, et~al.
\newblock Training a helpful and harmless assistant with reinforcement learning from human feedback.
\newblock \emph{arXiv preprint arXiv:2204.05862}, 2022.

\bibitem[Cai et~al.(2020)Cai, Yang, Jin, and Wang]{cai2020provably}
Q.~Cai, Z.~Yang, C.~Jin, and Z.~Wang.
\newblock Provably efficient exploration in policy optimization.
\newblock In \emph{International Conference on Machine Learning}, pages 1283--1294. PMLR, 2020.

\bibitem[Cen et~al.(2024)Cen, Mei, Goshvadi, Dai, Yang, Yang, Schuurmans, Chi, and Dai]{cen2024value}
S.~Cen, J.~Mei, K.~Goshvadi, H.~Dai, T.~Yang, S.~Yang, D.~Schuurmans, Y.~Chi, and B.~Dai.
\newblock Value-incentivized preference optimization: A unified approach to online and offline rlhf.
\newblock \emph{arXiv preprint arXiv:2405.19320}, 2024.

\bibitem[Chen et~al.(2024{\natexlab{a}})Chen, Liao, Li, and Fan]{chen2024step}
G.~Chen, M.~Liao, C.~Li, and K.~Fan.
\newblock Step-level value preference optimization for mathematical reasoning.
\newblock \emph{arXiv preprint arXiv:2406.10858}, 2024{\natexlab{a}}.

\bibitem[Chen et~al.(2022)Chen, Ma, Wang, and Cohen]{chen2022program}
W.~Chen, X.~Ma, X.~Wang, and W.~W. Cohen.
\newblock Program of thoughts prompting: Disentangling computation from reasoning for numerical reasoning tasks.
\newblock \emph{arXiv preprint arXiv:2211.12588}, 2022.

\bibitem[Chen et~al.(2024{\natexlab{b}})Chen, Deng, Yuan, Ji, and Gu]{chen2024self}
Z.~Chen, Y.~Deng, H.~Yuan, K.~Ji, and Q.~Gu.
\newblock Self-play fine-tuning converts weak language models to strong language models.
\newblock \emph{arXiv preprint arXiv:2401.01335}, 2024{\natexlab{b}}.

\bibitem[Choshen et~al.(2019)Choshen, Fox, Aizenbud, and Abend]{choshen2019weaknesses}
L.~Choshen, L.~Fox, Z.~Aizenbud, and O.~Abend.
\newblock On the weaknesses of reinforcement learning for neural machine translation.
\newblock \emph{arXiv preprint arXiv:1907.01752}, 2019.

\bibitem[Christiano et~al.(2017)Christiano, Leike, Brown, Martic, Legg, and Amodei]{christiano2017deep}
P.~F. Christiano, J.~Leike, T.~Brown, M.~Martic, S.~Legg, and D.~Amodei.
\newblock Deep reinforcement learning from human preferences.
\newblock \emph{Advances in neural information processing systems}, 30, 2017.

\bibitem[Cobbe et~al.(2021{\natexlab{a}})Cobbe, Kosaraju, Bavarian, Chen, Jun, Kaiser, Plappert, Tworek, Hilton, Nakano, Hesse, and Schulman]{cobbe2021gsm8k}
K.~Cobbe, V.~Kosaraju, M.~Bavarian, M.~Chen, H.~Jun, L.~Kaiser, M.~Plappert, J.~Tworek, J.~Hilton, R.~Nakano, C.~Hesse, and J.~Schulman.
\newblock Training verifiers to solve math word problems.
\newblock \emph{arXiv preprint arXiv:2110.14168}, 2021{\natexlab{a}}.

\bibitem[Cobbe et~al.(2021{\natexlab{b}})Cobbe, Kosaraju, Bavarian, Chen, Jun, Kaiser, Plappert, Tworek, Hilton, Nakano, et~al.]{cobbe2021training}
K.~Cobbe, V.~Kosaraju, M.~Bavarian, M.~Chen, H.~Jun, L.~Kaiser, M.~Plappert, J.~Tworek, J.~Hilton, R.~Nakano, et~al.
\newblock Training verifiers to solve math word problems.
\newblock \emph{arXiv preprint arXiv:2110.14168}, 2021{\natexlab{b}}.

\bibitem[Coste et~al.(2023)Coste, Anwar, Kirk, and Krueger]{coste2023reward}
T.~Coste, U.~Anwar, R.~Kirk, and D.~Krueger.
\newblock Reward model ensembles help mitigate overoptimization.
\newblock \emph{arXiv preprint arXiv:2310.02743}, 2023.

\bibitem[Dong et~al.(2023)Dong, Xiong, Goyal, Zhang, Chow, Pan, Diao, Zhang, SHUM, and Zhang]{dong2023raft}
H.~Dong, W.~Xiong, D.~Goyal, Y.~Zhang, W.~Chow, R.~Pan, S.~Diao, J.~Zhang, K.~SHUM, and T.~Zhang.
\newblock {RAFT}: Reward ranked finetuning for generative foundation model alignment.
\newblock \emph{Transactions on Machine Learning Research}, 2023.
\newblock ISSN 2835-8856.
\newblock URL \url{https://openreview.net/forum?id=m7p5O7zblY}.

\bibitem[Dong et~al.(2024)Dong, Xiong, Pang, Wang, Zhao, Zhou, Jiang, Sahoo, Xiong, and Zhang]{dong2024rlhf}
H.~Dong, W.~Xiong, B.~Pang, H.~Wang, H.~Zhao, Y.~Zhou, N.~Jiang, D.~Sahoo, C.~Xiong, and T.~Zhang.
\newblock Rlhf workflow: From reward modeling to online rlhf.
\newblock \emph{arXiv preprint arXiv:2405.07863}, 2024.

\bibitem[Engstrom et~al.(2020)Engstrom, Ilyas, Santurkar, Tsipras, Janoos, Rudolph, and Madry]{engstrom2020implementation}
L.~Engstrom, A.~Ilyas, S.~Santurkar, D.~Tsipras, F.~Janoos, L.~Rudolph, and A.~Madry.
\newblock Implementation matters in deep policy gradients: A case study on ppo and trpo.
\newblock \emph{arXiv preprint arXiv:2005.12729}, 2020.

\bibitem[Ethayarajh et~al.(2024)Ethayarajh, Xu, Muennighoff, Jurafsky, and Kiela]{ethayarajh2024kto}
K.~Ethayarajh, W.~Xu, N.~Muennighoff, D.~Jurafsky, and D.~Kiela.
\newblock Kto: Model alignment as prospect theoretic optimization.
\newblock \emph{arXiv preprint arXiv:2402.01306}, 2024.

\bibitem[Gao et~al.(2023{\natexlab{a}})Gao, Madaan, Zhou, Alon, Liu, Yang, Callan, and Neubig]{gao2023pal}
L.~Gao, A.~Madaan, S.~Zhou, U.~Alon, P.~Liu, Y.~Yang, J.~Callan, and G.~Neubig.
\newblock Pal: Program-aided language models.
\newblock In \emph{International Conference on Machine Learning}, pages 10764--10799. PMLR, 2023{\natexlab{a}}.

\bibitem[Gao et~al.(2023{\natexlab{b}})Gao, Schulman, and Hilton]{gao2023scaling}
L.~Gao, J.~Schulman, and J.~Hilton.
\newblock Scaling laws for reward model overoptimization.
\newblock In \emph{International Conference on Machine Learning}, pages 10835--10866. PMLR, 2023{\natexlab{b}}.

\bibitem[Gou et~al.(2023{\natexlab{a}})Gou, Shao, Gong, Shen, Yang, Duan, and Chen]{gou2023critic}
Z.~Gou, Z.~Shao, Y.~Gong, Y.~Shen, Y.~Yang, N.~Duan, and W.~Chen.
\newblock Critic: Large language models can self-correct with tool-interactive critiquing.
\newblock \emph{arXiv preprint arXiv:2305.11738}, 2023{\natexlab{a}}.

\bibitem[Gou et~al.(2023{\natexlab{b}})Gou, Shao, Gong, Yang, Huang, Duan, Chen, et~al.]{gou2023tora}
Z.~Gou, Z.~Shao, Y.~Gong, Y.~Yang, M.~Huang, N.~Duan, W.~Chen, et~al.
\newblock Tora: A tool-integrated reasoning agent for mathematical problem solving.
\newblock \emph{arXiv preprint arXiv:2309.17452}, 2023{\natexlab{b}}.

\bibitem[Gui et~al.(2024)Gui, G{\^a}rbacea, and Veitch]{gui2024bonbon}
L.~Gui, C.~G{\^a}rbacea, and V.~Veitch.
\newblock Bonbon alignment for large language models and the sweetness of best-of-n sampling.
\newblock \emph{arXiv preprint arXiv:2406.00832}, 2024.

\bibitem[Gulcehre et~al.(2023)Gulcehre, Paine, Srinivasan, Konyushkova, Weerts, Sharma, Siddhant, Ahern, Wang, Gu, et~al.]{gulcehre2023reinforced}
C.~Gulcehre, T.~L. Paine, S.~Srinivasan, K.~Konyushkova, L.~Weerts, A.~Sharma, A.~Siddhant, A.~Ahern, M.~Wang, C.~Gu, et~al.
\newblock Reinforced self-training (rest) for language modeling.
\newblock \emph{arXiv preprint arXiv:2308.08998}, 2023.

\bibitem[Guo et~al.(2024{\natexlab{a}})Guo, Xiong, and Wang]{guo2024alignment}
S.~Guo, W.~Xiong, and C.~Wang.
\newblock "alignment guidebook.
\newblock \emph{Notion Blog}, 2024{\natexlab{a}}.

\bibitem[Guo et~al.(2024{\natexlab{b}})Guo, Zhang, Liu, Liu, Khalman, Llinares, Rame, Mesnard, Zhao, Piot, et~al.]{guo2024direct}
S.~Guo, B.~Zhang, T.~Liu, T.~Liu, M.~Khalman, F.~Llinares, A.~Rame, T.~Mesnard, Y.~Zhao, B.~Piot, et~al.
\newblock Direct language model alignment from online ai feedback.
\newblock \emph{arXiv preprint arXiv:2402.04792}, 2024{\natexlab{b}}.

\bibitem[Hendrycks et~al.(2021)Hendrycks, Burns, Kadavath, Arora, Basart, Tang, Song, and Steinhardt]{hendrycks2021measuring}
D.~Hendrycks, C.~Burns, S.~Kadavath, A.~Arora, S.~Basart, E.~Tang, D.~Song, and J.~Steinhardt.
\newblock Measuring mathematical problem solving with the math dataset.
\newblock \emph{arXiv preprint arXiv:2103.03874}, 2021.

\bibitem[Hoang~Tran(2024)]{snorkelai@pair}
B.~H. Hoang~Tran, Chris~Glaze.
\newblock Snorkel-mistral-pairrm-dpo.
\newblock \url{https://huggingface.co/snorkelai/Snorkel-Mistral-PairRM-DPO}, 2024.
\newblock URL \url{https://huggingface.co/snorkelai/Snorkel-Mistral-PairRM-DPO}.

\bibitem[Hong et~al.(2024)Hong, Lee, and Thorne]{hong2024orpo}
J.~Hong, N.~Lee, and J.~Thorne.
\newblock Orpo: Monolithic preference optimization without reference model.
\newblock \emph{arXiv preprint arXiv:2403.07691}, 2\penalty0 (4):\penalty0 5, 2024.

\bibitem[Jiang et~al.(2023)Jiang, Sablayrolles, Mensch, Bamford, Chaplot, Casas, Bressand, Lengyel, Lample, Saulnier, et~al.]{jiang2023mistral}
A.~Q. Jiang, A.~Sablayrolles, A.~Mensch, C.~Bamford, D.~S. Chaplot, D.~d.~l. Casas, F.~Bressand, G.~Lengyel, G.~Lample, L.~Saulnier, et~al.
\newblock Mistral 7b.
\newblock \emph{arXiv preprint arXiv:2310.06825}, 2023.

\bibitem[Jiao et~al.(2024)Jiao, Qin, Liu, Chen, and Joty]{jiao2024learning}
F.~Jiao, C.~Qin, Z.~Liu, N.~F. Chen, and S.~Joty.
\newblock Learning planning-based reasoning by trajectories collection and process reward synthesizing.
\newblock \emph{arXiv preprint arXiv:2402.00658}, 2024.

\bibitem[Lai et~al.(2024)Lai, Tian, Chen, Yang, Peng, and Jia]{lai2024step}
X.~Lai, Z.~Tian, Y.~Chen, S.~Yang, X.~Peng, and J.~Jia.
\newblock Step-dpo: Step-wise preference optimization for long-chain reasoning of llms.
\newblock \emph{arXiv preprint arXiv:2406.18629}, 2024.

\bibitem[Lightman et~al.(2023)Lightman, Kosaraju, Burda, Edwards, Baker, Lee, Leike, Schulman, Sutskever, and Cobbe]{lightman2023let}
H.~Lightman, V.~Kosaraju, Y.~Burda, H.~Edwards, B.~Baker, T.~Lee, J.~Leike, J.~Schulman, I.~Sutskever, and K.~Cobbe.
\newblock Let's verify step by step.
\newblock \emph{arXiv preprint arXiv:2305.20050}, 2023.

\bibitem[Lin et~al.(2023)Lin, Tan, Lin, Zheng, Pi, Zhang, Diao, Wang, Zhao, Yao, et~al.]{lin2023speciality}
Y.~Lin, L.~Tan, H.~Lin, Z.~Zheng, R.~Pi, J.~Zhang, S.~Diao, H.~Wang, H.~Zhao, Y.~Yao, et~al.
\newblock Speciality vs generality: An empirical study on catastrophic forgetting in fine-tuning foundation models.
\newblock \emph{arXiv preprint arXiv:2309.06256}, 2023.

\bibitem[Liu and Yao(2024)]{liu2024augmenting}
H.~Liu and A.~C.-C. Yao.
\newblock Augmenting math word problems via iterative question composing.
\newblock \emph{arXiv preprint arXiv:2401.09003}, 2024.

\bibitem[Liu et~al.(2023{\natexlab{a}})Liu, Netrapalli, Szepesvari, and Jin]{liu2023optimistic}
Q.~Liu, P.~Netrapalli, C.~Szepesvari, and C.~Jin.
\newblock Optimistic mle: A generic model-based algorithm for partially observable sequential decision making.
\newblock In \emph{Proceedings of the 55th Annual ACM Symposium on Theory of Computing}, pages 363--376, 2023{\natexlab{a}}.

\bibitem[Liu et~al.(2023{\natexlab{b}})Liu, Zhao, Joshi, Khalman, Saleh, Liu, and Liu]{liu2023statistical}
T.~Liu, Y.~Zhao, R.~Joshi, M.~Khalman, M.~Saleh, P.~J. Liu, and J.~Liu.
\newblock Statistical rejection sampling improves preference optimization.
\newblock \emph{arXiv preprint arXiv:2309.06657}, 2023{\natexlab{b}}.

\bibitem[Liu et~al.(2024{\natexlab{a}})Liu, Qin, Wu, Shen, Khalman, Joshi, Zhao, Saleh, Baumgartner, Liu, et~al.]{liu2024lipo}
T.~Liu, Z.~Qin, J.~Wu, J.~Shen, M.~Khalman, R.~Joshi, Y.~Zhao, M.~Saleh, S.~Baumgartner, J.~Liu, et~al.
\newblock Lipo: Listwise preference optimization through learning-to-rank.
\newblock \emph{arXiv preprint arXiv:2402.01878}, 2024{\natexlab{a}}.

\bibitem[Liu et~al.(2024{\natexlab{b}})Liu, Lu, Zhang, Liu, Guo, Yang, Blanchet, and Wang]{liu2024provably}
Z.~Liu, M.~Lu, S.~Zhang, B.~Liu, H.~Guo, Y.~Yang, J.~Blanchet, and Z.~Wang.
\newblock Provably mitigating overoptimization in rlhf: Your sft loss is implicitly an adversarial regularizer.
\newblock \emph{arXiv preprint arXiv:2405.16436}, 2024{\natexlab{b}}.

\bibitem[Lu et~al.(2024)Lu, Zhou, Wang, Ren, Shi, Pan, and Zhan]{lu2024step}
Z.~Lu, A.~Zhou, K.~Wang, H.~Ren, W.~Shi, J.~Pan, and M.~Zhan.
\newblock Step-controlled dpo: Leveraging stepwise error for enhanced mathematical reasoning.
\newblock \emph{arXiv preprint arXiv:2407.00782}, 2024.

\bibitem[Luo et~al.(2023)Luo, Sun, Xu, Zhao, Lou, Tao, Geng, Lin, Chen, and Zhang]{luo2023wizardmath}
H.~Luo, Q.~Sun, C.~Xu, P.~Zhao, J.~Lou, C.~Tao, X.~Geng, Q.~Lin, S.~Chen, and D.~Zhang.
\newblock Wizardmath: Empowering mathematical reasoning for large language models via reinforced evol-instruct.
\newblock \emph{arXiv preprint arXiv:2308.09583}, 2023.

\bibitem[Meng et~al.(2024)Meng, Xia, and Chen]{meng2024simpo}
Y.~Meng, M.~Xia, and D.~Chen.
\newblock Simpo: Simple preference optimization with a reference-free reward.
\newblock \emph{arXiv preprint arXiv:2405.14734}, 2024.

\bibitem[Meta(2024)]{meta_llama3}
Meta.
\newblock Introducing meta llama 3: The most capable openly available llm to date.
\newblock \emph{Meta AI Blog}, 2024.
\newblock \url{https://ai.meta.com/blog/meta-llama-3/}.

\bibitem[Mishra et~al.(2022)Mishra, Finlayson, Lu, Tang, Welleck, Baral, Rajpurohit, Tafjord, Sabharwal, Clark, et~al.]{mishra2022lila}
S.~Mishra, M.~Finlayson, P.~Lu, L.~Tang, S.~Welleck, C.~Baral, T.~Rajpurohit, O.~Tafjord, A.~Sabharwal, P.~Clark, et~al.
\newblock Lila: A unified benchmark for mathematical reasoning.
\newblock \emph{arXiv preprint arXiv:2210.17517}, 2022.

\bibitem[Mitra et~al.(2024)Mitra, Khanpour, Rosset, and Awadallah]{mitra2024orca}
A.~Mitra, H.~Khanpour, C.~Rosset, and A.~Awadallah.
\newblock Orca-math: Unlocking the potential of slms in grade school math.
\newblock \emph{arXiv preprint arXiv:2402.14830}, 2024.

\bibitem[Munos et~al.(2023)Munos, Valko, Calandriello, Azar, Rowland, Guo, Tang, Geist, Mesnard, Michi, et~al.]{munos2023nash}
R.~Munos, M.~Valko, D.~Calandriello, M.~G. Azar, M.~Rowland, Z.~D. Guo, Y.~Tang, M.~Geist, T.~Mesnard, A.~Michi, et~al.
\newblock Nash learning from human feedback.
\newblock \emph{arXiv preprint arXiv:2312.00886}, 2023.

\bibitem[Nemirovskij and Yudin(1983)]{nemirovskij1983problem}
A.~S. Nemirovskij and D.~B. Yudin.
\newblock Problem complexity and method efficiency in optimization.
\newblock 1983.

\bibitem[OpenAI(2023)]{OpenAI2023GPT4TR}
OpenAI.
\newblock Gpt-4 technical report.
\newblock \emph{ArXiv}, abs/2303.08774, 2023.

\bibitem[Ouyang et~al.(2022)Ouyang, Wu, Jiang, Almeida, Wainwright, Mishkin, Zhang, Agarwal, Slama, Ray, et~al.]{ouyang2022training}
L.~Ouyang, J.~Wu, X.~Jiang, D.~Almeida, C.~Wainwright, P.~Mishkin, C.~Zhang, S.~Agarwal, K.~Slama, A.~Ray, et~al.
\newblock Training language models to follow instructions with human feedback.
\newblock \emph{Advances in Neural Information Processing Systems}, 35:\penalty0 27730--27744, 2022.

\bibitem[Pace et~al.(2024)Pace, Mallinson, Malmi, Krause, and Severyn]{pace2024west}
A.~Pace, J.~Mallinson, E.~Malmi, S.~Krause, and A.~Severyn.
\newblock West-of-n: Synthetic preference generation for improved reward modeling.
\newblock \emph{arXiv preprint arXiv:2401.12086}, 2024.

\bibitem[Pang et~al.(2024)Pang, Yuan, Cho, He, Sukhbaatar, and Weston]{pang2024iterative}
R.~Y. Pang, W.~Yuan, K.~Cho, H.~He, S.~Sukhbaatar, and J.~Weston.
\newblock Iterative reasoning preference optimization.
\newblock \emph{arXiv preprint arXiv:2404.19733}, 2024.

\bibitem[Pi et~al.(2024)Pi, Han, Xiong, Zhang, Liu, Pan, and Zhang]{pi2024strengthening}
R.~Pi, T.~Han, W.~Xiong, J.~Zhang, R.~Liu, R.~Pan, and T.~Zhang.
\newblock Strengthening multimodal large language model with bootstrapped preference optimization.
\newblock \emph{arXiv preprint arXiv:2403.08730}, 2024.

\bibitem[Rafailov et~al.(2023)Rafailov, Sharma, Mitchell, Ermon, Manning, and Finn]{rafailov2023direct}
R.~Rafailov, A.~Sharma, E.~Mitchell, S.~Ermon, C.~D. Manning, and C.~Finn.
\newblock Direct preference optimization: Your language model is secretly a reward model.
\newblock \emph{arXiv preprint arXiv:2305.18290}, 2023.

\bibitem[Rafailov et~al.(2024)Rafailov, Hejna, Park, and Finn]{rafailov2024r}
R.~Rafailov, J.~Hejna, R.~Park, and C.~Finn.
\newblock From r to q*: Your language model is secretly a q-function.
\newblock \emph{arXiv preprint arXiv:2404.12358}, 2024.

\bibitem[Richemond et~al.(2024)Richemond, Tang, Guo, Calandriello, Azar, Rafailov, Pires, Tarassov, Spangher, Ellsworth, et~al.]{richemond2024offline}
P.~H. Richemond, Y.~Tang, D.~Guo, D.~Calandriello, M.~G. Azar, R.~Rafailov, B.~A. Pires, E.~Tarassov, L.~Spangher, W.~Ellsworth, et~al.
\newblock Offline regularised reinforcement learning for large language models alignment.
\newblock \emph{arXiv preprint arXiv:2405.19107}, 2024.

\bibitem[Rosset et~al.(2024)Rosset, Cheng, Mitra, Santacroce, Awadallah, and Xie]{rosset2024direct}
C.~Rosset, C.-A. Cheng, A.~Mitra, M.~Santacroce, A.~Awadallah, and T.~Xie.
\newblock Direct nash optimization: Teaching language models to self-improve with general preferences.
\newblock \emph{arXiv preprint arXiv:2404.03715}, 2024.

\bibitem[Schulman et~al.(2017)Schulman, Wolski, Dhariwal, Radford, and Klimov]{schulman2017proximal}
J.~Schulman, F.~Wolski, P.~Dhariwal, A.~Radford, and O.~Klimov.
\newblock Proximal policy optimization algorithms.
\newblock \emph{arXiv preprint arXiv:1707.06347}, 2017.

\bibitem[Shani et~al.(2024)Shani, Rosenberg, Cassel, Lang, Calandriello, Zipori, Noga, Keller, Piot, Szpektor, et~al.]{shani2024multi}
L.~Shani, A.~Rosenberg, A.~Cassel, O.~Lang, D.~Calandriello, A.~Zipori, H.~Noga, O.~Keller, B.~Piot, I.~Szpektor, et~al.
\newblock Multi-turn reinforcement learning from preference human feedback.
\newblock \emph{arXiv preprint arXiv:2405.14655}, 2024.

\bibitem[Shao et~al.(2022)Shao, Huang, and Huang]{shao2022chaining}
Z.~Shao, F.~Huang, and M.~Huang.
\newblock Chaining simultaneous thoughts for numerical reasoning.
\newblock \emph{arXiv preprint arXiv:2211.16482}, 2022.

\bibitem[Shao et~al.(2024)Shao, Wang, Zhu, Xu, Song, Zhang, Li, Wu, and Guo]{shao2024deepseekmath}
Z.~Shao, P.~Wang, Q.~Zhu, R.~Xu, J.~Song, M.~Zhang, Y.~Li, Y.~Wu, and D.~Guo.
\newblock Deepseekmath: Pushing the limits of mathematical reasoning in open language models.
\newblock \emph{arXiv preprint arXiv:2402.03300}, 2024.

\bibitem[Singh et~al.(2023)Singh, Co-Reyes, Agarwal, Anand, Patil, Liu, Harrison, Lee, Xu, Parisi, et~al.]{singh2023beyond}
A.~Singh, J.~D. Co-Reyes, R.~Agarwal, A.~Anand, P.~Patil, P.~J. Liu, J.~Harrison, J.~Lee, K.~Xu, A.~Parisi, et~al.
\newblock Beyond human data: Scaling self-training for problem-solving with language models.
\newblock \emph{arXiv preprint arXiv:2312.06585}, 2023.

\bibitem[Sutton and Barto(2018)]{sutton2018reinforcement}
R.~S. Sutton and A.~G. Barto.
\newblock \emph{Reinforcement learning: An introduction}.
\newblock MIT press, 2018.

\bibitem[Swamy et~al.(2024)Swamy, Dann, Kidambi, Wu, and Agarwal]{swamy2024minimaximalist}
G.~Swamy, C.~Dann, R.~Kidambi, Z.~S. Wu, and A.~Agarwal.
\newblock A minimaximalist approach to reinforcement learning from human feedback.
\newblock \emph{arXiv preprint arXiv:2401.04056}, 2024.

\bibitem[Tajwar et~al.(2024)Tajwar, Singh, Sharma, Rafailov, Schneider, Xie, Ermon, Finn, and Kumar]{tajwar2024preference}
F.~Tajwar, A.~Singh, A.~Sharma, R.~Rafailov, J.~Schneider, T.~Xie, S.~Ermon, C.~Finn, and A.~Kumar.
\newblock Preference fine-tuning of llms should leverage suboptimal, on-policy data.
\newblock \emph{arXiv preprint arXiv:2404.14367}, 2024.

\bibitem[Tang et~al.(2024)Tang, Guo, Zheng, Calandriello, Munos, Rowland, Richemond, Valko, Pires, and Piot]{tang2024generalized}
Y.~Tang, Z.~D. Guo, Z.~Zheng, D.~Calandriello, R.~Munos, M.~Rowland, P.~H. Richemond, M.~Valko, B.~{\'A}. Pires, and B.~Piot.
\newblock Generalized preference optimization: A unified approach to offline alignment.
\newblock \emph{arXiv preprint arXiv:2402.05749}, 2024.

\bibitem[Team(2024)]{team2024codegemma}
C.~Team.
\newblock Codegemma: Open code models based on gemma.
\newblock \emph{arXiv preprint arXiv:2406.11409}, 2024.

\bibitem[Team et~al.(2023)Team, Anil, Borgeaud, Wu, Alayrac, Yu, Soricut, Schalkwyk, Dai, Hauth, et~al.]{team2023gemini}
G.~Team, R.~Anil, S.~Borgeaud, Y.~Wu, J.-B. Alayrac, J.~Yu, R.~Soricut, J.~Schalkwyk, A.~M. Dai, A.~Hauth, et~al.
\newblock Gemini: a family of highly capable multimodal models.
\newblock \emph{arXiv preprint arXiv:2312.11805}, 2023.

\bibitem[Team et~al.(2024)Team, Mesnard, Hardin, Dadashi, Bhupatiraju, Pathak, Sifre, Rivi{\`e}re, Kale, Love, et~al.]{team2024gemma}
G.~Team, T.~Mesnard, C.~Hardin, R.~Dadashi, S.~Bhupatiraju, S.~Pathak, L.~Sifre, M.~Rivi{\`e}re, M.~S. Kale, J.~Love, et~al.
\newblock Gemma: Open models based on gemini research and technology.
\newblock \emph{arXiv preprint arXiv:2403.08295}, 2024.

\bibitem[Tong et~al.(2024)Tong, Zhang, Wang, Wu, and He]{tongdart}
Y.~Tong, X.~Zhang, R.~Wang, R.~Wu, and J.~He.
\newblock Dart-math: Difficulty-aware rejection tuning for mathematical problem-solving.
\newblock 2024.

\bibitem[Toshniwal et~al.(2024)Toshniwal, Moshkov, Narenthiran, Gitman, Jia, and Gitman]{toshniwal2024openmathinstruct}
S.~Toshniwal, I.~Moshkov, S.~Narenthiran, D.~Gitman, F.~Jia, and I.~Gitman.
\newblock Openmathinstruct-1: A 1.8 million math instruction tuning dataset.
\newblock \emph{arXiv preprint arXiv:2402.10176}, 2024.

\bibitem[Touvron et~al.(2023)Touvron, Martin, Stone, Albert, Almahairi, Babaei, Bashlykov, Batra, Bhargava, Bhosale, et~al.]{touvron2023llama}
H.~Touvron, L.~Martin, K.~Stone, P.~Albert, A.~Almahairi, Y.~Babaei, N.~Bashlykov, S.~Batra, P.~Bhargava, S.~Bhosale, et~al.
\newblock Llama 2: Open foundation and fine-tuned chat models.
\newblock \emph{arXiv preprint arXiv:2307.09288}, 2023.

\bibitem[Tunstall et~al.(2023)Tunstall, Beeching, Lambert, Rajani, Rasul, Belkada, Huang, von Werra, Fourrier, Habib, et~al.]{tunstall2023zephyr}
L.~Tunstall, E.~Beeching, N.~Lambert, N.~Rajani, K.~Rasul, Y.~Belkada, S.~Huang, L.~von Werra, C.~Fourrier, N.~Habib, et~al.
\newblock Zephyr: Direct distillation of lm alignment.
\newblock \emph{arXiv preprint arXiv:2310.16944}, 2023.

\bibitem[Uesato et~al.(2022)Uesato, Kushman, Kumar, Song, Siegel, Wang, Creswell, Irving, and Higgins]{uesato2022solving}
J.~Uesato, N.~Kushman, R.~Kumar, F.~Song, N.~Siegel, L.~Wang, A.~Creswell, G.~Irving, and I.~Higgins.
\newblock Solving math word problems with process-and outcome-based feedback.
\newblock \emph{arXiv preprint arXiv:2211.14275}, 2022.

\bibitem[Wang et~al.(2023{\natexlab{a}})Wang, Li, Shao, Xu, Dai, Li, Chen, Wu, and Sui]{wang2023math}
P.~Wang, L.~Li, Z.~Shao, R.~Xu, D.~Dai, Y.~Li, D.~Chen, Y.~Wu, and Z.~Sui.
\newblock Math-shepherd: Verify and reinforce llms step-by-step without human annotations.
\newblock \emph{CoRR, abs/2312.08935}, 2023{\natexlab{a}}.

\bibitem[Wang et~al.(2024)Wang, Wang, Liu, Chen, Yuan, Peng, and Ji]{mint2024a}
X.~Wang, Z.~Wang, J.~Liu, Y.~Chen, L.~Yuan, H.~Peng, and H.~Ji.
\newblock Mint: Multi-turn interactive evaluation for tool-augmented llms with language feedback.
\newblock In \emph{Proc. The Twelfth International Conference on Learning Representations (ICLR2024)}, 2024.

\bibitem[Wang et~al.(2023{\natexlab{b}})Wang, Liu, and Jin]{wang2023rlhf}
Y.~Wang, Q.~Liu, and C.~Jin.
\newblock Is rlhf more difficult than standard rl?
\newblock \emph{arXiv preprint arXiv:2306.14111}, 2023{\natexlab{b}}.

\bibitem[Wei et~al.(2022)Wei, Wang, Schuurmans, Bosma, Xia, Chi, Le, Zhou, et~al.]{wei2022chain}
J.~Wei, X.~Wang, D.~Schuurmans, M.~Bosma, F.~Xia, E.~Chi, Q.~V. Le, D.~Zhou, et~al.
\newblock Chain-of-thought prompting elicits reasoning in large language models.
\newblock \emph{Advances in neural information processing systems}, 35:\penalty0 24824--24837, 2022.

\bibitem[Williams(1992)]{williams1992simple}
R.~J. Williams.
\newblock Simple statistical gradient-following algorithms for connectionist reinforcement learning.
\newblock \emph{Machine learning}, 8:\penalty0 229--256, 1992.

\bibitem[Williams and Peng(1991)]{williams1991function}
R.~J. Williams and J.~Peng.
\newblock Function optimization using connectionist reinforcement learning algorithms.
\newblock \emph{Connection Science}, 3\penalty0 (3):\penalty0 241--268, 1991.

\bibitem[Xie et~al.(2022)Xie, Foster, Bai, Jiang, and Kakade]{xie2022role}
T.~Xie, D.~J. Foster, Y.~Bai, N.~Jiang, and S.~M. Kakade.
\newblock The role of coverage in online reinforcement learning.
\newblock \emph{arXiv preprint arXiv:2210.04157}, 2022.

\bibitem[Xie et~al.(2024{\natexlab{a}})Xie, Foster, Krishnamurthy, Rosset, Awadallah, and Rakhlin]{xie2024exploratory}
T.~Xie, D.~J. Foster, A.~Krishnamurthy, C.~Rosset, A.~Awadallah, and A.~Rakhlin.
\newblock Exploratory preference optimization: Harnessing implicit q*-approximation for sample-efficient rlhf.
\newblock \emph{arXiv preprint arXiv:2405.21046}, 2024{\natexlab{a}}.

\bibitem[Xie et~al.(2024{\natexlab{b}})Xie, Goyal, Zheng, Kan, Lillicrap, Kawaguchi, and Shieh]{xie2024monte}
Y.~Xie, A.~Goyal, W.~Zheng, M.-Y. Kan, T.~P. Lillicrap, K.~Kawaguchi, and M.~Shieh.
\newblock Monte carlo tree search boosts reasoning via iterative preference learning.
\newblock \emph{arXiv preprint arXiv:2405.00451}, 2024{\natexlab{b}}.

\bibitem[Xiong et~al.()Xiong, Dong, Ye, Wang, Zhong, Ji, Jiang, and Zhang]{xiong2024iterative}
W.~Xiong, H.~Dong, C.~Ye, Z.~Wang, H.~Zhong, H.~Ji, N.~Jiang, and T.~Zhang.
\newblock Iterative preference learning from human feedback: Bridging theory and practice for rlhf under kl-constraint.
\newblock In \emph{Forty-first International Conference on Machine Learning}.

\bibitem[Xu et~al.(2023)Xu, Lee, Sukhbaatar, and Weston]{xu2023some}
J.~Xu, A.~Lee, S.~Sukhbaatar, and J.~Weston.
\newblock Some things are more cringe than others: Preference optimization with the pairwise cringe loss.
\newblock \emph{arXiv preprint arXiv:2312.16682}, 2023.

\bibitem[Yao et~al.(2022)Yao, Zhao, Yu, Du, Shafran, Narasimhan, and Cao]{yao2022react}
S.~Yao, J.~Zhao, D.~Yu, N.~Du, I.~Shafran, K.~Narasimhan, and Y.~Cao.
\newblock React: Synergizing reasoning and acting in language models.
\newblock \emph{arXiv preprint arXiv:2210.03629}, 2022.

\bibitem[Ye et~al.(2024)Ye, Xiong, Zhang, Jiang, and Zhang]{ye2024theoretical}
C.~Ye, W.~Xiong, Y.~Zhang, N.~Jiang, and T.~Zhang.
\newblock A theoretical analysis of nash learning from human feedback under general kl-regularized preference.
\newblock \emph{arXiv preprint arXiv:2402.07314}, 2024.

\bibitem[Yu et~al.(2023)Yu, Jiang, Shi, Yu, Liu, Zhang, Kwok, Li, Weller, and Liu]{yu2023metamath}
L.~Yu, W.~Jiang, H.~Shi, J.~Yu, Z.~Liu, Y.~Zhang, J.~T. Kwok, Z.~Li, A.~Weller, and W.~Liu.
\newblock Metamath: Bootstrap your own mathematical questions for large language models.
\newblock \emph{arXiv preprint arXiv:2309.12284}, 2023.

\bibitem[Yuan et~al.(2024)Yuan, Cui, Wang, Ding, Wang, Deng, Shan, Chen, Xie, Lin, et~al.]{yuan2024advancing}
L.~Yuan, G.~Cui, H.~Wang, N.~Ding, X.~Wang, J.~Deng, B.~Shan, H.~Chen, R.~Xie, Y.~Lin, et~al.
\newblock Advancing llm reasoning generalists with preference trees.
\newblock \emph{arXiv preprint arXiv:2404.02078}, 2024.

\bibitem[Yuan et~al.(2023{\natexlab{a}})Yuan, Yuan, Li, Dong, Tan, and Zhou]{yuan2023scaling}
Z.~Yuan, H.~Yuan, C.~Li, G.~Dong, C.~Tan, and C.~Zhou.
\newblock Scaling relationship on learning mathematical reasoning with large language models.
\newblock \emph{arXiv preprint arXiv:2308.01825}, 2023{\natexlab{a}}.

\bibitem[Yuan et~al.(2023{\natexlab{b}})Yuan, Yuan, Tan, Wang, Huang, and Huang]{yuan2023rrhf}
Z.~Yuan, H.~Yuan, C.~Tan, W.~Wang, S.~Huang, and F.~Huang.
\newblock Rrhf: Rank responses to align language models with human feedback without tears.
\newblock \emph{arXiv preprint arXiv:2304.05302}, 2023{\natexlab{b}}.

\bibitem[Yue et~al.(2023)Yue, Xingwei~Qu, Fu, Huang, Sun, Su, and Chen]{yue2023mammoth}
X.~Yue, G.~Z. Xingwei~Qu, Y.~Fu, W.~Huang, H.~Sun, Y.~Su, and W.~Chen.
\newblock Mammoth: Building math generalist models through hybrid instruction tuning.
\newblock \emph{arXiv preprint arXiv:2309.05653}, 2023.

\bibitem[Yue et~al.(2024)Yue, Zheng, Zhang, and Chen]{yue2024mammoth2}
X.~Yue, T.~Zheng, G.~Zhang, and W.~Chen.
\newblock Mammoth2: Scaling instructions from the web.
\newblock \emph{arXiv preprint arXiv:2405.03548}, 2024.

\bibitem[Zelikman et~al.(2022)Zelikman, Wu, Mu, and Goodman]{zelikman2022star}
E.~Zelikman, Y.~Wu, J.~Mu, and N.~Goodman.
\newblock Star: Bootstrapping reasoning with reasoning.
\newblock \emph{Advances in Neural Information Processing Systems}, 35:\penalty0 15476--15488, 2022.

\bibitem[Zhan et~al.(2023)Zhan, Uehara, Kallus, Lee, and Sun]{zhan2023provable}
W.~Zhan, M.~Uehara, N.~Kallus, J.~D. Lee, and W.~Sun.
\newblock Provable offline reinforcement learning with human feedback.
\newblock \emph{arXiv preprint arXiv:2305.14816}, 2023.

\bibitem[Zhang et~al.(2024{\natexlab{a}})Zhang, Zhou, Wei, Zhao, Sha, Wang, and Wen]{zhang2024evaluating}
B.~Zhang, K.~Zhou, X.~Wei, X.~Zhao, J.~Sha, S.~Wang, and J.-R. Wen.
\newblock Evaluating and improving tool-augmented computation-intensive math reasoning.
\newblock \emph{Advances in Neural Information Processing Systems}, 36, 2024{\natexlab{a}}.

\bibitem[Zhang et~al.(2024{\natexlab{b}})Zhang, Yu, Sharma, Yang, Wang, Hassan, and Wang]{zhang2024self}
S.~Zhang, D.~Yu, H.~Sharma, Z.~Yang, S.~Wang, H.~Hassan, and Z.~Wang.
\newblock Self-exploring language models: Active preference elicitation for online alignment.
\newblock \emph{arXiv preprint arXiv:2405.19332}, 2024{\natexlab{b}}.

\bibitem[Zhang(2023)]{zhang2023mathematical}
T.~Zhang.
\newblock \emph{Mathematical analysis of machine learning algorithms}.
\newblock Cambridge University Press, 2023.

\bibitem[Zhang et~al.(2024{\natexlab{c}})Zhang, Yu, Peng, Song, Tian, Huo, Jiang, Mi, and Yu]{zhang2024iterative}
Y.~Zhang, D.~Yu, B.~Peng, L.~Song, Y.~Tian, M.~Huo, N.~Jiang, H.~Mi, and D.~Yu.
\newblock Iterative nash policy optimization: Aligning llms with general preferences via no-regret learning.
\newblock \emph{arXiv preprint arXiv:2407.00617}, 2024{\natexlab{c}}.

\bibitem[Zhao et~al.(2023)Zhao, Joshi, Liu, Khalman, Saleh, and Liu]{zhao2023slic}
Y.~Zhao, R.~Joshi, T.~Liu, M.~Khalman, M.~Saleh, and P.~J. Liu.
\newblock Slic-hf: Sequence likelihood calibration with human feedback.
\newblock \emph{arXiv preprint arXiv:2305.10425}, 2023.

\bibitem[Zheng et~al.(2024)Zheng, Wang, Ji, Huang, and Peng]{zheng2024weak}
C.~Zheng, Z.~Wang, H.~Ji, M.~Huang, and N.~Peng.
\newblock Weak-to-strong extrapolation expedites alignment.
\newblock \emph{arXiv preprint arXiv:2404.16792}, 2024.

\bibitem[Zheng et~al.(2021)Zheng, Han, and Polu]{zheng2021minif2f}
K.~Zheng, J.~M. Han, and S.~Polu.
\newblock Minif2f: a cross-system benchmark for formal olympiad-level mathematics.
\newblock \emph{arXiv preprint arXiv:2109.00110}, 2021.

\bibitem[Zhong et~al.(2022)Zhong, Xiong, Zheng, Wang, Wang, Yang, and Zhang]{zhong2022gec}
H.~Zhong, W.~Xiong, S.~Zheng, L.~Wang, Z.~Wang, Z.~Yang, and T.~Zhang.
\newblock Gec: A unified framework for interactive decision making in mdp, pomdp, and beyond.
\newblock \emph{arXiv preprint arXiv:2211.01962}, 2022.

\bibitem[Zhong et~al.(2024)Zhong, Feng, Xiong, Zhao, He, Bian, and Wang]{zhong2024dpo}
H.~Zhong, G.~Feng, W.~Xiong, L.~Zhao, D.~He, J.~Bian, and L.~Wang.
\newblock Dpo meets ppo: Reinforced token optimization for rlhf.
\newblock \emph{arXiv preprint arXiv:2404.18922}, 2024.

\bibitem[Zhou et~al.(2022)Zhou, Sch{\"a}rli, Hou, Wei, Scales, Wang, Schuurmans, Cui, Bousquet, Le, et~al.]{zhou2022least}
D.~Zhou, N.~Sch{\"a}rli, L.~Hou, J.~Wei, N.~Scales, X.~Wang, D.~Schuurmans, C.~Cui, O.~Bousquet, Q.~Le, et~al.
\newblock Least-to-most prompting enables complex reasoning in large language models.
\newblock \emph{arXiv preprint arXiv:2205.10625}, 2022.

\bibitem[Zhu et~al.(2022)Zhu, Wang, Zhang, Zhang, Huang, Gan, Zhang, and Yang]{zhu2022solving}
X.~Zhu, J.~Wang, L.~Zhang, Y.~Zhang, Y.~Huang, R.~Gan, J.~Zhang, and Y.~Yang.
\newblock Solving math word problems via cooperative reasoning induced language models.
\newblock \emph{arXiv preprint arXiv:2210.16257}, 2022.

\bibitem[Ziebart(2010)]{ziebart2010modeling}
B.~D. Ziebart.
\newblock \emph{Modeling purposeful adaptive behavior with the principle of maximum causal entropy}.
\newblock Carnegie Mellon University, 2010.

\bibitem[Ziegler et~al.(2019)Ziegler, Stiennon, Wu, Brown, Radford, Amodei, Christiano, and Irving]{ziegler2019fine}
D.~M. Ziegler, N.~Stiennon, J.~Wu, T.~B. Brown, A.~Radford, D.~Amodei, P.~Christiano, and G.~Irving.
\newblock Fine-tuning language models from human preferences.
\newblock \emph{arXiv preprint arXiv:1909.08593}, 2019.

\end{thebibliography}
